%% file: main.tex
\documentclass[10pt]{article} 

\input{preamble-packages-and-package-settings.tex}

\input{preamble-geometry-settings.tex}
\input{preamble-math_commands.tex}
\input{preamble-environment-definitons.tex}

\input{__figure.tikzstyles}

\title{Equivariant non-linear maps for neural networks\\on homogeneous spaces}

\usepackage{authblk}

\author[1]{\centering Elias Nyholm\thanks{Equal contribution, ordered by first name.}\,}
\author[1]{\centering Oscar Carlsson$^*$}
\author[2]{\centering Maurice Weiler}
\author[1]{\centering Daniel Persson}
\affil[1]{\centering
    Department of Mathematical Sciences, \protect\\%
    Chalmers University of Technology \& University of Gothenburg,\protect \\%
    SE-412 96, Gothenburg, Sweden\vspace{0.8em}
}
\affil[2]{\centering
    Computer Science and Artificial Intelligence Laboratory,\protect\\%
    Massachusetts Institute of Technology,\protect\\%
    Cambridge, Massachusetts, USA
}

\begin{document}
\maketitle

\vspace*{1ex}
\begin{abstract}
\input{sections/abstract}
\end{abstract}

\newpage
\tableofcontents
\newpage

\section{Introduction}
\label{sec:introduction}

\input{sections/introduction.tex}

\section{Mathematical Preliminaries}
\label{sec:homogeneous-spaces-and-equivariance}
\input{sections/equivariance.tex}

\section{Vector features and linear equivariant operators}
\label{sec:linear-maps}
\input{sections/linear-maps}

\section{Arbitrary equivariant operators on vector features}
\label{sec:nonlinear-maps}
\input{sections/nonlinear-maps.tex}

\section{Deriving known equivariant architectures}
\label{sec:instances}
\input{sections/instances}

\section{Conclusions and future work}
\label{sec:conclusions}
\input{sections/conclusions}

\section*{Acknowledgments}
\label{sec:acknowledgment}
\input{sections/acknowledgments}

\bibliography{main}
\bibliographystyle{tmlr}

\newpage
\appendix
\input{sections/appendix-equivariance-relations.tex}
\end{document}

%% file: preamble-packages-and-package-settings.tex
\usepackage{amsthm}
\usepackage{amssymb}
\usepackage{amsmath,amsfonts,bm}
\usepackage{mathtools}
\usepackage[
]{todonotes}
\usepackage{hyperref}
\usepackage{url}
\usepackage{wrapfig}
\usepackage{changepage}
\usepackage{float}
\usepackage{changepage}

\usepackage{eso-pic} 
\usepackage{fancyhdr}
\usepackage{natbib}

\usepackage[T1]{fontenc}
\usepackage{lmodern}

\usepackage{quiver}
\usepackage{tikz}
\usetikzlibrary{calc}
\usetikzlibrary{cd}
\usetikzlibrary{fit,
				positioning,
				shapes,
				arrows.meta,
				backgrounds}
\usetikzlibrary{trees,
				positioning,
				arrows.meta,
				chains,
				shapes.geometric,
				decorations.pathreplacing,
				decorations.pathmorphing,
				shapes,
				matrix,
				shapes.symbols}

\usepackage{tikzit}
\usepackage{pgfplots}
\usepgfplotslibrary{groupplots,dateplot}
\DeclareUnicodeCharacter{2212}{−}
\pgfplotsset{compat=newest}

\usepackage{tikzscale}  

\usepackage{aliascnt}
\usepackage{cleveref}

%% file: preamble-geometry-settings.tex
\skip\footins 9pt plus 4pt minus 2pt
\def\footnoterule{\kern-3pt \hrule width 12pc \kern 2.6pt }
\leftmargini\leftmargin \leftmarginii 2em

%% file: preamble-math_commands.tex
\def\eqref#1{equation~\ref{#1}}

\def\1{\bm{1}}

\input{preamble-math-fonts.tex}

\input{preamble-math-fonts-per-use-case.tex}

\renewcommand{\hom}{\mathrm{Hom}}
\newcommand{\ind}{\mathrm{Ind}}

\DeclareMathOperator{\id}{id}

\newcommand{\isom}{\cong}

\newcommand{\lak}{\overset{\leftarrow}{\kappa}}

%% file: preamble-math-fonts.tex
\def\mcI{{\mathcal{I}}}

\newcommand{\bb}[1]{\mathbb{#1}}
\newcommand{\mr}[1]{\mathrm{#1}}
\newcommand{\mc}[1]{\mathcal{#1}}
\newcommand{\mb}[1]{\mathbf{#1}}



%% file: preamble-math-fonts-per-use-case.tex
\DeclareMathAlphabet{\mathsfit}{\encodingdefault}{\sfdefault}{m}{sl}
\SetMathAlphabet{\mathsfit}{bold}{\encodingdefault}{\sfdefault}{bx}{n}



%% file: preamble-environment-definitons.tex
\theoremstyle{plain}
\newtheorem{theorem}{Theorem}[section]

\newaliascnt{corollary}{theorem}
\newtheorem{corollary}[corollary]{Corollary}
\aliascntresetthe{corollary}
\crefname{corollary}{corollary}{corollaries}
\Crefname{corollary}{Corollary}{Corollaries}

\newaliascnt{proposition}{theorem}
\newtheorem{proposition}[proposition]{Proposition}
\aliascntresetthe{proposition}
\crefname{proposition}{proposition}{propositions}
\Crefname{proposition}{Proposition}{Propositions}

\newaliascnt{lemma}{theorem}
\newtheorem{lemma}[lemma]{Lemma}
\aliascntresetthe{lemma}
\crefname{lemma}{lemma}{lemmas}
\Crefname{lemma}{Lemma}{Lemmas}

\theoremstyle{definition}
\newaliascnt{definition}{theorem}
\newtheorem{definition}[definition]{Definition}
\aliascntresetthe{definition}
\crefname{definition}{definition}{definitions}
\Crefname{definition}{Definition}{Definitions}

\newaliascnt{example}{theorem}
\newtheorem{example}[example]{Example}
\aliascntresetthe{example}
\crefname{example}{example}{examples}
\Crefname{example}{Example}{Examples}

\theoremstyle{remark}
\newaliascnt{remark}{theorem}
\newtheorem{remark}[remark]{Remark}
\aliascntresetthe{remark}
\crefname{remark}{remark}{remarks}
\Crefname{remark}{Remark}{Remarks}

\makeatletter
\def\@maketitle{%
  \newpage
  \null
  \vskip 2em%
  \begin{center}%
  \let \footnote \thanks
    {\huge \@title \par}%
    \vskip 4.2em%
    {\large
      \lineskip 1.3em%
      \begin{tabular}[t]{c}%
        \@author
      \end{tabular}\par}%
    \vskip 2.0em%
    {\large \@date}%
  \end{center}%
  \par
  \vskip 2.4em}
\makeatother

\renewenvironment{abstract}
  {%
  \begin{adjustwidth}{2cm}{2cm}%
  \begin{center}%
  \textbf{Abstract}%
  \end{center}%
  }
  {%
  \end{adjustwidth}%
  }

%% file: __figure.tikzstyles
\tikzstyle{red_node}=[fill=red, draw=black, shape=circle]
\tikzstyle{green_node}=[fill=green, draw=black, shape=circle]
\tikzstyle{blue_node}=[fill=blue, draw=black, shape=circle]
\tikzstyle{black_point}=[fill=black, draw=black, shape=circle, inner sep=2pt]
\tikzstyle{grey_dot}=[fill={rgb,255: red,128; green,128; blue,128}, draw={rgb,255: red,128; green,128; blue,128}, shape=circle, inner sep=2pt]
\tikzstyle{circle_small}=[fill=none, draw=black, shape=circle, inner sep=2pt]
\tikzstyle{circle_medium}=[fill=none, draw=black, shape=circle, inner sep=4pt]
\tikzstyle{circle_large}=[fill=none, draw=black, shape=circle, inner sep=8pt]
\tikzstyle{sky_node}=[fill={rgb,255: red,88; green,176; blue,227}, draw=black, shape=circle]
\tikzstyle{coffee_node}=[fill={rgb,255: red,72; green,55; blue,41}, draw=black, shape=circle]
\tikzstyle{brick_node}=[fill={rgb,255: red,241; green,90; blue,34}, draw=black, shape=circle]
\tikzstyle{dusk_node}=[fill={rgb,255: red,39; green,0; blue,137}, draw=black, shape=circle]
\tikzstyle{energy_node}=[fill={rgb,255: red,255; green,203; blue,5}, draw=black, shape=circle]
\tikzstyle{leaf_node}=[fill={rgb,255: red,127; green,181; blue,57}, draw=black, shape=circle]
\tikzstyle{small_dot}=[fill=black, draw=black, shape=circle, inner sep=1.5pt]
\tikzstyle{grey_dot_small}=[fill={rgb,255: red,128; green,128; blue,128}, draw=none, shape=circle, inner sep=1pt]
\tikzstyle{dark_grey_small_dot}=[fill={rgb,255: red,64; green,64; blue,64}, draw=none, shape=circle, inner sep=1.3pt]
\tikzstyle{white_small_dot}=[fill=white, draw=none, shape=circle, inner sep=1.3pt]
\tikzstyle{rectangle}=[fill=none, draw=black, shape=rectangle]

\tikzstyle{grey}=[-, draw={rgb,255: red,64; green,64; blue,64}]
\tikzstyle{arrow_both_gap}=[->, shorten >=1mm, shorten <=1mm]
\tikzstyle{mapsto}=[{|->}]
\tikzstyle{equivalence}=[<->]
\tikzstyle{thin_grey}=[-, draw={rgb,255: red,64; green,64; blue,64}, line width=0.1pt]
\tikzstyle{line_dashed_grey}=[-, draw={rgb,255: red,64; green,64; blue,64}, dashed]
\tikzstyle{line_dashed}=[-, dashed]
\tikzstyle{red}=[-, draw=red]
\tikzstyle{blue}=[-, draw=blue]
\tikzstyle{green}=[-, draw=green]
\tikzstyle{brick_filled_transparent_no_draw}=[-, fill={rgb,255: red,241; green,90; blue,34}, draw=none, fill opacity=0.6]
\tikzstyle{vector}=[->]
\tikzstyle{line_gap_to_start}=[-, draw=black, fill=none, shorten <=1mm]
\tikzstyle{arrow_gap_to_end}=[->, shorten >=1mm, tikzit category=Arrows]
\tikzstyle{line_thick}=[-, thick]
\tikzstyle{vector_thick}=[thick, ->]
\tikzstyle{filled_dark_copper}=[-, fill={rgb,255: red,0; green,108; blue,92}, draw=none]
\tikzstyle{fill_black}=[-, fill=black, draw=white]
\tikzstyle{fill_white}=[-, fill=white, draw=black]
\tikzstyle{fill_grey_black_draw_white}=[-, fill={rgb,255: red,64; green,64; blue,64}, draw=white]
\tikzstyle{fill_grey_white_draw_black}=[-, fill={rgb,255: red,191; green,191; blue,191}, draw=black]
\tikzstyle{fill_grey_black_draw_black}=[-, fill={rgb,255: red,64; green,64; blue,64}]
\tikzstyle{fill_grey_white_draw_white}=[-, fill={rgb,255: red,191; green,191; blue,191}, draw=white]
\tikzstyle{fill_sky}=[-, fill={rgb,255: red,88; green,176; blue,227}, draw=black]
\tikzstyle{brick}=[-, draw={rgb,255: red,241; green,90; blue,34}, fill=none]
\tikzstyle{fill_brick_draw_black}=[-, fill={rgb,255: red,241; green,90; blue,34}]
\tikzstyle{dusk}=[-, draw={rgb,255: red,39; green,0; blue,137}]
\tikzstyle{fill_dusk_draw_black}=[-, fill={rgb,255: red,39; green,0; blue,137}]
\tikzstyle{fill_copper_draw_black}=[-, fill={rgb,255: red,0; green,169; blue,157}]
\tikzstyle{fill_leaf_draw_black}=[-, fill={rgb,255: red,127; green,181; blue,57}]
\tikzstyle{fill_energy_draw_black}=[-, fill={rgb,255: red,255; green,203; blue,5}]
\tikzstyle{fill_coffee_draw_black}=[-, fill={rgb,255: red,72; green,55; blue,41}]
\tikzstyle{fill_granite_draw_black}=[-, fill={rgb,255: red,93; green,111; blue,122}]
\tikzstyle{fill_west_coast_draw_black}=[-, fill={rgb,255: red,0; green,48; blue,80}]
\tikzstyle{grey_light}=[-, draw={rgb,255: red,191; green,191; blue,191}]
\tikzstyle{fill_energy_transparent_no_draw}=[-, draw=none, fill={rgb,255: red,255; green,203; blue,5}, fill opacity=0.6]
\tikzstyle{fill_white_transparent}=[-, fill=white, draw=none, fill opacity=0.6]
\tikzstyle{fill_dusk_transparent_no_draw}=[-, fill={rgb,255: red,39; green,0; blue,137}, draw=none, fill opacity=0.6]
\tikzstyle{grey_light_thin}=[-, draw={rgb,255: red,191; green,191; blue,191}, line width=0.1pt]
\tikzstyle{energy}=[-, draw={rgb,255: red,255; green,203; blue,5}]
\tikzstyle{leaf}=[-, draw={rgb,255: red,127; green,181; blue,57}]
\tikzstyle{coffee}=[-, draw={rgb,255: red,72; green,55; blue,41}]
\tikzstyle{sky}=[-, draw={rgb,255: red,88; green,176; blue,227}]
\tikzstyle{vector_brick}=[->, draw={rgb,255: red,241; green,90; blue,34}]
\tikzstyle{vector_energy}=[draw={rgb,255: red,255; green,203; blue,5}, ->]
\tikzstyle{fill_dusk_almost_opace}=[-, draw=none, fill={rgb,255: red,39; green,0; blue,137}, fill opacity=0.8]

%% file: sections/abstract.tex
This paper presents a novel framework for non-linear equivariant neural network layers on homogeneous spaces.
The seminal work of Cohen et al. on equivariant $G$-CNNs on homogeneous spaces characterized the representation theory of such layers in the \emph{linear} setting,
finding that they are given by convolutions with kernels satisfying so-called steerability constraints.
Motivated by the empirical success of non-linear layers, such as self-attention or input dependent kernels,
we set out to generalize these insights to the \emph{non-linear} setting.
We derive generalized steerability constraints that any such layer needs to satisfy
and prove the universality of our construction.
The insights gained into
the symmetry-constrained functional dependence of equivariant operators on feature maps and group elements informs the design of future equivariant neural network layers.
We demonstrate how several common equivariant network architectures --
$G$-CNNs, implicit steerable kernel networks, conventional and relative position embedded attention based transformers, and LieTransformers --
may be derived from our framework.

%% file: sections/introduction.tex
\textit{Geometric deep learning} is a subfield of machine learning which leverages prior knowledge on geometric properties as an inductive bias in designing either the model, its training, or both \citep{bronsteinGeometricDeepLearning2017,bronsteinGeometricDeepLearning2021,gerkenGeometricDeepLearning2023,weiler2023EquivariantAndCoordinateIndependentCNNs}. 
In particular, one can use \emph{symmetry} as fundamental design principle for network architectures.
The goal is often to construct neural networks that satisfy certain commutativity constraints with respect to group actions on their input and output spaces. 
Networks satisfying this transformation property are call \textit{equivariant neural networks}.

Equivariant models have been successfully applied in a variety of contexts, including molecular dynamics simulations \citep{zhang2023AI4Science,AlphaFold2021}, particle physics \citep{kondor-lorentz,Zhdanov2024CliffordSteerable} and medical imaging \citep{medical-imaging}. In addition, incorporating equivariance has been found to speed up the network training phase and improve general performance \citep{gerken2022}.

To utilise the full potential of equivariant networks, and to allow design of future equivariant layers, we need to deepen the theoretical understanding of equivariant networks.
One set of equivariant networks for which the theory is well understood is group convolutional neural networks ($G$-CNNs) and $G$-steerable CNNs \citep{kondor-trivedi,weiler20183dSteerableCNNs,cohen-theory-equivariant-hom,cesa2021ENsteerable,weiler2023EquivariantAndCoordinateIndependentCNNs}.

However, CNN-based theories are, without further additions, fundamentally limited in that they are inherently descriptions of \emph{linear} network layers.
In practice, these linear layers are often paired with simple equivariant non-linear activation layers which are crucial for the overall expressivity of the learning model.
But most importantly, contemporary non-linear neural network layers are not limited to simple activation layers.
Rather, modern network architectures include expressive layers such as the linear message passing together with a suitable non-linearity, either applied pointwise or directly on the messages, or the inherently non-linear attention mechanism \citep{vaswaniAttentionAllYou2017,mpnn}.
Particular instances of equivariant versions of these layers exist \citep{hutchinsonLieTransformerEquivariantSelfattention2021,implicit-kernels,equiformer}, but a general theory that connects these disparate architectures is absent from the literature.

Taking inspiration from the linear representation theory developed for equivariant CNNs we develop a \emph{theory of non-linear equivariant neural network layers} using general integral operator transforms based on the mathematical language of group theory and fibre bundles. 
Importantly, we formulate steerability constraints for general equivariant operators, and in doing so present a framework for constructing new equivariant architectures.

In particular, we show how our theory specialises to various important instances of equivariant machine learning layers:
$ G $-CNNs on homogeneous spaces \citep{cohen-theory-equivariant-hom},
the LieTransformer \citep{hutchinsonLieTransformerEquivariantSelfattention2021},
the Implicit Steerable CNNs \citep{implicit-kernels},
and the standard self-attention, both with and without \citep{vaswaniAttentionAllYou2017} relative position bias. 
See \cref{fig:relation-between-ml-architectures} for a diagrammatic representation of this specialisation.

This paper is fairly mathematical and in an effort to bridge the gap between mathematicians and machine learning scientists, we shall in what follows provide some motivational background, both from the machine learning perspective and the mathematical perspective.

\begin{figure}[h]
    \centering
    \tikzset{block/.style={rectangle, draw, text width=10em, text centered, rounded corners, minimum height=3em}}
        \begin{tikzpicture}
             [node distance=1.35cm,
             start chain=going below,]
            \node (n1)   at ( 0, 4) [block]  {Our Framework};
            \node (n2)   at ( 3.20, -1.6) [block] {Self-Attention with Relative Bias \citep{shaw-etal-2018-self}};
            \node (n3)   at ( 6.35, 0) [block] {LieTransformer \citep{hutchinsonLieTransformerEquivariantSelfattention2021}};
            \node (n4)   at (-3.20, -1.6) [block] {Implicit Steerable CNNs \citep{implicit-kernels}};
            \node (gcnn) at (-6.35, 0) [block] {$G$-CNNs on Homogeneous Spaces \citep{cohen-theory-equivariant-hom}};
            \node (n6) at (0, 0) [block] {Self-Attention \citep{vaswaniAttentionAllYou2017}};
            \draw [->] (n1) -- (n6.north)   node[midway,fill=white,xshift= 0pt, align=center] {\Cref{sec:standard-attention-from-general-map}\\(\Cref{thm:instance-self-attention})};
            \draw [->] (n1) to [out=315, in=100] node[midway,fill=white,xshift= 20pt, align=center] {\Cref{sec:relative-self-attention}\\(\Cref{thm:attention-positional})} (n2.north);
            \draw [->] (n1) to [out=225, in=80] node[midway,fill=white,xshift=  -20pt, align=center] {\Cref{sec:implicit-steerable-from-general-map}\\(\Cref{thm:instance-implicit-kernel})} (n4.north);
            \draw [->] (n1) to [out=350, in=120]    node[midway,fill=white,xshift= 0pt,yshift=10pt, align=center] {\Cref{sec:lietransformer-from-general-map}\\(\Cref{thm:instance-lietransformer})} (n3.north);
	        \draw [->] (n1) to [out=190, in=60]  node[midway,fill=white,xshift=0pt,yshift=10pt, align=center] {\Cref{sec:gcnn-from-general-map}\\(\Cref{prop:gcnn-from-general-equiv-map})} (gcnn.north);
        \end{tikzpicture}
        \caption{Architectures and models as special cases of our proposed framework.}
    \label{fig:relation-between-ml-architectures}
\end{figure}
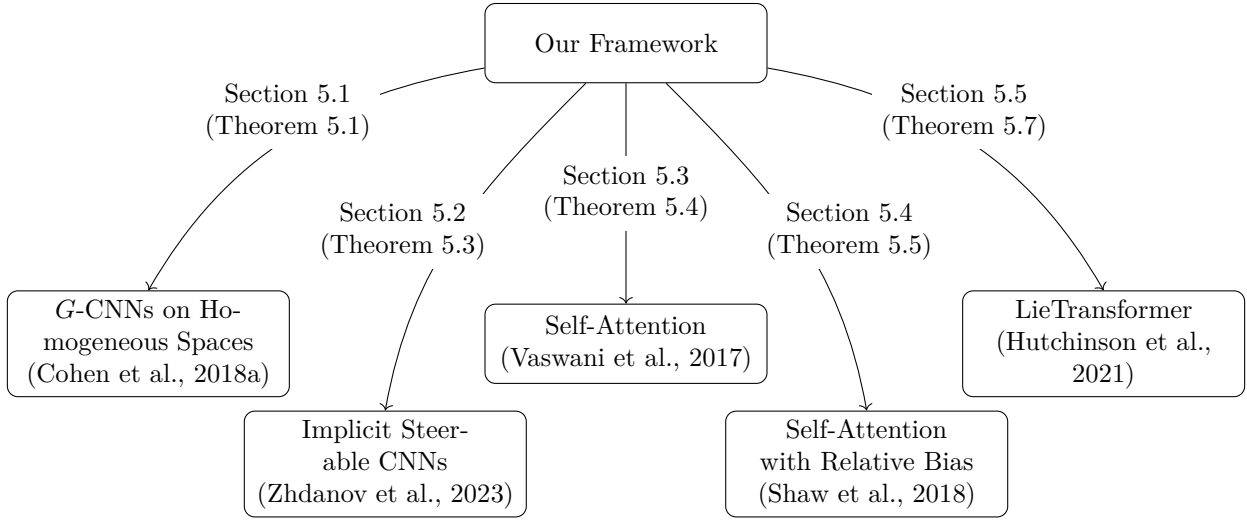

\subsection{Motivation from a machine learner's point of view}\label{sec:machine-learning-motivation}
To motivate this study from a machine learner's point of view we will digress shortly to show specific similarities between CNNs and continuous attention.

Transformer neural networks \citep{vaswaniAttentionAllYou2017} are one of the current go-to architectures when it comes to inherently non-linear models.
There have been several works on making transformers equivariant to certain group actions \citep{fuchs2020se3transformers,romeroGroupEquivariantStandAlone2021,hutchinsonLieTransformerEquivariantSelfattention2021,heGaugeEquivariantTransformer2021,liaoEquiformerV2ImprovedEquivariant2023,kundu2025steerabletransformers} but all of these take the bottom-up approach of modifying the standard attention mechanism to be equivariant and do not take this further by connecting their theory to other architectures.

The (discrete) self attention layer acts on $n$ input tokens $\{f_i\}_{i=1}^n$, where $f_i\in \bb{R}^d$, which, for notational convenience, is often collected into a single feature matrix $F\in \bb{R}^{n\times d}$.
With this setup the standard attention layer can be written as
\begin{equation}\label{eq:pre-standard-attention}
    F\ \mapsto\ O\,=\,\mathrm{Sim}(Q,K)V\,=\,\mathrm{Sim}(F \mathbf{W}_Q,F \mathbf{W}_K)F \mathbf{W}_V\ \ \in\ \bb{R}^{n\times d''},
\end{equation}
where the queries $\mb{Q}=F\mb{W}_Q$, keys $\mb{K}=F\mb{W}_K$ and values $\mb{V}=F\mb{W}_V$ are computed by applying linear maps $\mb{W}_Q,\, \mb{W}_K,\in \bb{R}^{d\times d'}$ and $\mb{W}_V\in\bb{R}^{d\times d''}$ to the input $F$.
Note that all keys, queries, and values are created from the same input $F$; this is the reason for the name self attention.

The $ \mathrm{Sim} $ function is a matrix valued map computing some similarity between the tokens in the queries and the keys, and hence is of shape $ n \times n $.
Often the similarity function, in addition to computing the alignment of the queries and keys, also includes some normalisation where the standard choice is the $ \mathrm{softmax} $ function, which ensures that the row-wise sum of the similarity matrix is one.
A very common choice for computing the alignment is the inner product similarity function. 
With a inner product similarity and $ \mathrm{softmax} $ normalisation one obtains the standard attention as
\begin{equation}\label{eq:standard-attention}
    F\ \mapsto\ O=\mr{Softmax}\left(\frac{QK^\top}{\sqrt{d}}\right)V=\mr{Softmax}\left(\frac{F\mb{W}_Q\mb{W}_K^\top F^\top}{\sqrt{d'}}\right)F\mb{W}_V\ \ \in\ \bb{R}^{n\times d''}.
\end{equation}

If we want to rephrase \cref{eq:pre-standard-attention} in terms of the individual tokens, we get a sum
\begin{equation}\label{eq:self-attention-as-sum}
    O_j=\sum_{i=1}^n \mathrm{Sim}(Q,K)_{ji}V_i,
\end{equation}
where $\mathrm{Sim}$ is the matrix valued similarity function, acting on the full set of queries $Q$ and keys $K$, and $V_i$ is the $i$:th token in the values.

We now want to rewrite this as an integral over some space $ X $. 
First we note that if $ X=\left\{ x_1, \cdots, x_n \right\} $ is a discrete set we can use these to index our tokens and define the feature at each point $ x \in X $ through a map $ f:X \to \mathbb{R}^{d} $.
Then we set $Q=\mb{W}_Q^\top f,\ K=\mb{W}_K^\top f$, and $ V=\mb{W}_V^\top f$, which allows \cref{eq:self-attention-as-sum} to be written as
\begin{equation}
    o(x)=\sum_{x^{\prime}\in X} \mathrm{Sim}(\mb{W}_Q^\top f,\mb{W}_K^\top f)(x,x^{\prime})\mb{W}_V^\top f(x^{\prime})=\int_X \mathrm{Sim}(\mb{W}_Q^{\top}f,\mb{W}_K^{\top}f)(x,x')\mb{W}_V^{\top}f(x')\mathrm{d}x',
\end{equation}
where the last equality holds if $ X $ is a discrete set and the measure is the counting measure.
Now, if we allow $ X $ to be a continuous space we obtain a continuous form of the self attention.

However, we can go one step further by combining the similarity function $\mathrm{Sim}$ and the values matrix into a single instance by defining a map $\alpha(f,x,x')=\mathrm{Sim}(\mb{W}_Q^{\top}f,\mb{W}_K^{\top}f)(x,x')\mb{W}_V^\top$.
This results in the continuous attention on the form
\begin{equation}\label{eq:introduction-alpha-map}
    f_o(x)=[\Phi^{ \text{Attention} } f](x)=\int_X\alpha(f,x,x')f(x')\mathrm{d}x'.
\end{equation}

Here, note that this is reminiscent of the form of a general linear map formulated as an integral transform with a matrix-valued two argument kernel $\kappa$:
\begin{equation}\label{eq:general-lin-transform}
    [\Phi^{\text{Linear}} f](x)=\int_X \kappa(x,x')f(x')\mathrm{d}x'.
\end{equation}
The essential difference here is that the attention weights $\alpha(f,x,x')$ depend not only on points $x$ and $x'$ but are also functions of the features $f$ -- this makes the computation non-linear and \emph{input dependent}.

\subsubsection*{Equivariance}
To derive convolutions, one imposes an equivariance constraint, with respect to translation, on the linear map in \cref{eq:general-lin-transform} \cite[Theorem 3.2.1]{weiler2023EquivariantAndCoordinateIndependentCNNs}.
It follows that the kernel $\kappa$ may depend on relative distances only,
\begin{equation}
    \kappa(x,x') = \kappa(x-y,x'-y),
\end{equation}
which is often referred to ``spatial weight sharing''.
One may then define a one-argument kernel through
$\kappa(x,x^{\prime})=\kappa(0,x^{\prime}-x)=:\hat \kappa(x^{\prime}-x)$
in terms of which the linear map in \cref{eq:general-lin-transform} becomes a, translation equivariant, convolutional layer:
\begin{equation}
    [\Phi^{\mathrm{Convolution}} f](x)=\int_X \hat\kappa(x'-x)f(x')\mathrm{d}x'.
\end{equation}
This raises some natural questions:
\begin{center}\it
    Can we generalize this to construct an equivariant version of continuous attention?
\end{center}
and
\begin{center}\it
    If both attention and convolutions are special cases of an overarching structure, \\
    can we generalise this further and include more architectures?
\end{center}
These questions are the main subject of this paper.

\subsection{Motivation from a mathematician's point of view}
A very popular building block of deep learning models are convolutional layers,
which are linear translation equivariant layers operating on spatial ``feature maps''.
In the recent years, convolution layers have been generalized to be equivariant under more general groups $G$.
The feature maps of such $G$-CNNs are therefore equipped with $G$-actions, which are mathematically formalized as so-called \textit{induced representations} \citep{mackeyInducedRepresentationsGroups1951,mackeyInducedRepresentationsLocally1953,ceccherini-silbersteinInducedRepresentationsMackey2009}.

Non-linear layers are of interest in the machine learning world since this increases the expressibility of the models. 
Often, as a general framework is missing, equivariant non-linear maps are constructed bottom-up from a given task.
The purpose of this paper is to extend the
theory of linear $G$-CNNs to the non-linear setting and to investigate which known architectures can be derived from the more general framework.

To be more specific, \cite{cohen-theory-equivariant-hom} presented a rich theory for group equivariant CNNs using induced representations for the case where the underlying manifold is a homogeneous space. 
This is natural as a homogeneous space $ X $ can be expressed as $ X \isom G/H $ where $H \leq G$ is a stabiliser subgroup of some point $ g_0 $ in the group $ G $.
Hence, there is a natural action of $G$ on data locally defined as functions $ f:X \to V_\sigma $ where $ V_\sigma $ is a left $ H $-space under the representation $ \sigma $.\footnote{Technically, such a ``function'' $f:X \to V_\sigma$ exist only locally $f:U\to V_\sigma$ where $U\subset X$. This stems from the fact that, unless the geometry is topologically trivial, no global choice of coordinates can be made.}
The space of such functions, or more generally, the space of all corresponding sections of an associated vector bundle $ \Gamma(G \times _\sigma V_\sigma) $ is isomorphic to the space of functions 
\begin{equation}
    \ind _{H}^{G}\sigma\ :=\ \big\{ f:G \to V_\sigma \,\big|\, f(gh)=\sigma(h^{-1})f(g),\ \forall h \in H \big\},
\end{equation}
which, together with the $G$-action $ [kf](g)=f(k^{-1}g) $, is called the induced representation. The ``induction'' process allows us to uniquely construct a $G$-representation on a specific function space from the $H$-representation $\sigma$.

Analogously to \cref{eq:general-lin-transform}, a linear map $ \Phi^\text{Linear}: \ind _{H^{\prime}}^{G}\rho \to \ind _{H}^{G}\sigma $ between induced representations is generally on the form
\begin{equation}\label{eq:general-lin-map-on-G}
    [\Phi ^\text{Linear} f](g)=\int_G \kappa(g,g^{\prime})f(g^{\prime})\mathrm{d}g^{\prime},
\end{equation}
where the kernel $ \kappa(g,g^{\prime})\in \hom(V_\rho,V_\sigma) $ satisfies a compatibility constraint
\begin{equation}
    \kappa(gh,g^{\prime}h^{\prime})=\sigma(h^{-1})\kappa(g,g^{\prime})\rho(h^{\prime})
\end{equation}
to ensure that the output is indeed in $ \ind _{H}^{G}\sigma $ for inputs $ f \in \ind _{H^{\prime}}^{G}\rho $.

In this setting $ G $-equivariant linear maps, also known as intertwiners, are just $ G $-equivariant CNN layers $\Phi^{G\text{-convolution}}: \ind _{H^{\prime}}^{G}\rho \to \ind _{H}^{G}\sigma $ between the induced representations describing the layer's input and output features \citep{cohen2018intertwinersinducedrepresentationswith}. 
Under mild conditions, it can be shown that all such intertwiners can be written in the form
\begin{equation}\label{eq:G-conv}
    [\Phi^{G\text{-convolution}} f](g)=\int_G \hat\kappa(g^{-1}g^{\prime})f(g^{\prime})\mathrm{d}g^{\prime},
\end{equation}
where $ \hat \kappa: G \to \hom(V_\rho,V_\sigma) $ is a generalized convolution kernel which assigns linear operators $\kappa(g^{-1}g^{\prime}) \in \hom(V_\rho,V_\sigma)$ to relative group elements $g^{-1}g'$.

Geometrically, the feature map $f$ is a section of a vector bundle $P\rightarrow G/H$ with fiber $V_\sigma$. Drawing on insights from equivariant cohomology, we can generalize this to provide a unified perspective on feature maps. This is conveniently visualised in the following diagram:
\[ 
\begin{tikzcd}[ampersand replacement=\&,cramped]
	P \&\& {P\times M} \&\& M \\
	\\
	G/H \&\& {P\times_H M} \&\& {H\backslash M}
	\arrow[, from=1-1, to=3-1]
	\arrow[, from=1-3, to=1-1]
	\arrow[, from=1-3, to=1-5]
	\arrow[, from=1-3, to=3-3]
	\arrow[, from=1-5, to=3-5]
	\arrow[, from=3-3, to=3-1]
	\arrow[, from=3-3, to=3-5]
\end{tikzcd}
\]
This is a commutative diagram, known as the Cartan mixing diagram, that captures the combination of a principal $H$-bundle, $P\rightarrow G/H$, and a $H$-space $M$ (e.g. $V_\sigma$), producing a fibre bundle with fibre $M$ (the associated bundle). We use various versions of this diagram throughout the paper as a convenient tool to capture the essential properties of equivariant maps. For more details see~\cref{sec:bundles}.

Inspired by the linear case where one can derive the $ G $-equivariant  map \cref{eq:G-conv} from \cref{eq:general-lin-map-on-G} we propose a more general, non-linear, map $ \Phi^{\mathrm{Non-linear}}:\, \ind _{H^{\prime}}^{G}\rho \,\to\, \ind _{H}^{G}\sigma $ 
between induced representations. 
We define this map as
\begin{equation}
\label{eq:NonLinear}
    [\Phi^{\mathrm{Non-linear}} f](g)\ =\ \int_{G}\omega(f,g,g^{\prime})\mathrm{d}g^{\prime},
\end{equation}
where
\begin{equation}
    \omega:\, \ind _{H^{\prime}}^{G}\rho \times G \times G \,\to\, V_\sigma
\end{equation}
maps into the target representation $ V_\sigma $.
Note that this is an analogue to the non-equivariant linear map \cref{eq:general-lin-map-on-G} and we need more constraints to make this $ G $-equivariant.

Additionally, since we have not further specified how $ \omega $ depends on the input data $ f $, this includes both non-linear and linear dependencies on $ f $.
Intuitively, one can view $ \omega $ as a map describing the spatial aggregation of features,
a version of message passing in which information is sent from $g^{\prime}$ to $g$. 

In this paper we will investigate the requirements for the map in \cref{eq:NonLinear} to be equivariant. 
More generally:

\begin{center}\it
A central goal of this paper is to develop a framework for equivariant non-linear integral maps \\ for neural networks on homogeneous spaces.
\end{center}

\subsection{Summary of results}
In this section we provide a brief summary of the main results of this paper. 
If not stated otherwise, assume $ G $ to be a (locally compact) Lie group and $ H,H^{\prime} \leq G $ any two subgroups.
Further, let $V_\sigma$ be a finite vector space and  $ \sigma $ be a representation of $ H $ on $V_\sigma$ and, similarly, $ V_\rho $ a representation of $ H^{\prime} $. 

The main contributions of this paper are as follows:
\begin{itemize}
    \item
        
        Linear
        maps between induced representations $ \Phi_\kappa :\ind _{H^{\prime}}^{G}\rho\to \ind _{H}^{G}\sigma $ can be written $ [\Phi_\kappa f](g)=\int_G \kappa(g,g^{\prime})f(g^{\prime})\mathrm{d}g^{\prime} $ for an operator valued integration kernel $ \kappa:G \times G \to \hom(V_\rho,V_\sigma) $ satisfying certain symmetry constraints.
        These constraints are usually formulated as, see e.g.\ \cite{cohen-theory-equivariant-hom}, 
        \begin{equation}
            \kappa(gh,g^{\prime}h^{\prime})=\sigma(h^{-1})\kappa(g,g^{\prime})\rho(h^{\prime}), \quad h\in H, h^{\prime}\in H^{\prime}.
        \end{equation}
        We prove new results on redundancies in the space of the associated integration kernels.
        Specifically, that the transformation under $ h^{\prime} $ comes from an equivalence relation construction
\begin{equation}
    \kappa \sim \kappa^{\prime}\quad\Leftrightarrow\quad \int_{H^{\prime}} \kappa(g,g^{\prime}h^{\prime})\rho(h^{\prime})\mathrm{d}h^{\prime}=\int_{H^{\prime}}\kappa^{\prime}(g,g^{\prime})\rho(h^{\prime})\mathrm{d}h^{\prime}.
\end{equation}
        Additionally, for each equivalence class there exist a unique kernel $ \kappa_0 $ satisfying $ \kappa_0(g,g^{\prime}h^{\prime})=\kappa_0(g,g^{\prime})\rho(h^{\prime}) $ 
        For more details, see lemma~\ref{lemma:kernel-redundancy}.

     \item 
     We introduce a novel framework for
     nonlinear
     equivariant operators $ \Phi^{\mathrm{Non-linear}}:\ind _{H^{\prime}}^{G}\rho\to \ind _{H}^{G}\sigma $ between feature maps as
\begin{equation}
    [\Phi^{\mathrm{Non-linear}} f](g)=\int_G \hat \omega(g^{-1}f,g^{\prime})\mathrm{d}g^{\prime},
\end{equation}
         see \cref{cor:steerable_functional_argument_reduction}.
         We prove their universality in \cref{prop:arbitrary-map-from-general-equiv-map}.
         The map $ \hat \omega:\ind _{H^{\prime}}^{G}\rho \times G\to \ind _{H}^{G}\sigma $ satisfies the constraints 
         \begin{equation}
             \hat \omega(h^{-1}g^{-1}f,g^{\prime}h^{\prime})=\sigma(h^{-1})\hat\omega(g^{-1}f,g^{\prime}), \quad \forall h\in H, h^{\prime}\in H^{\prime}.
         \end{equation}
     \item We show explicitly how the general map $ \Phi^{\mathrm{Non-linear}} $ specialises to well-established equivariant architectures through specific choices of $ \hat \omega $ as visualised in \cref{fig:specialisations-from-general-map}.
\end{itemize}
\begin{figure}[H]
\begin{tikzpicture}[
    >=Stealth,
    leaf/.style={          
      align=center,
      inner sep=2pt
    },
    convleaf/.style={       
      leaf,
      text width=16em
    },
    attnleaf/.style={       
      leaf,
      text width=21em,
      execute at begin node={\setlength\baselineskip{1.8\baselineskip}}
    },
    attnlabel/.style={       
      leaf,
      text width=21em,
    },
    arrowlabel/.style={     
      font=\footnotesize,
      pos=0.3,
      fill=white,
      inner sep=1pt
    }
]

  \node (omega)     at ( 0.325,  1.95) {$\omega(f,g,g')$};
  \node (omega_hat) at ($(omega)+(0,-3)$) {$\hat\omega(g^{-1}f,g')$};

  \coordinate (convC) at (-4,-1.5);
  \node[convleaf] (gcnn)  at ($(convC)+(0,1)$)
    {$\hat\kappa(g')[g^{-1}f](g')$};
  \node[convleaf] (gcnn_title) at ($(gcnn)+(0,0.5)$) {$ G $-CNN (\cref{prop:gcnn-from-general-equiv-map}):};
  \node[convleaf] (steer) at ($(convC)+(0,-1)$)
    {$\kappa\bigl(g',[g^{-1}f](e),[g^{-1}f](g')\bigr)\,[g^{-1}f](g')$};
  \node[convleaf] (steer_title) at ($(steer)+(0,0.7)$) {Implicit steerable CNN (\cref{thm:instance-implicit-kernel}):};
  
  \coordinate (attnC) at ($( 5.5,-1.989)-(0,0.8)+(0,1.1)$);
  \node[attnleaf] (stdattn) at ($(attnC)+(0,2.8)$)
    {$\mathrm{Softmax}\!\Bigl\{\tfrac{[g^{-1}f](e)^\top W_Q^\top W_K\,[g^{-1}f](g')}{\sqrt d}\Bigr\}W_V\,[g^{-1}f](g')$};
  \node[attnleaf] (stdattn_title) at ($(stdattn)+(0,0.8)$) {Self-attention (\cref{thm:instance-self-attention}):};
  \node[attnleaf] (relbias) at ($(attnC)+(0,0.2)$)
    {$\mathrm{Softmax}\!\Bigl\{\tfrac{[g^{-1}f](e)^\top W_Q^\top W_K\,[g^{-1}f](g')}{\sqrt d}\!+\psi(g')\Bigr\}$};
    \node[attnleaf] (relbias_two) at ($(relbias)-(0,0.7)$) {$\cdot W_V\,[g^{-1}f](g')$};
  \node[attnlabel] (relbias_title) at ($(relbias)+(0,1.0)$) {Relative Position Bias self-attention (\cref{thm:attention-positional}):};
  \node[attnleaf] (lietr)   at ($(attnC)+(0,-2.5)$)
    {$\mathrm{norm}\{\alpha([g^{-1}f](e),[g^{-1}f](g'),g')\}W_V\,[g^{-1}f](g')$};
  \node[attnleaf] (lie_title) at ($(lietr)+(0,0.6)$) {LieTransformer (\cref{thm:instance-lietransformer}):};

  \begin{pgfonlayer}{background}
    \node[
      name=convbox,
      draw, fill=blue!10, rounded corners,
      inner xsep=0.5em, inner ysep=0.5em,
      fit=(gcnn) (steer) (gcnn_title)
    ] {};
    \node[
      name=attnbox,
      draw, fill=orange!10, rounded corners,
      inner xsep=0.5em, inner ysep=0.5em,
      fit=(lietr) (stdattn) (relbias) (stdattn_title)
    ] {};
    \node[
      name=omegabox,
      draw, fill=green!10, rounded corners, 
      inner xsep=0.25em, inner ysep=0.25em,
      fit=(omega) (omega_hat)
    ] {};
  \end{pgfonlayer}

    \node[font=\footnotesize] at (convbox.north) [above=5.933em]
    {\bf Convolution based};
  \node[font=\footnotesize] at (attnbox.north) [above=0.5em]
    {\bf Attention based};
  \node[font=\footnotesize] at (omegabox.north) [above=0.50em]
    {\bf General framework};

  \draw[->,double] (omega) -- (omega_hat)
    node[pos=0.3,fill=green!10,rounded corners,inner sep=1pt]{Equivariance};

  \draw[->] (omega_hat) to[bend right=20]  ($(gcnn.east)+(-1,0)$)
    node[arrowlabel,above] {};

  \draw[->] (omega_hat) to[bend left=10]  ($(steer.east)+(0,0)$)
    node[arrowlabel,above] {};

  \draw[->] (omega_hat) to[bend left=20]  ($(stdattn.west)+(0.0,0.0)$)
    node[arrowlabel,above] {};

  \draw[->] (omega_hat) to[bend right=30] ($(relbias.west)+(0.3,0.0)$)
    node[arrowlabel,above] {};

  \draw[->] (omega_hat) to[bend right=20] ($(lietr.west)+(0,0)$)
    node[arrowlabel,above] {};

\end{tikzpicture}
\caption{Specialisations of the general framework based on the integrand $ \hat{\omega}:\mathcal{I}_\rho \times G\to V_\sigma $ to well-established architectures for different choices of the map $ \hat\omega $.
The $G$-CNN case is obtained when $ \hat{\omega} $ factorises in a kernel $ \hat{\kappa}: G \to \hom(V_\rho,V_\sigma) $, which is independent of the input features, and the input feature $ g^{-1}f:G \to V_\rho $.
The other cases also appears through different factorisations as some type of kernel contracted with the feature $ g^{-1}f $.
However, in all cases the kernel is also dependent on the input feature $ g^{-1}f $, which separates them from the $G$-CNN case and makes them inherently non-linear in $ f $. }\label{fig:specialisations-from-general-map}
\end{figure}
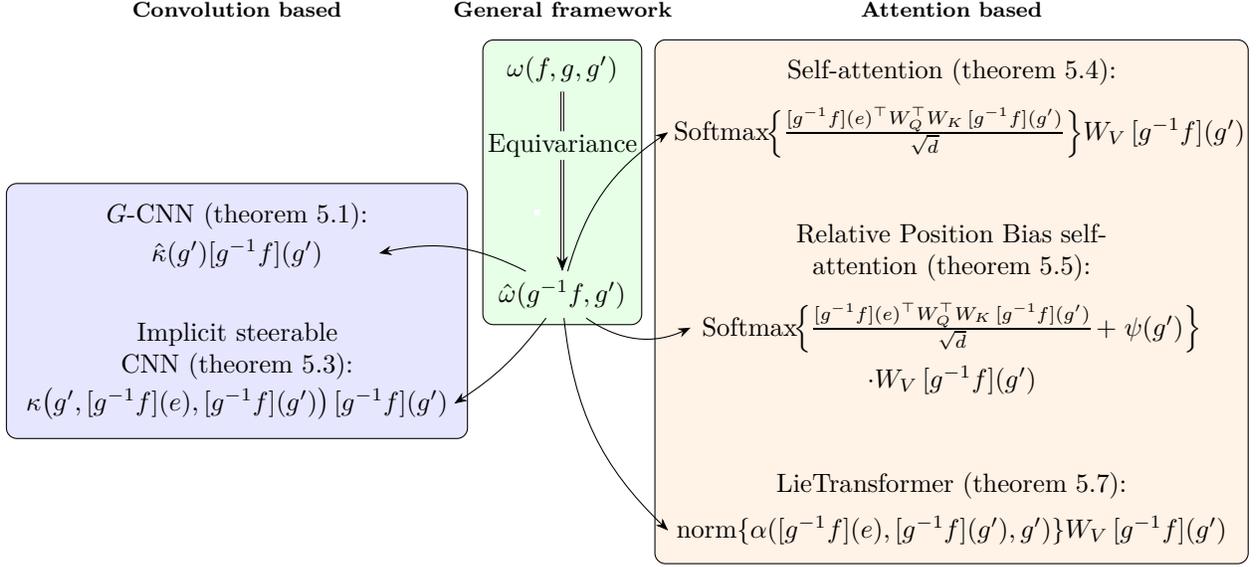

\subsection{Outline}
The outline of the paper is as follows. In \cref{sec:homogeneous-spaces-and-equivariance} we introduce some mathematical background on group equivariance, homogeneous spaces, fibre bundles, and induced representations.
Readers familiar with these topics can skip this section. 
\Cref{sec:linear-maps} reviews the theory of \emph{linear} equivariant maps of \cite{cohen-theory-equivariant-hom}. We uncover some novel insights in this story which, to our knowledge, have not been previously observed.
Specifically, we show that the redundancy present in the induced representation creates a natural equivalence class for integration kernels. Moreover, we demonstrate that what has previously been formulated as a condition on the kernel is actually a choice of  equivalence class. We then move on to develop a more general formulation of, potentially, \emph{non-linear} equivariant maps in \cref{sec:nonlinear-maps}.
Additionally, we show that this formulation is universal, i.e. includes \emph{all} equivariant non-linear maps.
Finally, in \cref{sec:instances} we derive several established architectures from the general framework through a specific choice of the map $ \omega $. 
For a preview of the architectures, see \cref{fig:specialisations-from-general-map}.
\Cref{sec:conclusions} concludes the paper and presents the conclusions of the work and some paths for future research. The paper also includes an appendix in which we present a distributional version of our general framework.

%% file: sections/equivariance.tex
As introduced by Felix Klein through the Erlangen program \citep{Klein1893}, the study of geometric spaces relates closely to the theory of group transformations.
One can argue that structures (functions, maps) defined on a geometric space should be compatible with the group transformations associated with the space in order to properly take into account its geometry.
In this section we define the necessary mathematical structures and constructions in order to talk about these notions in a well-defined manner.

Starting off in \cref{sec:equivariance} we define the action of groups on spaces and introduce the precise notion of transformation compatibility for maps on the space.
Following this, we present a certain well known  type of space called a homogeneous space in \cref{sec:homogeneous-spaces}, where the symmetry transformations are particularly nicely behaving.
For the development in subsequent chapters we will assume the underlying base space to have this structure.
In \cref{sec:bundles,sec:induced-reps-and-frobenius-reciprocity} we introduce further mathematical structures which arise the study of homogeneous spaces and symmetry-compatible maps, namely bundles and induced representations.
These structures provide the main tools for the developments in subsequent chapters.

People familiar with these topics can return to this section later, if needed. 

\subsection{Equivariance}\label{sec:equivariance}
We start off with formalising the way spaces transform, which is captured by the notion of group actions.
\begin{definition}[Actions, representations]
	A group action of the group $G$ on a space $X$ is a map
	\begin{equation}
		\begin{aligned}
			G \times X &\rightarrow X \\
			(g,x) &\mapsto g \vartriangleright x
		\end{aligned}
	\end{equation} satisfying
    \begin{align*}
	    (g_1 g_2) \vartriangleright x = g_1 \vartriangleright (g_2 \vartriangleright x)
    \end{align*}
    for all $g_1, g_2 \in G$, $x\in X$, and the identity element $e \in G$.
	The space $X$ on which the action is defined is called a $G$-space.
    In the special case where the space $X$ is a vector space and the action $g \vartriangleright: X \rightarrow X$ is a linear map for all $g \in G$, then the action is called a representation of $G$ on the vector space $X$.
\end{definition}
\begin{remark}
	We will throughout the article often use the short-hand notation $gx := g \vartriangleright x$ to enhance readability. 
    Additionally, when $G$ acts on a vector space $X$, the linear map $g\vartriangleright$ is denoted $\rho(g)$ and the vector space $\rho(g)$ acts on is noted $X_\rho$.
    Commonly the pair $(\rho,X_\rho)$ is together called a representation, or in short that $X_\rho$ is a $G$-representation.
\end{remark}

\begin{remark}
   The action defined above will have other derived properties that one might find natural, such as $e \vartriangleright x = x$ and $(g ^{-1}) \vartriangleright = (g \vartriangleright)^{-1}$. 
\end{remark}

\begin{definition}[Orbits, quotients]
	\label{def:orbits}
	Consider a $G$-space $X$.
	For each point $x \in X$ we can form the orbit $[x] = \{ g \vartriangleright x \mid g \in G \}$ consisting of all points in $X$ one can transform $x$ into by the action of $G$.
	In fact, the relation $x \sim x'$ if $[x] = [x']$ is an equivalence relation, and we can form the quotient space  $G \backslash X = \{ [x] \mid x \in X\}$ consisting of all orbits in $X$ by the action of $G$.
\end{definition}
\begin{remark}
	One can also define the action of a group from the right $x \mapsto  x \vartriangleleft g$, in which case it is common to also flip the notation for the quotient space $X/G$.
\end{remark}

Examples of actions arise virtually at any place where we want to consider symmetry transformations of some space or set.
\begin{example}
    Instances of group actions include permutations by $ \operatorname{Sym}(S)$ acting on a finite set $S$, which will be relevant in \cref{sec:standard-attention-from-general-map}, rotations by $SO(2)$ acting on $ \mathbb{R}^2$ and the diffeomorphism group $ \operatorname{Diff}(M)$ acting on a differentiable manifold $M$.
\end{example}
After defining some structure on a space, it is natural to ask what kind of maps preserve the given structure.
For example, homeomorphic maps are those that preserve topological structure and linear maps are those those maps that preserve linear (vector space) structure.
The structure-preserving maps on spaces equipped with $G$-actions are called $G$-equivariant maps.

\begin{definition}[Equivariant, invariant maps, intertwiners]\label{def:equivariant-maps-intertwiner}
\label{def:equivariance}
	Consider two $G$-spaces $X$ and $Y$.
	A map $f: X \rightarrow Y$ is called $G$-equivariant if it is  compatible with the actions of $G$ in the sense that
	\begin{equation}
		\label{eq:equivariance}
	f(gx) = gf(x)
	\end{equation}
	for all $g \in G$ and $x \in X$.
	If the action of $G$ on $Y$ is trivial  $gy \equiv y$ we say that a map $f$ satisfying \cref{eq:equivariance} is $G$-invariant, that is $ f(gx)=f(x) $. 
When it is clear which group $G$ is considered we will simply call $f$ equivariant or invariant, respectively.
In the particular setting when $X$ and $Y$ are vector spaces and $f$ is an equivariant linear map $f$ if called an intertwiner.
\end{definition}
\begin{remark}
    Since $ g $ in the right hand side of \cref{eq:equivariance} acts on the output of $ f $, a more clear notation might be $ f(gx)=g(f(x)) $ in order to avoid confusing this with the action of $ g $ on the function itself. 
    However, to avoid confusion and stacking parentheses we denote the action of a group element on a function as $ [gf](x) $.
\end{remark}

Throughout the article we will derive and motivate equivariance conditions for several diverse types of maps.
The definition is quite general, as the spaces $X$ and $Y$ can be anything from homogeneous spaces and linear spaces (see \cref{sec:features}) to function spaces and Hom-spaces (see \cref{sec:linear-maps,sec:nonlinear-maps}).

\subsection{Homogeneous spaces}\label{sec:homogeneous-spaces}
The previous section revolves around maps between spaces equipped with actions of some group $G$.
Various assumptions can be made about how exactly the group acts on these spaces, e.g. continuously on topological spaces or linearly on vector spaces.
A particularly useful property is that of a transistive action:

\begin{definition}[Homogeneous space]
\label{def:homogeneous-space}
	A $G$-space $X$ is called a homogeneous space if the action on $X$ is transitive.
	That is, $X$ is a homogeneous space if for any two points $x,y \in X$ there exist at least one group element $g \in G$ such that $g \vartriangleright x = y$ i.e. transforms $x$ to $y$.
\end{definition}

The following isomorphism is central to much of the theory built around homogeneous spaces.

\begin{proposition}\label{def:homogeneous-space-isomorphism}
	A homogeneous space $X$ is isomorphic to the quotient space $G/H$ of the group $G$ and its stabilizer subgroup $H = \operatorname{Stab}_{G}(x_0)$ of an arbitrary point $x_0 \in X$.
	Conversely, any quotient space $G/H$ is a homogeneous space equipped with the action of $G$ induced by the group operation.
	In particular, the construction is independent of the choice of $x_0$.
\end{proposition}
\begin{proof}
	See e.g. \citet{Lee2012} for proof.
\end{proof}
It is clear from the isomorphism that any group $G$ is itself a homogeneous space where the stabilizer subgroup $H$ is trivial.
More generally, one can interpret the isomorphism as the observation that homogeneous spaces are those that look like some group $G$ up to operations which leave a point $x_0$ fixed.
\begin{example}
	Homogeneous spaces are not seldom spaces which are defined without reference to a group action, but on which a transistive group action can be defined and as a consequence is isomorphism to a quotient space.
	Examples include the $n$-sphere $S^n \isom O(n)/O(n-1)$, the Euclidean space  $E^n \isom E(n)/O(n)$ and the projective space  $ \mathbb{R}P^n \isom SO(n+1) / O(n)$.
\end{example}

\subsection{Bundles and equivariance}\label{sec:bundles}
Before we move on to discussing data in neural networks in detail in \cref{sec:features} we need to introduce some fundamental definitions and concepts relating to fibre bundles.
These constructions arise naturally in the study of equivariance and homogeneous space and will play a central role in the framework presented in subsequent sections.

Formulated intuitively, a fibre bundle $ E $ with base space $ B $ and fibre $ F $ can be viewed as gluing a copy of $ F $ to each point of $ B $.
Additionally, each fibre bundle has a projection $ \pi:\, E \,\to\, B $ and the fibre over a specific point is the preimage of this projection $ \pi^{-1}(x)=:F_x $ and is isomorphic to the general fibre $ F $.
A trivial example of a fibre bundle is the trivial bundle with $ E=F \times B $, and generally, all fibre bundles are locally trivial.
A non-trivial example is the Möbius strip which can be constructed with the fibre $ F $ as the line segment $ F=[-1,1] $ glued to the circle $ B=S^{1} $, which is the base, with a global twist.

Sometimes the notation $E_F$ is used for the total space if we feel the need to specify the fibre.
Special cases include vector bundles where the fibre have the structure of vector spaces and principal bundles where the fibres are isomorphic as a set to some group.
We assume some familiarity with bundles and their sections, local maps from the base space to the total space, and refer the interested reader to \citet{isham} for a detailed introduction to the subject.
We will, however, define the principal bundles, as these will play a key role in subsequent chapters.

\begin{definition}[Principal $H$-bundle]\label{def:principal-bundle}
Consider a group $H$ and a right $H$-space $P$ with action $\lhd: P \times H \to P$.
A principal $H$-bundle $\pi:P\to B$ is a fibre bundle such that the projection is $H$-invariant $\pi(p\lhd h) = \pi(p)$ and that $H$ acts freely and transitively on the fibres.
\end{definition}
A consequence of that $H$ acts freely and transitively on the fibres is that the fibres are $H$-torsors, and given a choice of identity in each fibre, are isomorphic to $H$.
Since $ H $ acts freely and transitively on the fibres, the fibres themselves are $ H $-torsors.
That is, the fibres are isomorphic, as sets, to the group $ H $ since the fibres has no preferred identity. 
Also worth noting is that, though conventionally chosen that way, the $H$ need not act on $P$ from the right but can equally well be defined as a left action, as long as everything else is adjusted accordingly.
Unless otherwise stated, we will use a right action for the principal bundles.
\begin{example}
	\label{ex:homogeneous-principal-bundle}
A homogenous $G$-space $X$ with isomorphism $\psi: X \xrightarrow{\isom} G/H$ can be viewed as the base space of a principal $H$-bundle $G\to G/H$.
In this perspective, a choice of section $s:G/H \to G$ defines a lifting map which takes a point $x \in X$ to a representative group element $g_x := s(\psi(x)) \in G$.
In particular, $g_x$ is a representative of the orbit of the action of $H$ on $G$.
\end{example}
Now, allowing a group action on both the total and base space, independently of any fibre preserving group action, leads to the natural condition that the projection map commutes with the group action.
This results in a construction used in equivariant cohomology: the equivariant bundle.\footnote{We choose to use this terminology which is familiar from the mathematics community.
The same structure is referred to as an \textit{isometry/diffeomorphism pushforward} in \citet{weiler2023EquivariantAndCoordinateIndependentCNNs}.}

\begin{definition}[$G$-equivariant bundle]\label{def:equivariant-bundle}
    A fibre bundle $\pi:E\to B$ where both $E$ and $B$ are $G$-spaces is a $G$-equivariant bundle if $\pi(ge)=g\pi(e)$, with the corresponding group actions.
\end{definition}
\begin{remark}
If the base space is homogeneous these bundles are often simply called homogeneous bundles.
\end{remark}
\begin{remark}\label{rmk:equivariant-bundle-map-between-fibres}
From $\pi(ge)=g\pi(e)$ we see that the $G$-action maps fibres to fibres:
\begin{equation}
    E_{gb}=\{e\in E:\pi(e)=gb\}=\{e\in E:g^{-1}\pi(e)=b\}=\{ge\in E:\pi(e)=b\}=g\{e\in E:\pi(e)=b\}=gE_b.
\end{equation}
\end{remark}
While a principal $H$-bundle gives a group structure to the fibres if we have a representation of $H$ on the vector space $V$ there is a very natural construction of a bundle which has typical fibre $V$ but incorporates the group structure from both the principal bundle and $V$. 
\begin{definition}[Associated bundle]\label{def:associated-vector-bundle}
Given a principal $H$-bundle $\pi_P:P\to B$ and a left $H$-space $M$, the total space of the associated bundle $P\times_H M$ is constructed by defining a right $H$-action on the free product $P \times M$ as 
\begin{equation}
    (p,m)\triangleleft h:=(ph,h^{-1}m)
\end{equation}
and imposing the equivalence relation
\begin{equation}
    (p,m)\sim (ph,h^{-1}m).
\end{equation}
Then the total space of the associated bundle is constructed through  taking the quotient by this equivalence relation:
\begin{equation}
    P\times _H M=(P\times M)/\sim.
\end{equation}
The projection to the base space is given by $\pi([p,m])=\pi_P(p)$ which is independent of the choice of representative of the coset since $\pi_P$ is independent of the $H$ action on the total space. 
Furthermore, the fibre of the associated bundle is $M$.
\end{definition}
\begin{remark}
A common use case is for $M$ to be a $H$-representation $(\rho,V_\rho)$ on a vector space $V_\rho$.
In those cases this is called an associated vector bundle.
Equipping the space of sections $\Gamma(P \times _H V_ \rho)$ with pointwise addition and multiplications by scalars in the field underlying $V_{\rho}$ turns this space into a vector space.
\end{remark}
\begin{remark}
   Regarding notation of the total space of the associated bundle.
   In the setting in which the space $ M $ used to construct the associated bundle is a vector space with a $ H $-representation $ \left( \sigma,V_\sigma \right) $ we denote the total space as $ P \times _\sigma V_\sigma $, while in the general setting where only a $ H $-action is present we leave the total space as $ P \times _H M $. 
\end{remark}
The construction of the associated bundle and its relation to the corresponding principal bundle can be represented diagramatically through the Cartan's mixing diagram which is a tool  used in equivariant cohomology. 
\begin{definition}[Cartan's mixing diagram \citep{tuIntroductoryLecturesEquivariant2020}]\label{def:cartan-mixing-diagram}
    Given a principal $H$-bundle $\alpha:P\to B$ and a left $H$-space $M$, Cartan's mixing diagram is the commutative diagram 
    \begin{equation}\label{diagram:cartans-mixing-diagram}
\begin{tikzcd}[cramped]
	P && {P\times M} && M \\
	\\
	B && {P\times_G M} && {H\backslash M}
	\arrow["\alpha"', from=1-1, to=3-1]
	\arrow["{\pi_1}"', from=1-3, to=1-1]
	\arrow["{\pi_2}", from=1-3, to=1-5]
	\arrow["\beta"{description}, from=1-3, to=3-3]
	\arrow["\gamma", from=1-5, to=3-5]
	\arrow["{\tau_\alpha}", from=3-3, to=3-1]
	\arrow["{\tau_\beta}"', from=3-3, to=3-5]
\end{tikzcd}
    \end{equation}
    where
    \begin{align*}
        \tau_\alpha([p,m])=\alpha(p),&&       \tau_\gamma([p,m])=\gamma(m)=Hm
    \end{align*}
    and $\beta$ is the quotient map of $P\times M$ with respect to the diagonal right $H$-action $(p,m)h=(ph,h^{-1}m)$.
    Note that $\tau_\gamma$ is well-defined since all representatives $(p,m)$ in the same equivalence class map to the same object $Hm$ in $H\backslash M$. 
    Moreover, by symmetry, if $\gamma$ is a principal $H$-bundle, then $\tau_\beta$ is a fibre bundle with fibre $P$.
\end{definition}
We can, in addition to the action of $H$ on the total space $P$, introduce a new group $G$ acting on both the total space $P$ and the base space $B$. This leads to the following definition of an equivariant principal bundle which will be relevant later.
\begin{definition}[$G$-equivariant  principal $H$-bundle \citep{zou2021notesequivariantbundles}]\label{def:equivariant-principal-bundle}
Let $G,H$ be groups.
Then $\alpha:P \to{}B$ is called a $G$-equivariant principal $H$-bundle if
\begin{enumerate}
    \item The map $\alpha$ is a principal $H$-bundle; 
    \item Both $P$ and $B$ are $G$-spaces and $\alpha$ is $G$-equivariant;
    \item The actions of $G$ and $H$ commute on $P$.
\end{enumerate}

\end{definition}

With the definition above of the associated vector bundle $ \pi_A:G \times _\rho V_\rho \to G/H $ it is natural to define a $ G $-action on the total space as $ k[g,v]=[kg,v] $.
Together with the principal bundle $ \pi_H:G\to G/H $, obtained from the homogeneous  space $ G/H ,$ and the $ G $-action on this space $ k(gH)=(kg)H $, it is easy to check that the $ G $-action commutes with the projection of the associated bundle: 
\begin{equation}
     \pi_A(k[g,v])=\pi_A([kg,v])=\pi_H(kg)=(kg)H=k(gH)=k\pi_H(g)=k \pi_A([g,v]).
\end{equation}
This shows that the associated vector bundle with this $ G $-action is a $ G $-equivariant associated vector bundle.

Now, since we're interested in sections of this associated vector bundle we need to make the sections compatible with this equivariant structure.
This is done by defining a suitable $ G $-action on the space of sections.

\begin{definition}[$G$-representation on sections of associated vector bundles]\label{def:G-rep-on-sections-associated-bundle}
    Given an associated vector bundle $G\times _\rho V_{\rho}\to G/H$ the space of sections $\Gamma(G\times_\rho V_{\rho})$ is naturally equipped with a $G$-representation\footnote{This construction is referred to as an \textit{isometry pushforward of sections} in \citet{weiler2023EquivariantAndCoordinateIndependentCNNs}.} acting as
\begin{equation}
    [\rho_\Gamma (g)f](x)=gf(g^{-1}x).  
\end{equation}

\begin{remark}
    It might not be immediately obvious that $gf(g ^{-1}x )$ is still in the fibre over $x$.
    For this purpose, note that $[\rho_\Gamma (g)f](x)\in(G\times_\rho V_{\rho})_x$ since $f(g^{-1}x)\in(G\times_\rho V_{\rho})_{g^{-1}x}$ and \cref{rmk:equivariant-bundle-map-between-fibres} says that $g:(G\times_\rho V_{\rho})_x\to (G\times_\rho V_{\rho})_{gx}$. 
    This gives that $g:(G\times_\rho V_{\rho})_{g^{-1}x}\to (G\times_\rho V_{\rho})_{x}$ and hence $[\rho_\Gamma (g)f](x)\in(G\times_\rho V_{\rho})_x$ showing that this action preserves the fibres.
\end{remark}
\end{definition}

Since the fibre of the associated vector bundle is isomorphic to the representation $(\rho,V_{\rho})$ it is natural to ask what this $G$-action on sections of the associated bundle looks like for sections $f_{V_{\rho}}:G/H\to E$ of a vector bundle with fibre $V_{\rho}$.

\begin{remark}
    For this process we need a local reference in the fibres of the $H$-principal bundle $G \to G/H$.
    This reference value is determined by a choice of local section $s: G/H \to G$ which amounts to a local choice of coordinates.
    The section exists globally if and only if the bundle $G \to G/H$ is trivial, however choosing this section locally and working with transition functions between different choices is perfectly doable and well studied.
    Due to this we omit these details here and refer the reader to, e.g., \cite{sontzPrincipalBundlesClassical2015}, for the technical details.
\end{remark}

To get the $G$-representation on the sections $f_{V_{\rho}}$ of the vector bundle $E_{V_\rho} \to G/H$ we start by writing a section $f$ of the associated bundle as $f(x)=[s(x),f_{V_{\rho}}(x)]$ where $s:G/H\to G$ is a local section and $f_{V_{\rho}}:G/H\to E_{V_\rho}$ is  a section of the vector bundle.

\begin{remark}\label{rmk:fibre-twist-map}
One important detail is the interaction between the section $s$ and the $G$-action; because, while $\pi(gs(x))=\pi(s(gx))$ in the base space, in the total space we have $gs(x)=s(gx)h$ for some $h\in H$.
Rearranging we see that this $h$ depends on $g$ and $x$ and we get the expression $\mathrm{h}(x,g)=s(gx)^{-1}gs(x)$.

This function explains the mismatch between $s(x)\in P_x$ and $s(g^{-1}x)\in P_{g^{-1}x}$ when $s(x)$ is pushed to $P_x$ via left multiplication with $g$.
Intuitively, $h(x,g)$ encodes the twist in the fibre and is visualised in \cref{fig:twist-in-fibre}. 
For more details, see \cite{cohen-theory-equivariant-hom}.\footnote{The element $ \mathrm{h}(x,g) $ is called \textit{(isometry) induced gauge transformation} in \cite{weiler2023EquivariantAndCoordinateIndependentCNNs}.}
\end{remark}
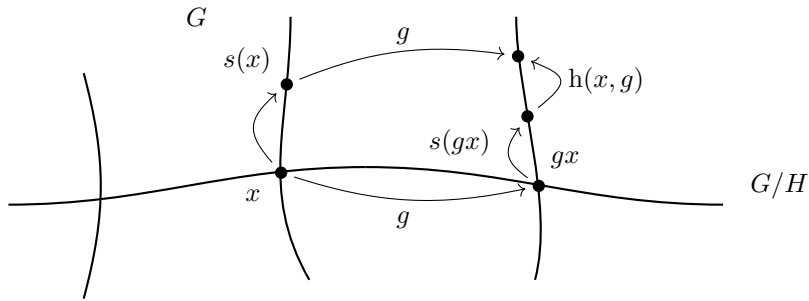
\begin{figure}[h]
\[ 
\begin{tikzpicture}
	\begin{pgfonlayer}{nodelayer}
		\node [style=none] (0) at (-7.5, 0.25) {};
		\node [style=none] (1) at (-2.75, 0.75) {};
		\node [style=none] (2) at (2, 0.25) {};
		\node [style=none] (3) at (2.75, 0.5) {$G/H$};
		\node [style=none] (4) at (-6.5, 2) {};
		\node [style=none] (5) at (-3.75, 2.75) {};
		\node [style=none] (6) at (-0.75, 2.75) {};
		\node [style=none] (7) at (-6.5, -1) {};
		\node [style=none] (8) at (-3.5, -0.75) {};
		\node [style=none] (9) at (-0.5, -0.75) {};
		\node [style=none] (10) at (-5, 2.75) {$G$};
		\node [style={small_dot}] (11) at (-0.45, 0.5) {};
		\node [style=none] (12) at (-0.1, 0.85) {$gx$};
		\node [style={small_dot}] (13) at (-0.6, 1.425) {};
		\node [style=none] (14) at (-1.5, 1.075) {$s(gx)$};
		\node [style={small_dot}] (19) at (-3.875, 0.675) {};
		\node [style={small_dot}] (20) at (-3.8, 1.85) {};
		\node [style=none] (21) at (-4.25, 0.375) {$x$};
		\node [style=none] (24) at (-4.325, 2.175) {$s(x)$};
		\node [style={small_dot}] (25) at (-0.725, 2.225) {};
		\node [style=none] (26) at (0.45, 1.9) {$\mathrm{h}(x,g)$};
		\node [style=none] (27) at (-2.25, 2.5) {$g$};
		\node [style=none] (28) at (-2.25, 0.05) {$g$};
	\end{pgfonlayer}
	\begin{pgfonlayer}{edgelayer}
		\draw [style={line_thick}, in=180, out=0] (0.center) to (1.center);
		\draw [style={line_thick}, in=180, out=0] (1.center) to (2.center);
		\draw [style={line_thick}, in=75, out=-75] (4.center) to (7.center);
		\draw [style={line_thick}, in=120, out=-90] (5.center) to (8.center);
		\draw [style={line_thick}, in=75, out=-90, looseness=0.75] (6.center) to (9.center);
		\draw [style={arrow_both_gap}, bend left=45, looseness=1.50] (11) to (13);
		\draw [style={arrow_both_gap}, bend left=45, looseness=1.50] (19) to (20);
		\draw [style={arrow_both_gap}, bend right=60, looseness=2.25] (13) to (25);
		\draw [style={arrow_both_gap}, bend left=15] (20) to (25);
		\draw [style={arrow_both_gap}, bend right=15] (19) to (11);
	\end{pgfonlayer}
\end{tikzpicture}
\]
    \caption{A visualisation of the map $ \mathrm{h}:G/H \times G\to H $ encoding the twist in the fibre.}\label{fig:twist-in-fibre}
\end{figure}

Sparing the reader the rest of the technical details, this results in the following representation on the sections of the vector bundle:
\begin{equation}\label{eq:rep-on-sections-of-vector-bundle}
(gf_{V_{\rho}})(x)=\rho(\mathrm{h}(x,g^{-1})^{-1})f_{V_{\rho}}(g^{-1}x).
\end{equation}

\subsection{Induced representations}\label{sec:induced-reps-and-frobenius-reciprocity}
Although it is intuitive to work with sections of a vector bundle over the homogeneous base space, this representation is rather cumbersome to work with as one needs to work with local coordinates and transition functions. 
Luckily there is a way to map these sections bijectively to functions $f:G\to V_{\rho}$ subject to the equivariance constraint 
\begin{equation}\label{eq:mackey-criteria}
    f(gh)=\rho(h^{-1})f(g),
\end{equation}
which utilises this transformation property to arrive at a representation free of local coordinates. 

Functions satisfying this criteria are sometimes called Mackey functions \citep{cohen-theory-equivariant-hom,cohen2018intertwinersinducedrepresentationswith} and we will use the same nomenclature. 
This space with the $G$-action $(gf)(k)=f(g^{-1}k)$ is a induced representation which is defined as follows:
\begin{definition}[The induced representation]\label{def:induced-representation}
Given a representation $(\rho, V_\rho)$ of a subgroup $H\leq G$ and a space of functions $\mcI:G\to V_\rho$, the induced representation is the subset 
\begin{equation}
	\label{eq:induced-representation}
    \mcI_\rho=\ind^G_H(\rho):=\{f\in \mcI:f(gh)=\rho(h^{-1})f(g)\}\subset \mcI
\end{equation}
together with the $G$-action
\begin{equation}
    \label{eq:induced-representation-G-action}
    [gf](k)=f(g^{-1}k).
\end{equation}
\end{definition}
\begin{remark}
The induced representation $\mcI_\rho$ is naturally equipped with pointwise addition and multiplication by scalars in the underlying field of $V_\rho$ which turns it into a vector space. 
\end{remark}
Now, it so happens that the space of such functions is isomorphic to the space of sections of the associated bundle.
\begin{proposition}[\cite{cohen-theory-equivariant-hom,kolarNaturalOperationsDifferential2013} \S 10.12]
	\label{prop:function-isomorphism}
Given a representation $(\rho, V_\rho)$ of a subgroup $H\subset G$ then there exists an equivariant bijection between $\mcI_\rho$ and $\Gamma(G\times _\rho V_\rho)$.  
\end{proposition}
\Cref{prop:function-isomorphism} is well illustrated using the Cartan mixing diagram as can be seen in \cref{fig:feature-maps-diagram}.
\begin{figure}[h]
\[
\begin{tikzcd}[cramped]
	G && {G\times V_\rho} && {V_\rho} \\
	\\
	{G/H} && {G\times_\rho V_\rho} && {H\backslash V_\rho}
	\arrow["{f_G}", curve={height=-24pt}, dashed, from=1-1, to=1-5]
	\arrow["{\pi_{H}}"{description}, from=1-1, to=3-1]
	\arrow["{\pi_1}"{description}, from=1-3, to=1-1]
	\arrow["{\pi_2}"{description}, from=1-3, to=1-5]
	\arrow["\beta"{description}, from=1-3, to=3-3]
	\arrow["{\pi_{H\backslash}}"{description}, from=1-5, to=3-5]
	\arrow["{f_X}", curve={height=-24pt}, dashed, from=3-1, to=3-3]
	\arrow["{\tau_{\pi_H}}"{description}, from=3-3, to=3-1]
	\arrow["{\tau_{\pi_{H\backslash}}}"{description}, from=3-3, to=3-5]
\end{tikzcd}
\]
\caption{The Cartan mixing diagram for the principal $H$-bundle $G \to G/H$ and the two equivalent representations of data: either  as an element $ f_G $ of the induced representation $ \mathcal{I}_\rho:\, G \,\to\, V_\rho $, or as a section $ f_X \,\in\, \Gamma(G \times _\rho V_\rho)  $ of the associated vector bundle.}
\label{fig:feature-maps-diagram}
\end{figure}

For the proof we refer the reader to the cited works.
However, the isomorphism used in the proof that will be useful later on and so for convenience we repeat it here.
Let $s:G/H\to G$ be a local section, and define two maps: one to lift sections of the associated bundle to the induced representation and one to map back down from the induced representation to sections of the associated bundle. 
Specifically, 
    \begin{align}
        &\Lambda^\uparrow:\Gamma(G\times_\rho V_\rho)\to \mcI_\rho&[\Lambda^\uparrow f_X](g)&=\rho(\mr{h}(g)^{-1})(f_X\circ \pi)(g)\label{eq:function_lift}\\
        &\Lambda^\downarrow:\mcI_\rho\to\Gamma(G\times_\rho V_\rho)&[\Lambda^\downarrow f_G](x)&=(f_G\circ s)(x),\label{eq:function_sink}
    \end{align}
    where $\mr{h}:G\to H$ acting as $\mr{h}(g)=s(gH)^{-1}g$.
    This construction results in $\Lambda^\uparrow=\Lambda^{\downarrow-1}$ and each choice of section $ s $ yields a different, but equivalent, lift.
This means that $\mcI_\rho$ and $\Gamma(G\times_\rho V_\rho)$ are different realisations of the induced representation $\ind^G_H\rho$ of the representation $(\rho, V_\rho)$, and hence we can work equally well with either.

Finally, we introduce some notation for the space of linear equivariant maps between induced representations.
\begin{definition}[Space of intertwiners]
    Let $\mcI_\rho$ and $\mcI_\sigma$ be induced representation of $H'$ on $(\rho,V_\rho)$ and $H$ on $(\sigma, V_\sigma)$ where $H,H'\subseteq G$, then we denote the space of all \emph{linear $G$-equivariant maps} $\Phi:\mcI_\rho\to\mcI_\sigma$ as $\hom_G(\mcI_\rho,\mcI_\sigma)$.
    Elements in $\hom_G(\mc{I}_\rho, \mc{I}_\sigma)$ are referred to as \emph{intertwiners} and are defined in \cref{def:equivariant-maps-intertwiner}. 
\end{definition}

Note the distinction between $\hom_G(\mathcal{I}_\rho,\mathcal{I}_\sigma)$, which consists of all $G$-equivariant linear maps, intertwiners, between $\mathcal{I}_\rho$ and $\mathcal{I}_\sigma$, of which $\Phi$ from the previous paragraph is an element, and $\hom(V_\rho, V_\sigma)$.
The elements of the latter, while being linear maps between two representations, are not assumed to be intertwiners themselves.

\subsection{Compactly supported functions}
In \cref{def:induced-representation} we refer to an arbitrary function space $ \mathcal{I}$ of which the induced representation $ \mathcal{I}_\rho$ is a subspace.
For our purposes the appropriate function space to consider will be the space of compactly supported smooth functions, $ \mathcal{I} = C^\infty_c(G, V)$.

\begin{definition}
\label{def:c-infinity-compact}
	The space $C^\infty_c(G,V)$ of compactly supported smooth functions from $G$ to $V$ is the space of all functions $f: G \to V$ where the closure of the support $ \operatorname{supp}(f) = \{g \in G \mid f(g) \neq 0\}$ is a compact subset of $G$ and and for which all derivatives $f^{(n)}$ are continuously differentiable.
If $(\rho, V_\rho)$ is a representation, then we will let $C^\infty_{c,\rho}(G,V_\rho)$ denote the induced representation $\ind ^G_H(\rho)$ as subset of $C^\infty_c(G,V_\rho)$.
\end{definition}

\begin{remark}
    Specifically, the projection from $ C _{c}^{\infty}(G,V_\rho) $ to $ C _{c,\rho}^{\infty}(G,V_\rho) \subset \mathcal{I}_\rho $ is given by 
    \begin{equation}
        P: C _{c}^{\infty}(G,V_\rho)\to C _{c,\rho}^{\infty}(G,V_\rho),\quad [P f](g)= \int_H \Delta_H(h)\rho(h^{-1})f(gh^{-1})\mathrm{d}h,
    \end{equation}
    where $ \Delta_H $ is the modularity function of $ H $, which is identically one for a unimodular group.
    Additionally, this projection is $ G $-equivariant in that $ [P[kf]](g)=[k[Pf]](g) $.
\end{remark}

Compactly supported functions enjoy many convenient properties in relation to integration, of which the following result will be particularly useful.

\begin{theorem}[Quotient Integral Formula]
	\label{thm:quotient-integral-formula}
	Assume that $G$ is a locally compact group with closed subgroup $H$ and $s: X \to G$ is a global section in the bundle $\pi: G \to X, g \mapsto g x_0$.
	Given Haar measures on $G$ and $H$, there exist a unique choice of measure on the homogeneous space $X \cong G/H$ such that for every function $f \in C^\infty_c(G,V)$ the following integral decomposes as
	\begin{equation}
		\label{eq:quotient-integral-formula}
	\int_G f(g) \mathrm{d}g = \int _X \int_H f(s(x)h) \mathrm{d}h \mathrm{d}x.
	\end{equation}
\end{theorem}
\begin{proof}
	See e.g. \citet{Deitmar2014} for proof.
\end{proof}
Note that in non-trivial bundles global sections do not generally exist.
Then $s(x)$ in \cref{eq:quotient-integral-formula} can then be replaced with any $g_x$ such that  $\pi(g_x) = x$, where $g_x$ does not have to depend continuously on $x$.
The reader is referred to \citet{Deitmar2014} for more details.

%% file: sections/linear-maps.tex
In this section we define feature maps and study equivariant operators between them.
On the neural network side, they represents data and equivariant network layers, respectively.
These constructions make use of the mathematical abstractions introduced in the previous chapter.
The content of this chapter is heavily inspired by, and a slight reformulation of, the material on $G$-CNNs in \citet{cohen-theory-equivariant-hom} and \citet{weiler2023EquivariantAndCoordinateIndependentCNNs}.
We present this framework based on linear operators here and, inspired by this, present a generalisation to non-linear maps in \cref{sec:nonlinear-maps}. 
Just as in the mentioned references we will generally work in the setting where $G$ is a unimodular group.

\subsection{Vector features on homogeneous spaces}
\label{sec:features}

We would like to describe  features as maps that assign to the points in the homogeneous space $ G/H $ some feature vector in the vector space $ V $.
Examples include image data, where $X$ is a grid of pixels and $Y$ is the vector space $ \mathbb{R}^3$ of three colour channels, and text data, where $X$ is an index set of integers and $Y$ is the space of all valid characters.\footnote{It is common to embed categorical data in a vector space, e.g., using \texttt{word2vec}.}

Unfortunately, it is not possible to have our features be global functions $ f:\, G/H \,\to\, V $ unless the vector bundle $ E_V \to G/H $ is trivial.\footnote{In order to have a global function requires a global continuous choice of coordinates, a global continuous section $ s:\, G/H \,\to\, G $, and this only exists if the principal bundle $ G \to G/H $ is trival.}
Instead we first need local coordinates which are expressed through a choice of a local section $ s:U \to G $, where $ U $ is a trivialising neighbourhood in $ G/H $. 
Once such a section is chosen we can express the feature map in these local coordinates, and this is generally what is done in implementations.

Working with such local descriptions of features necessitates the use of transition function when moving between areas with different local coordinates.
While this is standard, doing pen and paper work with these are somewhat annoying.
Luckily, as was described in \cref{sec:induced-reps-and-frobenius-reciprocity}, we can use the choice of local coordinates $ s:G/H \to G $ to lift the local features to a global, essentially coordinate free, perspective as elements of the induced representation.
In that setting the transformation of the features are encoded as a $ H $-redundancy through the Mackey condition: $ f(gh)=\rho(h^{-1})f(g) $.

This redundancy makes the features easy to work with from a theoretical standpoint, but if used in a machine learning model this would result in much higher memory use by an order $ \left| H \right|  $.
Hence, it is more memory efficient to work with features in local coordinates.

\begin{definition}[Features on a homogeneous space]
	The feature maps on $G/H$ are defined as the sections of the associated vector bundle $\Gamma(G \times_\rho V_\rho)$.
\end{definition}

\begin{definition}[Features on a group]
	The feature maps on $G$ are defined as the elements in the $ G $-representation space $\mathcal{I}_\rho$ induced from the representation $\rho$ of $ H $ on features $V_\rho$.
    We will mainly consider the case $ \mathcal{I} = C^\infty_{c,\rho}(G,V_\rho)$
    as given in \cref{def:c-infinity-compact}.
\end{definition}
Note that, as established in \cref{prop:function-isomorphism}, our two notions of feature maps are isomorphic.
The feature maps fit into the Cartan mixing diagram introduced in the previous section. 
The relevant Cartan mixing diagram is given in \cref{fig:feature-maps-diagram}.

\subsection{Linear operators on feature vector fields}
\label{sec:lin-ops-on-vec-feats}

We now move onto studying operators which are $G$-equivariant homomorphisms between the function spaces defined in the previous section.

Again, we require that relevant operators are compatible with the group actions on their input and output, which in this case are group actions on (function spaces of) feature vector fields.
These operators model layers in neural networks in the same way as the maps of the previous section model data, 
and we will see that we get two different criteria on the integration kernels: one from requiring equivariance and a compatibility condition coming from the redundancy in Mackey functions.

As noted in \cref{sec:homogeneous-spaces-and-equivariance}, the spaces $ \mathcal{I}_\rho$ and $\Gamma(G \times_H V_\rho)$ of feature maps are both vector spaces.
In the present section it is assumed that all operators are linear maps on these spaces.
However, this linearity assumption is lifted in the subsequent \cref{sec:nonlinear-maps}, which makes up the main contribution of this article.

More specifically, we begin by investigating (not necessarily equivariant) linear maps $\Phi$ between induced representations $\mcI_\rho$ and $\mcI_\sigma$ and derive the fact that these can be written as an integral with a two-argument kernel.
We will see that the space of those two-argument integration kernels can itself be thought of as an induced representation of linear maps between the representations $V_\rho$ and $V_\sigma$. We will later impose the equivariance condition on the map $\Phi:\mcI_\rho\to\mcI_\sigma$ in \cref{sec:equivariant-lin-operators}. 
This will allow us to define the kernels on a sequence of domains with a rich geometrical structure.

The material of this section concerns the $G$-equivariant linear maps defined in \cite{cohen-theory-equivariant-hom} and repeated in \cref{def:gcnn-layer}.
We add additional context and detail not present in \cite{cohen-theory-equivariant-hom} on the space on the kernels (see \cref{lemma:induced-image-constraint,lemma:kernel-redundancy}), their domains (see, e.g., the text below \cref{prop:double-coset-kernels} and the text around \cref{fig:kernel-on-GxG-factors}), and how they connect into the same general framework as the feature maps (see \cref{cor:kernel-isomorphism}).

To be able to develop the mathematical side of our results, we will have to make some technical assumptions on the feature maps $f$.
In particular, we restrict attention to the space $C^\infty_c(G, V)$ of smooth maps of compact support between $G$ and $V$.
This restriction is natural, as in applications, the feature maps are often zero outside some finite region.
We will also assume that the group $G$ is locally compact.
These restrictions are in line with other works of similar focus, see e.g. \citet{aronssonHomogeneousVectorBundles2022} and \citet{jenner2021steerable}.
Integration kernels, which will play a central role in this section, then naturally lie in the dual space $ \mathcal{D}'(G,V) = (C^\infty_c(G,V))'$, i.e. the space of vector-valued distributions.

In this and subsequent sections, we will let $ C^\infty_{c,\rho}(G, V_\rho)$ and $C^\infty_{c,\sigma}(G, V_\sigma)$ denote the subspaces of induced representations in $C^\infty_{c}(G, V_\rho)$ and $C^\infty_c(G, V_\sigma)$ respectively.
This is in line with the notation introduced in \cref{def:induced-representation}.
Note that the restriction to compactly supported feature maps preserves the bijection to (compactly supported) sections of an associated vector bundle presented in \cref{prop:function-isomorphism}.
On this form, we can study which restrictions on the kernel $\kappa$ are necessary for the resulting operators to have desired properties.
In particular, we want the operator to preserve the condition $f(gh) = \rho(h^{-1})f(g)$, i.e. map induced representations to other induced representations.
However we do allow the representations to be induced from two different subgroups $H,H' \leq G$.

Letting the subgroups to be different allows to study the more general case of $G$ equivariant maps between $\ind^G_{H}(\sigma)$ and $\ind^G_{H'}(\rho)$.
In applications, however, the most common case is that one takes the subgroups to be the same.

\begin{lemma}
	\label{lemma:induced-image-constraint}
    Consider a subgroup $H \subseteq G$ and a representation $\sigma$ of $H$ on a vector space $V_\sigma$ and a function $f \in C^\infty_c(G, V')$, where $V'$ does not necessarily need to be a representation.
    Then the function $G\to V_\sigma$ given by
    \begin{equation}
        g \mapsto \int_G \kappa(g, g') f(g') \mathrm{d}g'
    \end{equation}
    lies in the induced representation $\mathcal{I}_\sigma$ if and only if the map $\kappa:G\times G\to\hom(V',V_\sigma)$ satisfies $\kappa(gh, g') = \sigma(h^{-1})\kappa(g, g')$ for all $h \in H$.
\end{lemma}
\begin{remark}
    Although this lemma is presented in a quite general setting, the most common case here is that $f$ is in the induced representation $\mathcal{I}_\rho$ of the representation $(\rho,V_\rho)$ of $H'$ since we want to in the end study the $G$-equivariant maps between induced representations.
\end{remark}
\begin{proof}
    By definition we have $\int_G \kappa(g, g') f(g') \mathrm{d}g' \in \mathcal{I}_\sigma$ if and only if
    \begin{equation}
    \begin{aligned}
	    \int_G \kappa(gh, g') f(g') \mathrm{d}g' &= \sigma(h^{-1})\int_G \kappa(g, g') f(g') \mathrm{d}g' \\
						     &= \int_G \sigma(h^{-1})\kappa(g, g') f(g') \mathrm{d}g'
    \end{aligned}
    \end{equation}
    for all $h \in H, g \in G$.
    Since the function $f$ is arbitrary, we conclude
     \begin{equation}
	     \kappa(gh, g') = \sigma(h^{-1})\kappa(g, g')
    \end{equation}
    for all $g \in G, h \in H$ and almost all $g' \in G$.
\end{proof}

The next result shows that if the function $f$ in the integrand lies in the induced representation then there exist a redundancy in the space of kernels.
We can remove this redundancy by picking representatives from each equivalence class of kernels.
\begin{lemma}
	\label{lemma:kernel-redundancy}
	Consider subgroups $H, H' \leq G$, representations $\rho, \sigma$ of $H', H$ on the vector spaces $V_\rho, V_\sigma$, respectively, and the induced representation $C^\infty_{c,\rho}(G, V_\rho)$ of $H'$.
    Let $\kappa, \kappa': G \times G \to \hom(V_\rho, V_\sigma)$, define an equivalence condition on kernels
    \begin{equation}
	    \label{eq:equivalent-kernels}
        \kappa \sim \kappa' \iff \int_{H'} \kappa(g, g'h)\rho(h^{-1}) \mathrm{d}h = \int_{H'} \kappa'(g, g'h)\rho(h^{-1}) \mathrm{d}h
    \end{equation}
    and consider the equivalence classes $[\kappa]$ generated by this relation.
    Then all kernels in the same class $[\kappa]$ define the same integral operator
    \begin{equation}
        f(g) \mapsto \int_G\kappa(g, g')f(g') \mathrm{d}g'
    \end{equation}
    on elements $f \in C^\infty_{c,\rho}(G, V_\rho)$ of the induced representation, and in each equivalence class $[\kappa]$ there exist exactly one element $\kappa_0$ such that
    \begin{equation}
    \kappa_{0}(g, g'h) = \kappa_{0}(g, g')\rho(h).
    \end{equation}
\end{lemma}
\begin{proof}
	Given Haar measures $ \mathrm{d}h'$ on $H'$ and $ \mathrm{d}g'$ on $G$, \cref{thm:quotient-integral-formula} shows that there exist a unique measure on $X \isom G/H$ such that
	\begin{equation}
		\int_G\kappa(g, g')f(g') \mathrm{d}g' = \int_X \left[ \int_{H'} \kappa(g, s(x)h')f(s(x)h') \mathrm{d}h' \right] \mathrm{d}x
	\end{equation}
	for some $s: X \rightarrow G$ which picks a representative from each equivalence class $[g]$ associated to $x \in X$.
	The fact that $f(gh) = \rho(h^{-1})f(g)$ implies that $f$ is in fact independent of $h'$ and can be moved out of the inner integral
	\begin{equation}
		\int_X \left[ \int_{H'} \kappa(g, s(x)h')f(s(x)h') \mathrm{d}h' \right] \mathrm{d}x
		= \int_X \left[ \int_{H'} \kappa(g, s(x)h')\rho(h'^{-1})  \mathrm{d}h' \right] f(s(x)) \mathrm{d}x.
	\end{equation}
	This shows that two kernels $\kappa, \kappa'$ equivalent under the relation in \cref{eq:equivalent-kernels} yield the same operator.
	It is straightforward to show that the relation $\sim$ has the three defining properties of an equivalence relation: reflexivity, symmetry and transistivity.
	Further, if $\kappa_0 \sim \kappa$, then $\kappa_0(g, g'h) = \kappa(g,g')\rho(h)$ if and only if
	\begin{equation}
		\begin{aligned}
			\int_{H'} \kappa(g, g'h)\rho(h^{-1}) \mathrm{d}h &= \int_{H'} \kappa_0(g, g'h)\rho(h^{-1}) \mathrm{d}h \\
									 &= \int_{H'} \kappa_0(g, g')\rho(h)\rho(h^{-1}) \mathrm{d}h \\
									 &= \kappa_0(g, g') \int_{H'} \mathrm{d}h \\
									 &= \kappa_0(g, g').
		\end{aligned}
	\end{equation}
	Hence, the map $\kappa \mapsto \kappa_0$ to the representative in $[\kappa]$ is given by
	\begin{equation}
	\kappa_0(g, g') =\int_{H'} \kappa(g, g'h)\rho(h^{-1}) \mathrm{d}h.
	\end{equation}
\end{proof}
The redundancy in kernel space formalised by \cref{lemma:kernel-redundancy} is an artifact of the lift from the homogeneous space $X \isom G/H$ to the full group $G$.
And just like with the induced representations, we can fix a representative in the equivalence class $[\kappa]$ by fixing the values of $\kappa$ on $H$-orbits using a representation $\rho$.

The two lemmas inform our definition of kernels of interest, which we choose such that they are in bijection with linear operators on induced representations.
\begin{definition}[Two argument integral kernel]\label{def:two-arg-int-kernel}
Consider two, not necessarily distinct, subgroups $H,H' \leq G$.
    Let $(\rho, V_\rho)$  and $(\sigma, V_\sigma)$ be two representations of $H'$ and of $H$ respectively.
    Then we define the space of two argument kernels as
    \begin{equation}\label{eq:two-arg-kernel-constraint}
        \mc{K}=\{\kappa:G\times G\to \hom(V_\rho,V_\sigma)\ |\ \kappa(gh,g'h')=\sigma(h^{-1})\kappa(g,g')\rho(h'),\ \forall g,g'\in G, h\in H, h'\in H'\}.
    \end{equation}
\end{definition}

\begin{theorem}
	\label{thm:non-equivariant-kernels}
    Consider two representations $(\rho, V_\rho)$ of $H'$ and $(\sigma, V_\sigma)$ of $H$ for two subgroups $H, H' \leq G$.
	A linear operator $\Phi\in \hom( C^\infty_{c,\rho}(G, V_\rho), C^\infty_{c,\sigma}(G, V_\sigma))$ between the spaces of compactly supported smooth induced representations can always be written as an integral operator
    \begin{equation}\label{eq:two-argument-kernel-integral}
		[\Phi f](g) = \int_G \kappa(g, g') f(g') \mathrm{d}g'
	\end{equation}
    for some integral kernel $\kappa\in \mc{K}$.
	Furthermore, the space of such linear operators and the space of such kernels are isomorphic.
\end{theorem}
\begin{proof}
	In broad terms, the fact that $\Phi$ can be written on integral form follows from the Schwartz Kernel Theorem in combination with \cref{lemma:induced-image-constraint,lemma:kernel-redundancy}.
	The isomorphism as vector spaces follows from the linearity of the integral $\int_G$.
    For details of the proof we refer the reader to \citet{aronssonHomogeneousVectorBundles2022}.
\end{proof}

\begin{remark}
    Despite the $H, H'$-compatibility condition in \cref{eq:two-arg-kernel-constraint}, $\Phi$ is not assumed or implied to be equivariant but remains an unconstrained linear map.
    Its role is merely to ensure that $\Phi f$ respects the Mackey condition on induced representation in \cref{eq:induced-representation} if $f$ does.
\end{remark}

One might spot similarities between the $H,H'$-compatibility contraint on $\kappa$ and the equivariance constraint on the induced representation of \cref{eq:mackey-criteria}.
This turns out to be more than just a similarity; the kernels as defined in \cref{def:two-arg-int-kernel} are in fact induced representations themselves.
To see this we first need to define the representation on the $\hom$ space.

\begin{definition}[Representation on the hom space]\label{def:hom-representation}
    Given $L\in \hom(V_\rho,V_\sigma) \cong V_\sigma\otimes V_\rho^*$, we define the representation $\sigma\otimes\rho^*$ of $H\times H'\leq G\times G$ on $L$ as\footnote{This is similar to the definition of kernel fields as $ M $-morphisms in \cite{weiler2023EquivariantAndCoordinateIndependentCNNs}, see definition 12.2.1.}
\begin{equation}
    (\sigma\otimes \rho^*)(h,h')L=\sigma(h)L\rho(h'^{-1}).
\end{equation}
\end{definition}
With this definition we can now make the precise statement.
\begin{proposition}\label{prop:two-arg-kernel-is-induced-rep}
	Given the representation $(\sigma\otimes\rho^*, \hom(V_\rho,V_\sigma))$ we have
    \begin{equation}
        \mc{I}_{\sigma\otimes\rho^*}\isom \mc{K}.
    \end{equation}
    That is, the space of two argument integral kernels as defined in \cref{def:two-arg-int-kernel} is the induced representation of $(\sigma\otimes\rho^*,\hom(V_\rho,V_\sigma))$.
\end{proposition}
\begin{proof}
    For proof see \cite{cohen-theory-equivariant-hom}.
\end{proof}

As we have now established that the kernels of linear integral operators are induced representations, we can use \cref{prop:function-isomorphism} to write down an equivalent space in terms of sections in a bundle.
\begin{corollary}\label{cor:kernel-isomorphism}
	The space of integral kernels of linear operators are isomorphic to the space of section of the associated vector bundle
	\begin{equation}
		(G \times G) \times _{\sigma\otimes\rho^*} \hom(V_\rho,V_\sigma) \rightarrow (G \times G)/(H \times H').
	\end{equation}
\end{corollary}
\begin{proof}
    From previous results have that the induced representation $\ind _{H}^{G}(\eta)$ is isomorphic to the space of sections of the corresponding associated bundle $\Gamma(G \times _\eta V_\eta)$.
    This statement immediately follows since $\hom(V_\rho,V_\sigma)$ is a $\eta$-representation of $H \times H^{\prime}$ through $\eta=\sigma \otimes \rho^*$, and $H \times H^{\prime} \leq G \times G$.
\end{proof}
These results can be conveniently represented by a Cartan mixing diagram, as presented in \cref{fig:mixing-diagram-linear-kernels}.
This diagram can be directly compared with \cref{fig:feature-maps-diagram} of the previous section.

\begin{figure}[h]
\[
\begin{tikzcd}[cramped]
	{G\times G} && {(G\times G)\times \mathrm{Hom}(V_\rho,V_\sigma)} && {\mathrm{Hom}(V_\rho,V_\sigma)} \\
	\\
	{(G\times G)/(H\times H')} && {(G\times G)\times_{\sigma\otimes \rho^*}\mathrm{Hom}(V_\rho,V_\sigma)} && {\mathrm{Hom}(V_\rho,V_\sigma)/(H\times H')}
	\arrow["\kappa", curve={height=-24pt}, dashed, from=1-1, to=1-5]
	\arrow["{\pi_{H \times H'}}"{description}, from=1-1, to=3-1]
	\arrow["{\pi_1}"{description}, from=1-3, to=1-1]
	\arrow["{\pi_2}"{description}, from=1-3, to=1-5]
	\arrow["\beta"{description}, from=1-3, to=3-3]
	\arrow["{\pi_{\operatorname{Hom}}}"{description}, from=1-5, to=3-5]
	\arrow["{\kappa_X}", curve={height=-18pt}, dashed, from=3-1, to=3-3]
	\arrow["{\tau_{\pi_{H\times H'}}}", from=3-3, to=3-1]
	\arrow["{\tau_{\pi_{\operatorname{Hom}}}}"', from=3-3, to=3-5]
\end{tikzcd}
\]
\caption{The Cartan mixing diagram perspective on equivalent formulations of the integral kernels of \cref{thm:non-equivariant-kernels}.
    Note that all solid arrows commute.
The kernel can be viewed either as an element of the sections of an associated vector bundle $ \kappa_X\in\Gamma((G\times G)\times_{\sigma\otimes \rho^*}\mathrm{Hom}(V_\rho,V_\sigma)) $ or, equivalently, as elements of the induced representation $ \kappa\in\mathcal{I}_{\sigma \otimes \rho^*} $.}
\label{fig:mixing-diagram-linear-kernels}
\end{figure}

\subsection{Kernels of equivariant linear operators}\label{sec:equivariant-lin-operators}
We now impose the equivariance condition on the map $\Phi:\mcI_\rho\to\mcI_\sigma$ between the representations $\mcI_\rho$ induced by $H'$ acting on $(\rho,V_\rho)$ and $\mcI_\sigma$ induced by $H$ acting on $(\sigma, V_\sigma)$.
Since the $G$-action on these spaces, as defined in \cref{eq:induced-representation-G-action} and repeated here for convenience, is
\begin{equation}
    [gf](k)=f(g^{-1}k),
\end{equation}
we get that for $\Phi$ to be equivariant it must satisfy
\begin{equation}\label{eq:equivariance-for-lin-map}
    (k[\Phi f])(g)=[\Phi f](k^{-1}g)=[\Phi (kf)](g).
\end{equation}
This puts an additional constraint on the kernels.
\begin{proposition}[Theorem 3.1 in \cite{cohen-theory-equivariant-hom}]\label{prop:gcnn-integral}
    Given an equivariant linear operator $\Phi:\mcI_\rho\to\mcI_\sigma$ between induced representations, then the kernel $\kappa$ in the integral formulation of $\Phi$, defined in \cref{eq:two-argument-kernel-integral}, is constant on $G$-orbits in $G\times G$:
\begin{equation}
\kappa(u(g,g'))=\kappa(ug,ug')=\kappa(g,g'), \quad \forall u\in G.
\end{equation}
Furthermore, the map $\Phi$ can be written
	\begin{equation}
        [\Phi f](g) = \int_G \hat\kappa(g^{-1} g') f(g') \mathrm{d}g'
	\end{equation}
    for a single argument kernel$\,$\footnote{These are called \textit{isometry invariant kernel fields} in \cite{weiler2023EquivariantAndCoordinateIndependentCNNs}, see definition 13.2.3.} $\hat\kappa:G\to \hom(V_\rho,V_\sigma)$ satisfying 
    \begin{equation}\label{eq:one-arg-kernel-constraint}
		\hat\kappa(h gh') = \sigma(h)\hat\kappa(g) \rho(h'), \quad \forall h\in H,\ h'\in H'.
    \end{equation}
\end{proposition}

\begin{proof}
    For the full proof see \cite{cohen-theory-equivariant-hom}.
    
Although we refer the reader to \cite{cohen-theory-equivariant-hom} for the proof, we reproduce the construction of the one-argument kernel as it will be relevant later.

To get the one-argument kernel note that the invariance of the kernel on the $ G $-orbits $ \kappa(u(g,g^{\prime}))=\kappa(g,g^{\prime}) $ allows the definition of the one-argument kernel as $\kappa(g,g')=\kappa(e,g^{-1}g')=:\hat\kappa(g^{-1}g')$.
\end{proof}

\begin{remark}
    Note that many properties of the two argument integral kernel $\kappa(g, g')$ also hold for the single argument kernel $\hat{\kappa}(g)$.
    In particular, there exist many kernels without the property $\hat{\kappa}(gh') = \hat{\kappa}(g)\rho(h')$ that yield equivariant operators under the integral sign as was shown in \cref{lemma:kernel-redundancy}.
    However, this condition picks out one kernel per equivalence class so that there exist a bijection between kernels and operators. 
\end{remark}

The enforcing of analytical properties, e.g. equivariance, restricts the underlying geometry. 
    This insight is something we have not seen in other works to date.
    Since $\kappa\in \mcI_{\rho\times \sigma}$ there is a well defined representation of $G\times G$ as 
    \begin{equation}
        [(u,k)\kappa](g,g')=\kappa(u^{-1}g,k^{-1}g').
    \end{equation}
    In this light, the previous $G$-action on $G\times G$ as $ k(g,g^{\prime})=(kg,kg^{\prime}) $ is really the action of the diagonal subgroup $\{(g,g)\in G\times G\ |\ g\in G\}\isom G$. 
    Motivated by the constancy of the kernels on the orbits of the diagonal subgroup, we can introduce an equivalence relation 
    \begin{equation}
    (g,g')\sim (k,k') \quad \text{iff}\quad \exists u\in G: (k,k')=(ug,ug').
    \end{equation}
    Taking the quotient with respect to this equivalence relation yields a projection $G\times G\to G$ as $\pi_G((g,g'))= g^{-1}g'$.
    Which is exactly the input to the one-argument kernel, and hence we can write $\kappa=\hat\kappa\circ\pi_G$ resulting in the commutative diagram in \cref{fig:kernel-on-GxG-factors}.
    \begin{figure}[h]
        \[
\begin{tikzcd}[ampersand replacement=\&,cramped]
	{G\times G} \&\& {\mathrm{Hom}(V_\rho,V_\sigma)} \\
	\\
	G
	\arrow["\kappa"{description}, from=1-1, to=1-3]
	\arrow["{\pi_G}"{description}, from=1-1, to=3-1]
	\arrow["{\hat\kappa}"{description}, from=3-1, to=1-3]
\end{tikzcd}
\]
            \caption{The kernel $\kappa$ factors through the projection $\pi_G:G\times G\to G$ as $ \kappa=\hat{\kappa} \circ \pi_G $.}\label{fig:kernel-on-GxG-factors}
    \end{figure}
    That is, $\kappa$ factors through the qoutient, or projection, map $\pi_G$.

    Furthermore, since we had a right action of $H\times H'$ on $G\times G$ this projection gives a natural definition of an $H\times H'$ action on $G$ through $(g,g')\triangleleft(h,h')=(gh,g'h')\overset{\pi}{\mapsto} h^{-1}g^{-1}g'h'$.
    Hence defining a right $H\times H'$ action on $G$ as $g\triangleleft(h,h^{\prime})=h^{-1}gh'$ yield that this commutes with the projection map $\pi_G((g,g')\triangleleft(h,h'))=\pi_G((g,g'))\triangleleft(h,h')$.

    Finally, it is not difficult to show that $\pi_G:G\times G\to G$ as $\pi_G(g,g')=g^{-1}g'$ forms a principal $G$-bundle, and with the $H\times H'$ action this becomes a $H\times H'$-equivariant principal $G$-bundle, see \cref{def:equivariant-principal-bundle}. 
    Worth noting is that here, for the sake of convenience, the $G$-action on the total space $G\times G$ is defined as a left action rather than a right action, but this does not matter for the details.

Now, a reasonable question is if the space of kernels on $G$ spans the space of linear $G$-equivariant maps between induced representations.
As is formalised by the following theorems, the answer to the question turns out to be affirmitive.
These results can be regarded as specialisations of \cref{thm:non-equivariant-kernels} in the equivariant setting.
\begin{theorem}[Theorem 3.2 in \cite{cohen-theory-equivariant-hom}]
\label{thm:equivariant-kernels}
	The space of intertwiners $\hom_G(\mc{I}_\rho, \mc{I}_\sigma)$ is isomorphic to the space of bi-equivariant kernels on $G$, defined as
    \begin{equation}
        \mc{K}_G=\{\hat\kappa:G\to \hom(V_\rho,V_\sigma) \mid\hat\kappa(hgh') = \sigma(h)\hat\kappa(g) \rho(h'),\ \forall g\in G, h\in H, h'\in H'\},
    \end{equation}
    and hence linear equivariant maps between induced representations are equivariant convolutions.
\end{theorem}
\begin{remark}
   This theorem was investigated and expanded further in \cite{aronssonHomogeneousVectorBundles2022} as, without proper care, the map $ \id:\mathcal{I}_\rho \to \mathcal{I}_\rho $, which is obviously an equivariant map, cannot be written as a convolution.
\end{remark}

Based on this theorem we make the following definition, following \cite{cohen-theory-equivariant-hom}.
\begin{definition}[$G$-CNN layer]\label{def:gcnn-layer}
A $G$-CNN layer is a linear $G$-equivariant map $\Phi:\mcI_\rho\to \mcI_\sigma$ of the form
\begin{equation}\label{eq:gcnn-layer}
    [\Phi f](g)=[\kappa\star f](g)=\int_G\hat\kappa(g^{-1}g')f(g')\mathrm{d}g',
\end{equation}
where $ \kappa\in \mathcal{K}_G $.
\end{definition}

Imposing the equivariance property thus affects the underlying geometry on the domain of the kernels.
We can however go further by decomposing $g\in G$ as $g=s(gH)\mr{h}(g)$ where $s:G/H\to G$ is a section and $\mr{h}(g)\in H$ is defined through $\mr{h}(g)=s(gH)^{-1}g$.
The decomposition is a special case of the more general map $\mr{h}:G/H\times G \to H$ defined as $\mathrm{h}(x,g)=s(gx)^{-1}gs(x)$ evaluated at $\mr{h}(eH,g)=s(gH)^{-1}gs(eH)$, and where, without loss of generality, we choose $s$ such that $s(eH)=e$.
For more details, see the discussion in \cref{rmk:fibre-twist-map}. 
This results in a isomorphism between the space of one-argument kernels defined on $G$ to a space of kernels defined on the quotient $G/H$. 
\begin{remark}\label{rmk:lin-domain-reduction-non-equiv}
 Note that the isomorphism between the kernels defined on $ G $ and the kernels defined on $ G/H^{\prime} $ only relies on the transformation properties of the kernel under the right action of $ H^{\prime} $ on the argument.
 Hence, one can do a similar construction in the non-equivariant case to reduce from kernels $ \kappa:\, G \times G \,\to\, \hom(V_\rho,V_\sigma) $ to kernels $ \kappa_X:\, G \times G/H^{\prime} \,\to\, \hom(V_\rho,V_\sigma) $.    
\end{remark}

\begin{proposition}[Theorem 3.3 in \cite{cohen-theory-equivariant-hom}]\label{prop:g-cnn-kernel-on-hom-space}
    The space of intertwiners $\hom_G(\mc{I}_\rho, \mc{I}_\sigma)$ is isomorphic to the space of left-equivariant kernels on $X\isom G/H'$, defined as
\begin{equation}
    \mc{K}_X=\{\lak:X\to \hom(V_\rho,V_\sigma) \mid\lak(hx)=\sigma(h)\lak(x)\rho(\mr{h}'(x,h)^{-1}),\ \forall x\in X,h\in H\}.
\end{equation}
\end{proposition}

We can further relate the space of linear $G$-equivariant maps to a space of kernels on the double coset space $H\backslash G/H'$.

\begin{proposition}[Theorem 3.4 in \cite{cohen-theory-equivariant-hom}]\label{prop:double-coset-kernels}
    The space of intertwiners $\hom_G(\mc{I}_\rho, \mc{I}_\sigma)$ is isomorphic to the space of $H^{\gamma(x)H}$-equivariant kernels on $H\backslash G/H'$, defined as
    \begin{equation}
        \mc{K}_D=\{\overline{\kappa}:H\backslash G/H'\to \hom(V_\rho,V_\sigma) \mid\overline{\kappa}(x)=\sigma(h)\overline{\kappa}(x)\rho^x(h)^{-1},\ \forall x\in H\backslash G/H', h\in H^{\gamma(x)H}\},
    \end{equation}
    where $\gamma:H\backslash G/H'\to G$ is a section (choice of double coset representative), and $\rho^x$ is a representation of the stabiliser $H^{\gamma(x)H}=\{h\in H \mid h\gamma(x)H'=\gamma(x)H'\}\leq H'$ defined as
    \begin{equation}
        \rho^x(h)=\rho(\mr{h}'(\gamma(x)H',h))=\rho(\gamma(x)^{-1}h\gamma(x)).
    \end{equation}
\end{proposition}

In words, the Mackey property of the kernels with respect to the subgroups $H,H'$ allows us to reduce the domains of the kernels from kernels on $G$ to kernels on $X\isom G/H'$, and finally to kernels on the double coset space $H\backslash G/H'$. 
Graphically this domain reduction can be expressed in the diagram shown in \cref{fig:domain-reduction-gcnn-case}.
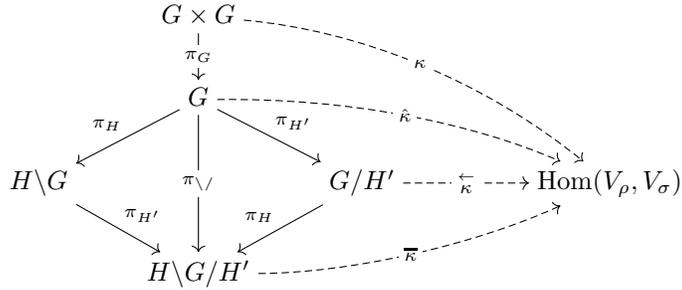
\begin{figure}[h]
    \[
\begin{tikzcd}[cramped]
	& {G\times G} \\
	& G \\
	{H\backslash G} && {G/H'} && {\mathrm{Hom}(V_\rho,V_\sigma)} \\
	& {H\backslash G/H'}
	\arrow["{\pi_G}"{description}, from=1-2, to=2-2]
	\arrow["\kappa"{description}, curve={height=-18pt}, dashed, from=1-2, to=3-5]
	\arrow["{\pi_{H}}"', from=2-2, to=3-1]
	\arrow["{\pi_{H'}}", from=2-2, to=3-3]
	\arrow["{\hat\kappa}"{description}, curve={height=-12pt}, dashed, from=2-2, to=3-5]
	\arrow["{\pi_{\backslash/}}"{description}, from=2-2, to=4-2]
	\arrow["{\pi_{H'}}", from=3-1, to=4-2]
	\arrow["{\overset{\leftarrow}{\kappa}}"{description}, dashed, from=3-3, to=3-5]
	\arrow["{\pi_{H}}"', from=3-3, to=4-2]
	\arrow["{\overline{\kappa}}"{description}, curve={height=12pt}, dashed, from=4-2, to=3-5]
\end{tikzcd}
    \]
    \caption{Graphical representation of the domain reduction for the kernels. 
        The kernel can be viewed either as a two-argument kernel $ \kappa\in \mathcal{K} $, as a one-argument kernel $ \hat{\kappa}\in \mathcal{K}_G $, as a kernel on the homogeneous space $ \overset{ \leftarrow }{\kappa}\in \mathcal{K}_X $ or as a kernel on the double coset space $ \overline{\kappa} \in \mathcal{K}_D $.
        All these kernel spaces are isomorphic.
    Note that all solid arrows commute.}\label{fig:domain-reduction-gcnn-case}
\end{figure}
\begin{remark}
    Although the domains are related with simple projections, the isomorphisms linking the space of integration kernels are more intricate.
\end{remark}

This can be collected into a single diagram based on the Cartan mixing diagram in \cref{fig:mixing-diagram-linear-kernels} with the addition of the domain reduction diagram in \cref{fig:domain-reduction-gcnn-case}. 
The resulting diagram is shown in \cref{fig:complete-gcnn-diagram}.
    \begin{figure}[h]
\begin{adjustwidth}{-2cm}{-1cm}
    \centering
    \[
\begin{tikzcd}[cramped]
	& G && {G\times G} && {(G\times G)\times \mathrm{Hom}(V_\rho,V_\sigma)} && {\mathrm{Hom}(V_\rho,V_\sigma)} \\
	{H\backslash G} && {G/H'} \\
	& {H\backslash G/H'} && {(G\times G)/(H\times H')} && {(G\times G)\times_{\sigma\otimes\rho^*}\mathrm{Hom}(V_\rho,V_\sigma)} && {\mathrm{Hom}(V_\rho,V_\sigma)/(H\times H')}
	\arrow["{\hat\kappa}"{description}, shift left=2, curve={height=-30pt}, dashed, from=1-2, to=1-8]
	\arrow["{\pi_{H}}"', from=1-2, to=2-1]
	\arrow["{\pi_{H'}}", from=1-2, to=2-3]
	\arrow["{\pi_{\backslash/}}"{description}, from=1-2, to=3-2]
	\arrow["{\pi_G}"{description}, from=1-4, to=1-2]
	\arrow["\kappa"{description, pos=0.4}, curve={height=-18pt}, dashed, from=1-4, to=1-8]
	\arrow["{\pi_{H \times H'}}"{description, pos=0.7}, from=1-4, to=3-4]
	\arrow["{\pi_1}"{description}, from=1-6, to=1-4]
	\arrow["{\pi_2}"{description}, from=1-6, to=1-8]
	\arrow["\beta"{description, pos=0.7}, from=1-6, to=3-6]
	\arrow["{\pi_{\operatorname{Hom}}}"{description}, from=1-8, to=3-8]
	\arrow["{\pi_{H'}}", from=2-1, to=3-2]
	\arrow["{\overset{\leftarrow}{\kappa}}"{description}, shift left, curve={height=12pt}, dashed, from=2-3, to=1-8]
	\arrow["{\pi_{H}}"', from=2-3, to=3-2]
	\arrow["{\overline{\kappa}}"{description}, curve={height=12pt}, dashed, from=3-2, to=1-8]
	\arrow["\pi"{description}, from=3-4, to=3-2]
	\arrow["{\kappa_X}"', curve={height=18pt}, dashed, from=3-4, to=3-6]
	\arrow["{\tau_{\pi_{H\times H'}}}", from=3-6, to=3-4]
	\arrow["{\tau_{\pi_{\operatorname{Hom}}}}"', from=3-6, to=3-8]
\end{tikzcd} 
\]
\end{adjustwidth}
    \caption{The complete mixing diagram for the equivariant linear kernels over homogeneous spaces which is constructed from joining \cref{fig:mixing-diagram-linear-kernels} with \cref{fig:domain-reduction-gcnn-case}.
    This shows that there are five ways to view the integration kernels: two without equivariance and three with equivariance. 
    Without equivariance one can view the kernel as a section of an associated vector bundle $ \kappa_X\in\Gamma((G\times G)\times_{\sigma\otimes \rho^*}\mathrm{Hom}(V_\rho,V_\sigma)) $ or as an element in the induced representation $ \kappa\in\mathcal{K}=\mathcal{I}_{\sigma \otimes \rho^*} $.
    With equivariance the kernel can be viewed as a one-argument kernel $ \hat\kappa \in \mathcal{K}_G $ on $ G $, as a kernel on the homogeneous space $ \overset{\leftarrow}{\kappa} \in \mathcal{K}_X $, or as a kernel on the double coset space $ \overline{\kappa}\in \mathcal{K}_D  $.
}\label{fig:complete-gcnn-diagram}
\end{figure}
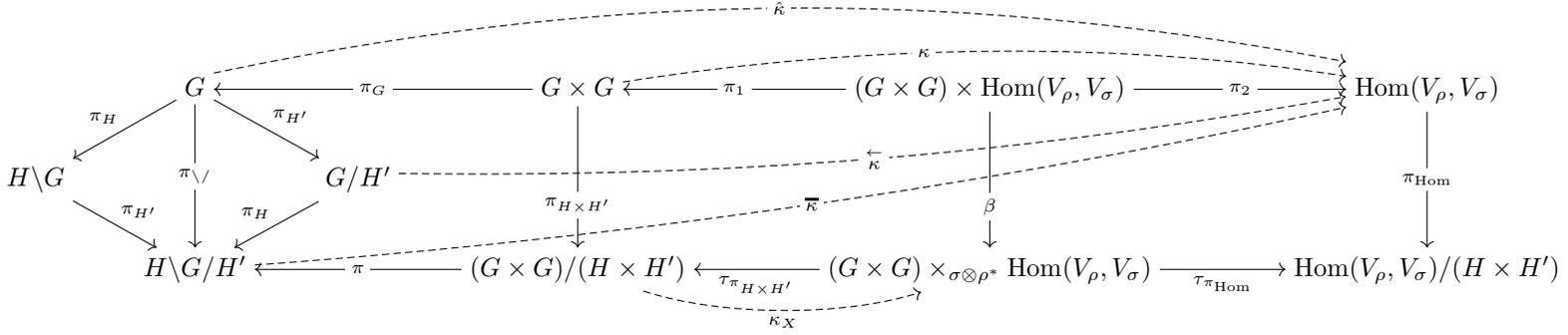
    Note that since $H\times H'$ acts on each factor of $G\times G$ we have that $(G\times G)/(H\times H')\isom G/H\times G/H'$, hence we can view the projection $\pi_{H \times H^{\prime}}:G\times G\to (G\times G)/(H\times H')$ in \cref{fig:complete-gcnn-diagram} as $\pi_{H \times H^{\prime}}:G\times G\to G/H\times G/H'$.

    The map $ \pi: G/H \times G/H^{\prime} \,\to\, H \backslash G/H^{\prime} $ is defined as $\pi(gH,g'H')=Hg^{-1}g'H'$ and defines a fibre bundle where the fibre can be viewed as either $(G\times H)/H$ or $(G\times H')/H'$, which in either case is isomorphic to $G$.
    Comparing this to the bundle $\pi_G:G\times G\to G$ where the fibre is $G$ we see that the fibres are the same.

    This means that, naturally, all horisontal arrows in \cref{fig:complete-gcnn-diagram} pairwise generates the same fibrations; in order from left to right: $G$, $\hom(V_\rho,V_\sigma)$, $G\times G$.
    Following the same reasoning all the vertical arrows yield fibres $H\times H'$.

    Worth noting is that altough $ \pi_{ \backslash /}: G \,\to\, H \backslash G/H^{\prime} $ generates a fibre of $ H \times H^{\prime} $ it is a principal $ H \times H^{\prime} $-bundle iff $ H \cap gH^{\prime}g^{-1}= \left\{ e \right\}  $ for all $ g \in G $.

%% file: sections/nonlinear-maps.tex
In \cref{sec:linear-maps} we covered \emph{linear} maps between induced representations and in this section we generalize this to the setting of arbitrary \emph{non-linear} maps.
We start in \cref{sec:nonlinear_maps_as_integral_operators} by defining such non-linear
maps as integral operators.
Section \ref{sec:nonlin-equivariant-maps} studies the symmetry constraints imposed on them when additionally demanding their equivariance.
As in \cref{sec:linear-maps}, if nothing else is stated, $G$ is a unimodular group.

\subsection{Non-linear maps as integral operators}
\label{sec:nonlinear_maps_as_integral_operators}
As we have previously motivated, the goal is to construct a framework that encompasses all non-linear maps between induced representations.
Now, a general map $ \Phi $ between two induced representations $ \mathcal{I}_\rho $ and $ \mathcal{I}_\sigma $ can take any number of forms, but a common occurrence in machine learning is to have some kind of spatial aggregation, or message passing, that allows the output at any specific point to depend on the input features at other points.

Taking inspiration from this practice, and making the spatial aggregation explicit, we propose that $ \Phi: \mathcal{I}_\rho \to \mathcal{I}_\sigma $ takes the form
\begin{equation}\label{eq:prototype-general-map}
	[\Phi f](g) = \int_{G^{\prime}} \omega(f, g,g') \mathrm{d}g'
\end{equation}
where $ \omega:\, \mathcal{I}_\rho \times G \times G^{\prime} \,\to\, V_\sigma $ is a map which,
conditioned on the whole feature field $f$,
aggregates features from any $g^{\prime}\in G^{\prime}$ into an output feature at $g\in G$.
Note that, a priori, $ G $ and $ G^{\prime} $ can be different.

\begin{remark}
    We will treat $ G $ and $ G^{\prime} $ as groups without any relation between them. 
    However, when $ G=G^{\prime} $, then the message passing point of view is much clearer and the functional dependence means that $ f $ may potentially be evaluated at positions $ g^{\prime \prime}\in G $ different from $ g $ and $ g^{\prime} $.
\end{remark}

\begin{remark}
    The $f$ dependence in $\omega$ may also include derivatives of $f$.
    This is however beyond the scope of this paper and will not be discussed here.
\end{remark}

To further structure this we want to define the space of the integrands.
\begin{definition}\label{def:general-Omega-space}
    Let $\Omega \subset \operatorname{Map}(\mathcal{I}_\rho \times G \times G^{\prime}, V_\sigma)$ be the space of maps $\omega:\mcI_\rho\times G\times G^{\prime}\to V_\sigma$ satisfying 
    \begin{equation}\label{eq:three-arg-general-map-constraint}
    \omega(f,gh,g')=\sigma(h^{-1})\omega(f,g,g^{\prime}), \quad \forall h\in H.
    \end{equation}
\end{definition}

Note that \cref{def:general-Omega-space} is reminiscent of the equivariance condition used in the induced representation.
In order to fully put $ \omega $ in the same context as the previous setting, that is, describe it as an induced representation as we did for the feature maps and the integration kernels, we need a way to endow $ \mathcal{I}_\rho \times G \times G^{\prime} $ with a group structure such that we can view $ H \leq G $ as a subgroup of $ \mathcal{I}_\rho \times G \times G^{\prime} $. The easiest group structure to equip $ \mathcal{I}_\rho \times G \times G^{\prime} $ with is 
\begin{equation}
    (f,g,g^{\prime})(\hat{f},\hat{g},\hat{g}^{\prime})=(f+\hat{g},g \hat{g}, g^{\prime} \hat{g}^{\prime}),
\end{equation}
and with that we can make the following proposition.

\begin{proposition}
    $ \Omega $, as defined above, is an induced representation $ \ind _{H}^{\mathcal{I}_\rho \times G \times G^{\prime}}\sigma $ with the $ \mathcal{I}_\rho \times G \times G^{\prime} $-action 
    \begin{equation}
        [(\hat{f},\hat{g},\hat{g}^{\prime})\omega](f,g,g^{\prime})=\omega(f-\hat{f},\hat{g}^{-1}g,\hat{g}^{\prime-1}g^{\prime}),
    \end{equation}
    satisfying the Mackey condition 
    \begin{equation}
        \omega((f,g,g^{\prime})\triangleleft(0,h,e))=\omega(f,gh,g^{\prime})=\sigma(h^{-1})\omega(f,g,g^{\prime}).
    \end{equation}
    Equipping $ \Omega $ with pointwise addition and scalar multiplication makes it a vector space.
\end{proposition}
\begin{proof}
    This statement follows immediately from that $ V_\sigma $ is a left $ H $-space and that $ \mathcal{I}_\rho \times G \times G^{\prime} $ is a product group.
    Specifically, the inclusion $ \iota :H \to \mathcal{I}_\rho \times G \times G^{\prime} $ as $ \iota(H)= \left\{ 0 \right\} \times H \times \left\{ e \right\} $ yields that $ H $ is a subgroup to $ \mathcal{I}_\rho \times G \times G^{\prime} $.
    Hence $ \Omega $ together with the $ \mathcal{I}_\rho \times G \times G^{\prime} $-action 
    \begin{equation}
        (\hat{f},\hat{g}, \hat{g}^{\prime})\omega(f,g,g^{\prime})=\omega((\hat{f},\hat{g},\hat{g}^{\prime})^{-1}(f,g,g^{\prime}))=\omega(f-\hat{f},\hat{g}^{-1}g,\hat{g}^{\prime-1}g^{\prime}),
    \end{equation}
    forms an induced representation.
\end{proof}
Furthermore, taking the quotient by this right $ H $-action on $ \mathcal{I}_\rho \times G \times G^{\prime} $ results in a principal $ H $-bundle as $ \pi_H:\mathcal{I}_\rho \times G \times G^{\prime}\to \mathcal{I}_\rho \times G/H \times G^{\prime} $.
As a consequence we can put this into the same context as the previous sections using the Cartan mixing diagram, see \cref{fig:mixing-diagram-for-omega}.
Additionally, by \cref{prop:function-isomorphism} we know that we can also represent these objects as local sections of an associated bundle $ \left( \mathcal{I}_\rho \times G \times G^{\prime} \right) \times _\sigma V_\sigma\to \mathcal{I}_\rho \times G/H \times G^{\prime} $.

\begin{figure}[h]
    \centering
    \[
\begin{tikzcd}[cramped]
	{\mathcal{I}_\rho \times G\times G'} && {(\mathcal{I}_\rho \times G\times G')\times V_\sigma} && {V_\sigma} \\
	\\
	{\mathcal{I}_\rho \times G/H\times G'} && {(\mathcal{I}_\rho \times G\times G')\times_\sigma V_\sigma} && {H\backslash V_\sigma}
	\arrow["\omega", curve={height=-18pt}, dashed, from=1-1, to=1-5]
	\arrow["{\pi_H}"{description}, from=1-1, to=3-1]
	\arrow["{\pi_1}"{description}, from=1-3, to=1-1]
	\arrow["{\pi_2}"{description}, from=1-3, to=1-5]
	\arrow["\beta"{description}, from=1-3, to=3-3]
	\arrow["{\pi_{H\backslash}}"{description}, from=1-5, to=3-5]
	\arrow["{\omega_X}", curve={height=-18pt}, dashed, from=3-1, to=3-3]
	\arrow["{\tau_{\pi_H}}", from=3-3, to=3-1]
	\arrow["{\tau_{\pi_{H\backslash}}}"', from=3-3, to=3-5]
\end{tikzcd}
\]
\caption{The mixing diagram for the induced representation used in the integral formulation of maps between induced representations.
    This shows that one can equivalently view the ``kernel'' either as an element of the induced representation $ \omega \in \Omega $ or as a section of the corresponding associated vector bundle $ \omega_X\in\Gamma((\mathcal{I}_\rho \times G\times G')\times_\sigma V_\sigma) $.}\label{fig:mixing-diagram-for-omega}
\end{figure}

Even though the spatial aggregation is well motivated from common practices in machine learning, it raises the question if the map $ \Phi $ is universal. 
That is, if any given map $ \lambda $ between induced representations can be represented by $ \Phi $ for some choice of $ \omega $.

\begin{theorem}\label{thm:generality-of-integral-formulation}
    Any map $\lambda:\mcI_\rho\to\mcI_\sigma$ between induced representations can be written as an integral
\begin{equation}\label{eq:general-map-between-induced-reps}
\lambda[f](g)=[\Phi f](g):=\int_{G^{\prime}}\omega(f,g,g')\mathrm{d}g',
\end{equation}
with a suitable choice of $\omega\in\Omega$. 
\end{theorem}
\begin{proof}
    Let $ \lambda $ be any map between $ \mathcal{I}_\rho $ and $ \mathcal{I}_\sigma $ and choose $\omega(f,g,g')=\delta(g')\lambda[f](g)$. 
This yields
\begin{equation}
[\Phi f](g)=\int_{G^{\prime}}\omega(f,g,g')\mathrm{d}g'=\int_{G^{\prime}}\delta(g')\lambda[f](g)\mathrm{d}g'=\lambda[f](g)
\end{equation}
as wanted.
To show the required Mackey-property of $\Phi f$ that $[\Phi f](gh)=\sigma(h^{-1})[\Phi f](g)$ we use the transformation of $\omega$ as $\omega(f,gh,g')=\sigma(h^{-1})\omega(f,g,g')$.
This property is satisfied here since 
\begin{equation}
 \omega(f,gh,g^{\prime})=\delta(g^{\prime})\lambda[f](gh)=\sigma(h^{-1})\delta(g^{\prime})\lambda[f](g) = \sigma(h^{-1})\omega(f,g,g^{\prime}).
\end{equation}
Now, using the definition of $\Phi f$ we obtain
\begin{equation}
[\Phi f](gh)=\int_G\omega(f,gh,g')\mathrm{d}g'=\sigma(h^{-1})\int_G\omega(f,g,g')\mathrm{d}g'=\sigma(h^{-1})[\Phi f](g)
\end{equation}
which results in $[\Phi f](gh)=\sigma(h^{-1})[\Phi f](g)$ and hence $\Phi f\in \mcI_\sigma$.
\end{proof}
This shows that the suggested form is universal.
 
\begin{remark}
   The theorem above should be interpreted in a distributional sense where the kernel $ \omega $ may be a generalised function.
   In \cref{app:functional-appraoch-to-non-lin-maps} we present the formulation of the theory in the language of distributions.
\end{remark}

\subsection{Equivariant non-linear maps}
\label{sec:nonlin-equivariant-maps}
Now we turn to the equivariant case for the non-linear map and start by stating a sufficient condition for the map to be $ G $-equivariant.

\begin{proposition}\label{prop:equivariance-constraint-for-general-kernel}
If $\omega\in \Omega$ satisfies
\begin{equation}\label{eq:equivariance-constraint-for-general-kernel}
\omega(f,g,g')=\omega(kf,kg,g'),\quad \forall k,g \in G,\ f \in \mathcal{I}_\rho, 
\end{equation}
then the integral map $\Phi:\mcI_\rho \to \mcI_\sigma$, as defined in \cref{eq:general-map-between-induced-reps}, is equivariant.
\end{proposition}
\begin{proof}
Again, for equivariance, we need $k[\Phi f](g)=[\Phi (kf)](g)$.
This results in
\begin{equation}
    \begin{aligned}
k[\Phi  f](g)&=[\Phi (kf)](g)&&\Leftrightarrow\\
\int_{G^{\prime}}\omega(f,k^{-1}g,g')\mathrm{d}g'&=\int_{G^{\prime}}\omega(kf,g,g')\mathrm{d}g.&&
    \end{aligned}
\end{equation}
Now, using that $ \omega(f,g,g')=\omega(kf,kg,g') $ we obtain that $ \omega(f,k^{-1}g,g')=\omega(kf,g,g') $, which shows the wanted result.
\end{proof}
   This shows, as we had in the linear case in \cref{sec:linear-maps}, that $ \omega $ is constant on some $ G $-orbit.
   Specifically, if we define a left action of $ G $ on $ \mathcal{I}_\rho \times G \times G^{\prime} $ as 
   \begin{equation}
       k(f,g,g^{\prime})=(kf,kg,g^{\prime}),
   \end{equation}
   then the equivariance condition becomes $ \omega(k(f,g,g^{\prime}))=\omega(f,g,g^{\prime}) $.
   This motivates a quotient by this action and let $ \omega $ descend to the resulting space.
   Naturally, this is done through a projection $ \pi_G:\mathcal{I}_\rho \times G \times G^{\prime} \to \mathcal{I}_\rho \times G^{\prime} $ acting as $ \pi_G(f,g,g^{\prime})=(g^{-1}f,g^{\prime}) $.
   Recall that, in the linear case, the integration kernel for equivariant maps was constant on $ G $-orbits, $ \kappa(k(g,g^{\prime}))=\kappa(kg,kg^{\prime})=\kappa(g,g^{\prime}) $, and this allowed to reduce the kernel to a single argument through $ \kappa(g,g^{\prime})=\kappa(e,g^{-1}g^{\prime}):=\hat{\kappa}(g^{-1}g^{\prime}) $.
   We can do the same process for the three-argument $ \omega $ in order to reduce it to two arguments.

\begin{remark}
    The projection $ \pi_G(f,g,g^{\prime})=(g^{-1}f,g^{\prime}) $ together with the right action of $ H $ on $ \mathcal{I}_\rho \times G \times G^{\prime} $ as $ \left( f,g,g^{\prime} \right) \triangleleft h= \left( f,gh,g^{\prime} \right) $ induces a right action of $ H $ on $ \mathcal{I}_\rho \times G^{\prime} $ through
    \begin{equation}
        \pi_G(f,gh,g^{\prime})=(h^{-1}g^{-1}f,g^{\prime})=(g^{-1}f,g^{\prime})\triangleleft h.
    \end{equation}
    With this action we get that $ \pi_G((f,g,g^{\prime})\triangleleft h)=\pi_G(f,g,g^{\prime})\triangleleft h $ and hence $ \pi_G:\mathcal{I}_\rho \times G \times G^{\prime}\to \mathcal{I}_\rho \times G^{\prime} $ is a $ H $-equivariant principal $ G $-bundle.
\end{remark}

\begin{theorem}\label{cor:steerable_functional_argument_reduction}
If $\omega\in \Omega$ satisfies $ \omega(f,g,g^{\prime})=\omega(kf,kg,g^{\prime}) $ for all $ k\in G $ then $ \Phi $ is equivariant and $ \omega $ can be reduced to a two-argument map $ \hat\omega(f,g^{\prime}) $ satisfying the Mackey constraint
\begin{equation}\label{eq:two-arg-general-map-constraint}
\hat\omega(hf,g')=\sigma(h)\hat\omega(f,g'), \quad \forall h\in H.
\end{equation}
Hence $ \Phi $ can be formulated as 
\begin{equation}
\label{eq:nonlinear-operator}
    [\Phi f](g)=\int_{G^{\prime}}\omega(f,g,g^{\prime})\mathrm{d}g^{\prime}= \int_{G^{\prime}} \hat{\omega}(g^{-1}f,g^{\prime})\mathrm{d}g^{\prime}.
\end{equation}
\end{theorem}
\begin{proof}
    That $ \Phi $ is equivariant is shown in the previous proposition.
    The two-argument map can be constructed by using $  \omega(f,g,g^{\prime})=\omega(kf,kg,g^{\prime})  $ with $k=g^{-1}$, for which $ \omega(f,g,g^{\prime})=\omega(g^{-1}f,e,g^{\prime}) $.
Following this we define the two-argument map as $ \hat\omega(f,g^{\prime}):=\omega(f,e,g^{\prime}) $, or equivalently $ \omega(f,g,g^{\prime})=:\hat\omega(g^{-1}f,g^{\prime}) $.

For the transformation constraint on $\hat\omega$, we note from the definition of $\omega$ that $\omega(f,gh,g')=\sigma(h^{-1})\omega(f,g,g')$.
Hence, with $ \omega(f,g,g^{\prime})=\hat\omega(g^{-1}f,g^{\prime}) $ we get that 
\begin{equation}
    \omega(f,gh,g')=\hat\omega(h^{-1}g^{-1}f,g'),
\end{equation}
and thus
\begin{equation}
    \sigma(h^{-1})\hat\omega(g^{-1}f,g')=\hat\omega(h^{-1}g^{-1}f,g'),    
\end{equation}
which gives the expected transformation $\hat\omega(hf,g')=\sigma(h)\hat\omega(f,g')$.
Putting this together we get 
\begin{equation}
[\Phi f](g)=\int_{G^{\prime}}\omega(f,g,g')\mathrm{d}g'=\int_{G^{\prime}}\hat\omega(g^{-1}f,g')\mathrm{d}g'.
\end{equation}
This completes the proof.
\end{proof}
\begin{remark}
    Intuitively one can view the process of projecting $ \mathcal{I}_\rho \times G \times G^{\prime} \,\to\, \mathcal{I}_\rho \times G^{\prime} $ that maps $ (f,g,g^{\prime}) \,\mapsto\, (g^{-1}f,g^{\prime}) $ as centring the feature map $ f $ at the group element $ g $. 
\end{remark}
\begin{definition}
    Let $ \hat{\Omega} \subset \operatorname{Map}(\mathcal{I}_\rho \times G^{\prime}, V_\sigma) $ denote the space of $ \hat{\omega}:\mathcal{I}_\rho \times G^{\prime} \to V_\sigma $ such that $ \hat{\omega}(hf,g^{\prime}h^{\prime})=\sigma(h)\hat{\omega}(f,g^{\prime}) $ for all $ h\in H,\, h^{\prime}\in H^{\prime} $.
\end{definition}

We can now state a result that this representation of (non-linear) equivariant maps is general.
To do this we choose $\hat\omega(g^{-1}f,g')=\delta(g')\lambda[g^{-1}f](e)$ where $\lambda$ is any map between induced rep satisfying the equivariance condition $\lambda[g^{-1}f](e)=\lambda[f](g)$.
Then, inserting this into the $\Phi f$ map yields $[\Phi f](g)=\lambda[f](g)$.

\begin{theorem}\label{prop:arbitrary-map-from-general-equiv-map}
	Given any equivariant operator $\lambda: \mathcal{I}_\rho \to \mathcal{I}_\sigma$ there exist a choice of $ \hat{\omega}$ where $\Phi = \lambda$, given by
\begin{equation}\label{eq:abritrary-from-general-map}
\hat\omega(g^{-1}f,g')=\delta(g')\lambda[g^{-1}f](e)
\end{equation}
where $\delta(g^{\prime})$ is the Dirac delta distribution.
\end{theorem}
\begin{proof}
	Let $ \lambda $ be any equivariant map between the induced representations.
    Then the operator given by the $ \hat{\omega}$ above is
	\begin{equation}
        [\Phi f](g) = \int_G \delta(g') \lambda[g^{-1}f(e) \mathrm{d}g' =  \lambda[g^{-1}f](e) \int_G \delta(g') \mathrm{d}g' =  \lambda[g^{-1}]f(e) = \lambda[f](g)
	\end{equation}
	where the final equality follows from the fact that $\lambda$ is an equivariant operator.
\end{proof}

Inspired by the investigation of the domain of the integration kernels, visualised in \cref{fig:complete-gcnn-diagram}, we now turn to the domains of the $ \omega $ maps.
The Cartan mixing diagram with this addition can be seen in \cref{fig:mixing-diagram-omega}.

When comparing to the $ G $-CNN case, we could reduce the domain of the one-argument kernel $ \hat\kappa $ by taking the quotient of $ H^{\prime} $ from the right and $ H $ from the left.
This raises the question if it is possible to take the left quotient of $ H $ in this more general setting.
Unfortunately this is not possible since $ H $ acts on the right in the $ G $-slot of $ \mathcal{I}_\rho \times G \times G^{\prime} $, and this action is turned, by the projection $ \pi_G $, into an action on the function space $ \mathcal{I}_\rho $.

Hence taking the quotient with respect to $ H $ would yield a base space $ H \backslash \mathcal{I}_\rho \times G^{\prime} $ and no behaviour of the $ \hat\omega $-map allows us to connect $ \hat{\omega}(f,g^{\prime}) $ to $ \hat{\omega}(Hf,g^{\prime}) $.
This prevents further reduction of the domain of $ \hat{\omega} $.
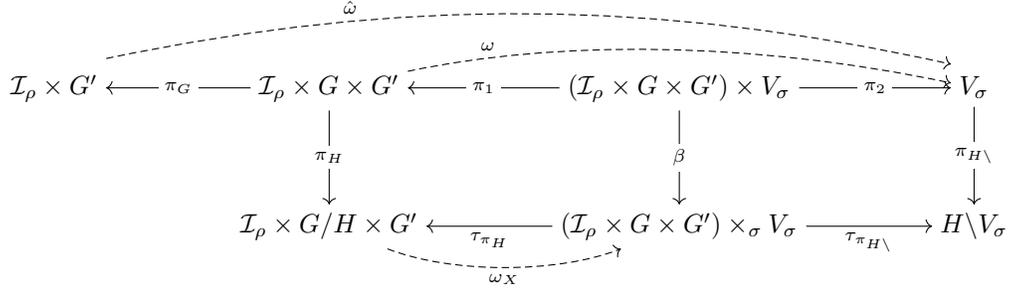
\begin{figure}[h]
    \[ 
\begin{tikzcd}[cramped]
	{\mathcal{I}_\rho \times G'} && {\mathcal{I}_\rho \times G\times G'} && {(\mathcal{I}_\rho \times G\times G')\times V_\sigma} && {V_\sigma} \\
	\\
	&& {\mathcal{I}_\rho \times G/H\times G'} && {(\mathcal{I}_\rho \times G\times G')\times_\sigma V_\sigma} && {H\backslash V_\sigma}
	\arrow["{\hat\omega}"{pos=0.3}, shift left=3, curve={height=-24pt}, dashed, from=1-1, to=1-7]
	\arrow["{\pi_G}"{description}, from=1-3, to=1-1]
	\arrow["\omega"{pos=0.2}, curve={height=-18pt}, dashed, from=1-3, to=1-7]
	\arrow["{\pi_H}"{description}, from=1-3, to=3-3]
	\arrow["{\pi_1}"{description}, from=1-5, to=1-3]
	\arrow["{\pi_2}"{description}, from=1-5, to=1-7]
	\arrow["\beta"{description}, from=1-5, to=3-5]
	\arrow["{\pi_{H\backslash}}"{description}, from=1-7, to=3-7]
	\arrow["{\omega_X}"', curve={height=18pt}, dashed, from=3-3, to=3-5]
	\arrow["{\tau_{\pi_H}}", from=3-5, to=3-3]
	\arrow["{\tau_{\pi_{H\backslash}}}"', from=3-5, to=3-7]
\end{tikzcd}
    \]
    \caption{The Cartan Mixing diagram for the non-linear case.
    Before equivariance we can view the ``kernel'' as either an element of the induced representation $ \omega \in \Omega $ or as a section of the corresponding associated vector bundle $ \omega _X \in \Gamma((\mathcal{I}_\rho \times G\times G)\times_\sigma V_\sigma) $.
When imposing equivariance it is viewed as a map $ \hat{\omega}\in \hat{\Omega} $.}\label{fig:mixing-diagram-omega}
\end{figure}

Another contrast to the linear $ G $-GCNN case is in the transformations of the integration kernel $ \hat\kappa $, the feature map $ f $, and the more general $ \hat{\omega} $ with respect to the subgroup $ H^{\prime} $.
Specifically, in the linear $ G $-CNN case the integrand is $ \hat\kappa(g^{-1}g^{\prime})f(g^{\prime}) $ where each transform under $ H^{\prime} $ as
\begin{equation}
    \hat{\kappa}(g^{-1}g^{\prime}h^{\prime})=\hat{\kappa}(g^{-1}g^{\prime})\rho(h^{\prime}) \qquad \text{and} \qquad f(g^{\prime}h^{\prime})=\rho(h^{\prime-1})f(g^{\prime}). 
\end{equation}
However, as these are contracted in the integrand they form an explicitly $ h^{\prime} $ invariant object:
\begin{equation}
    \hat{\kappa}(g^{-1}g^{\prime}h^{\prime})f(g^{\prime}h^{\prime})=\hat{\kappa}(g^{-1}g^{\prime})\rho(h^{\prime})\rho(h^{\prime-1})f(g^{\prime})=\hat{\kappa}(g^{-1}g^{\prime})f(g^{\prime}).
\end{equation}
This raises questions on how the $ \omega $ and $ \hat{\omega} $ transforms under $ h^{\prime} $.

\begin{corollary}
   Let $ \omega $ or $ \hat\omega $ factorise as
   \begin{equation}
       \omega(f,g,g^{\prime})=\alpha(f,g,g^{\prime})f(g^{\prime}) \qquad \text{or} \qquad  \hat{\omega}(g^{-1}f,g^{\prime})=\hat{\alpha}(g^{-1}f,g^{\prime})f(g^{\prime}),
   \end{equation}
   where $ \alpha:\, \mathcal{I}_\rho \times G \times G^{\prime} \,\to\, \hom(V_\rho,V_\sigma) $ and $ \hat{\alpha}:\, \mathcal{I}_\rho \times G^{\prime} \,\to\, \hom(V_\rho,V_\sigma) $. 
   Then they must satisfy
\begin{equation}
    \alpha(f,g,g^{\prime}h^{\prime})=\alpha(f,g,g^{\prime})\rho(h^{\prime}) \qquad \text{and} \qquad \hat{\alpha}(g^{-1}f,g^{\prime}h^{\prime})=\hat{\alpha}(g^{-1}f,g^{\prime})\rho(h^{\prime}), \quad \forall h^{\prime} \in H^{\prime}.
\end{equation}
\end{corollary}
\begin{proof}
    This follows from an identical argument to \cref{lemma:kernel-redundancy}.
\end{proof}

This indicates a way to view the map $ \hat{\omega} $ as an implicit contraction of induced representations such that the full object is $ H^{\prime} $ invariant. 
Furthermore, this suggests that specialisations of $ \hat{\omega} $ must be constructed from several induced representations in a very specific way such that the output is $ h^{\prime} $ invariant.

The following section will show how we can derive well established architectures through specific forms of the map $ \hat{\omega} $.

%% file: sections/instances.tex
We now demonstrate how our general framework can be used to derive various existing equivariant machine learning architectures.
This includes convolutional architectures (\cref{sec:gcnn-from-general-map,sec:implicit-steerable-from-general-map}) and attention-based architectures (\cref{sec:standard-attention-from-general-map,sec:relative-self-attention,sec:lietransformer-from-general-map}).
In what follows we assume that $ G=G^{\prime}$.

\subsection[Equivariant CNNs: {$G$-CNNs}]
           {Equivariant CNNs: \texorpdfstring{$\boldsymbol{G}$-CNNs}{G-CNNs}}
           \label{sec:gcnn-from-general-map}
First off, we want to detail how steerable $G$-CNNs \citep{cohen-theory-equivariant-hom}, surveyed in \cref{sec:linear-maps}, fit into the proposed framework of \cref{sec:nonlinear-maps}.
This is achieved by showing that a particular form of $ \hat{\omega}$ retrieves $G$-CNNs.
\begin{theorem}\label{prop:gcnn-from-general-equiv-map}
The $G$-CNN layer, as defined in \cref{def:gcnn-layer}, is obtained by the choice
\begin{equation}\label{eq:gcnn-from-general-map}
\hat\omega(g^{-1}f,g')\ =\ \hat\kappa(g')[g^{-1}f](g') \,,
\end{equation}
where
\begin{equation}
    \hat{\kappa}: G \to \hom(V_\rho, V_\sigma)
    \qquad\textup{such that}\qquad
    \hat{\kappa}(hgh') = \sigma(h) \hat{\kappa}(g) \rho(h')
\end{equation}
is a steerable convolution kernel as defined in \cref{thm:equivariant-kernels}.
\end{theorem}
\begin{proof}
Substituting the expression for $\hat\omega$ in the integral of \cref{eq:general-map-between-induced-reps} yields
\begin{equation}
[\Phi f](g)=\int_G \hat\omega(g^{-1}f,g')\mathrm{d}g'=\int_G\hat\kappa(g')[g^{-1}f](g')\mathrm{d}g'=\int_G\hat\kappa(g^{-1}g')f(g')\mathrm{d}g',
\end{equation}
where in the last step we used the $G$-action on $f$ which is $[kf](g)=f(k^{-1}g)$ and the change of variables $gg'\mapsto g'$.

Note that the transformation criteria of $\hat\omega$ is $\hat\omega(hf,g'h')=\sigma(h)\hat\omega(f,g')$.
Using this with \cref{eq:gcnn-from-general-map} we get that we should have
\begin{equation}
    \hat\kappa(g'h')[h^{-1}g^{-1}f](g'h')= \sigma(h^{-1})\hat\kappa(g')[g^{-1}f](g').
\end{equation}
We have to show consistency with \cref{eq:one-arg-kernel-constraint}.
From \cref{eq:one-arg-kernel-constraint} we get the $h'$ invariance through
\begin{equation}
    \hat\kappa(g'h')[g^{-1}f](g'h')=\hat\kappa(g')\rho(h')\rho(h^{\prime-1})[g^{-1}f](g')=\hat\kappa(g')[g^{-1}f](g').
\end{equation}
For the $h$ transformation, note that 
\begin{equation}
    \hat\kappa(g')[h^{-1}g^{-1}f](g')=\hat\kappa(g')f(ghg').
\end{equation}
Then, letting $g''=ghg'$ we get 
\begin{equation}
    \hat\kappa(h^{-1}g^{-1}g'')f(g'')=\sigma(h^{-1})\hat\kappa(g^{-1}g'')f(g''),
\end{equation}
where we again used the known transformation of $\hat\kappa$ stated in \cref{eq:one-arg-kernel-constraint}. 
The result follows from that $g,g''\in G$ and $f\in\mcI_\rho$ are arbitrary.
This shows that the choice of $\hat\omega$ is consistent with the transformation conditions on $\hat\omega$ and $\hat\kappa$, which completes the proof.
\end{proof}

The first factor in \cref{eq:gcnn-from-general-map} is the steerable convolutional kernel $\kappa$ as introduced in \cref{sec:linear-maps}.
In our framework, the defining characteristic of the family of convolution-based equivariant models is the independence of feature vectors $f$ in this factor.
One can contrast this with the subsequent results in \cref{sec:lietransformer-from-general-map}, where we show that $f $-dependence in this first factor is a central part attention-based equivariant models.
The second factor in \cref{eq:gcnn-from-general-map} is a single $f(g')$, which characterises the message-passing nature of convolutional networks.

\subsection{Implicit steerable convolutional kernels with node feature dependence}\label{sec:implicit-steerable-from-general-map}
The method of implicit steerable convolutional kernels \citet{implicit-kernels,Zhdanov2024CliffordSteerable} alleviates the problem of analytically solving the symmetry constraints put on the kernel of steerable equivariant CNNs \cite{weiler20183dSteerableCNNs,weiler2023EquivariantAndCoordinateIndependentCNNs}.
Instead, steerable kernels are implicitly learned using a surrogate neural network $k:\mathbb{R}^d \to \hom( \mathbb{R}^c, \mathbb{R}^{c'})$ designed to be equivariant in a way such that the kernel constraints are satisfied.
The use of such an auxiliary model allows for generalized kernels which depend not only on coordinates in $\mathbb{R}^d$, but additionally on feature vectors, such as node and edge features in point clouds.
This feature dependence turns implicit steerable convolutions into non-linear functions of feature maps, which fit naturally into our framework of non-linear equivariant maps.
The way these kernels depend on features also bear similarities to how attention-based architectures depend on the same objects, which we will derive in the following section.

Objects in this setting are defined on a homogeneous space $X$, which can be identified in the context of our work with the quotient space $G/H$.
Further, the space $X$ is assumed to be the Euclidean space $ \mathbb{R}^d$ and $G = \mathbb{R}^d \rtimes H$ is the affine group where the translations in  $ \mathbb{R}^d$ are equipped with addition as operation and $H$ is a subgroup of the orthogonal group $O(d)$.
These results can be extended to the setting where $H$ is a subgroup of $\operatorname{GL}(d)$ by introducing the Jacobian determinant $|\det h|$ in various expressions.
For convenience and to adopt the same assumptions as in \citet{implicit-kernels} we will assume $H \leq O(d)$ where the factor $| \det h | = 1$, and refer to \citet{weiler2023EquivariantAndCoordinateIndependentCNNs} for details on the more general case.

We also adopt the assumption that the subgroups associated with the input and output features are equal, $H = H'$.
We will however depart from the formulation of the original article slightly by considering continuous features throughout, where the original authors specialise to features only defined on a finite number of points in a point cloud.
This is to ease notation, and we remark that the point cloud formulation can be recovered by introducing delta functions at a finite number of locations in $X$.
The original implicit steerable convolutional kernel also included a dependence on edge features $z_{ij}$ which we choose to leave out of this formulation.
However, it is worth noting that the extension of our framework to continuous edge feature maps $f(x,x')$ is an interesting future direction of work.

The implicit steerable kernel is given by a neural network parametrized function
\begin{equation}
    k: \mathbb{R}^d \times \mathbb{R}^{d_z} \times \mathbb{R}^{d_z} \to \hom( \mathbb{R}^c, \mathbb{R}^{c'})
\end{equation}
of coordinates $\mathbb{R}^{d}$ and feature vectors $\mathbb{R}^{d_z}$ which satisfies the generalized steerability constraint
\begin{equation}
	\label{eq:steerability-constraint}
	k\big(h \vartriangleright (x - x'),\, \rho_z(h)f_z(x),\, \rho_z(h)f_z(x')\big)\ =\ \sigma(h) k\big(x - x',\, f_z(x),\, f_z(x')\big) \rho(h)^{-1}.
\end{equation}
Here $f_z(x), f_z(x')$ are continuous analogues of the node features $z_i, z_j$ in the original article.
In fact, since the features $f_z$ satisfy the same properties as the feature map $f$, we can absorb all features into one map which we will from now on denote $f$.
This condition generalises the corresponding condition for standard steerable kernels, with the standard expression recovered if the dependence on $f_z$ is dropped.
The convolution operation $\Psi_k: \Gamma(G \times_\rho \mathbb{R}^c) \to \Gamma(G \times_\sigma \mathbb{R}^{c'})$ given by the implicit kernel $k$ is
\begin{equation}
	\label{eq:implicit-conv}
	[\Psi_k f](x) = \int_X k\big(x' - x, f(x), f(x')\big)f(x') dx',
\end{equation}
As the implicit kernel is defined on $X = \mathbb{R}^d$ while the framework of \cref{sec:nonlinear-maps} is formalised with objects on the group $G$, we need to establish the equivalence between steerable kernels defined on $X$ and those defined on $G$.
This result, which can be seen as a kernel level equivalent to the isomorphisms established in \cref{prop:function-isomorphism}, is stated here as a lemma.
The result is stated here for a more general kernel $\hat{k}$ with domain $X \times \Gamma(G \times_\rho \mathbb{R}^{c})$.
The kernel $k$ is recovered by letting $\hat{k}(x'-x, T_{(-x)}f) = k\big(x'- x, [T_{(-x)}f](0), [T_{(-x)}f](x'-x)\big) = k\big(x'- x, f(x), f(x')\big)$.

\begin{lemma}
	\label{lemma:lifted-kernel}
    Assume $G = \mathbb{R}^d \rtimes H$ is an affine group with $H$ a subgroup of $O(d)$, and consider the space $\Gamma(G \times_\rho \mathbb{R}^c)$ of sections over the homogeneous space $X = \mathbb{R}^d$.
    Denote by $T_x$ the regular representation of $ \mathbb{R}^d$ on $\Gamma(G \times_\rho \mathbb{R}^c)$ given by $[T_xf](x') = f(x' - x)$.
    Given a map $\hat{k}: X \times \Gamma(G \times_\rho \mathbb{R}^{c}) \to \hom( \mathbb{R}^c, \mathbb{R}^{c'})$ and the resulting operator $\Psi_{\hat{k}}: \Gamma(G \times_\rho \mathbb{R}^{c}) \to \Gamma(G \times_\sigma \mathbb{R}^{c'})$ given by
\begin{equation}
	[\Psi_{\hat{k}} f](x) = \int_X \hat{k}(x' - x, T_{(-x)}f) f(x') \mathrm{d}x'
\end{equation}
then there exists a kernel $\kappa: G \times \mathcal{I}_\rho \to \hom(\mathbb{R}^{c}, \mathbb{R}^{c'})$ such that the resulting operator $\Phi_\kappa: \mathcal{I}_\rho \to \mathcal{I}_\sigma$ is the lifted version of $\Psi_k$ in the sense that
	\begin{equation}
		\label{eq:lifted-operator}
		\Phi_\kappa = \Lambda^{\uparrow} \circ \Psi_k \circ \Lambda^{\downarrow}
	\end{equation}
	where $\Lambda^{\uparrow}$ and $\Lambda^{\downarrow}$ are defined in \cref{eq:function_lift,eq:function_sink}.
	Further, $k$ satisfies the steerability constraint
	\begin{equation}
		\label{eq:steerability-x}
		\hat{k}(h \vartriangleright x, hf) = \sigma(h) \hat{k}(x, f) \rho(h)^{-1}
	\end{equation}
	if and only if $\kappa$ satisfies
	\begin{equation}
		\label{eq:steerability-g}
		\kappa(hgh', hf) = \sigma(h) \kappa(g, f) \rho(h')
	\end{equation}
	for all $h,h' \in H$.
\end{lemma}
\begin{proof}
	Since the group $G = \mathbb{R}^d \rtimes H$ is a direct product when considered as a space, there always exists a global section $s: \mathbb{R}^d \to G$ given by  $x \mapsto (x, 0)$.
	We use this section throughout the proof, including for the isomorphism $\Lambda^{\uparrow}$.
	Given the kernel $\hat{k}: X \times \Gamma (G \times_\rho \mathbb{R}^c) \to \hom( \mathbb{R}^c, \mathbb{R}^{c'})$ the lifted kernel $\kappa$ is given by
	\begin{equation}
		\label{eq:isomorphism-up}
		\kappa(g, f) = \hat{k}( \pi(g), \Lambda^{\downarrow} f) \rho( \mathrm{h}(g)).
	\end{equation}
	The inverse map, from a kernel $\kappa: G \times \mathcal{I}_\rho \to \hom ( \mathbb{R}^c, \mathbb{R}^{c'})$ to $\hat{k}$ is given by
	\begin{equation}
		\label{eq:isomorphism-down}
	\hat{k}(x, f) = \kappa( s(x), \Lambda^{\uparrow} f).
	\end{equation}
	Given the identities $s(x) = (x,0)$, $\pi((x, h)) = x$ and  $ \mathrm{h}((x,h)) = h$ it is straightforward to confirm that \cref{eq:steerability-x} and  \cref{eq:isomorphism-up} implies \cref{eq:steerability-g}
	\begin{equation}
		\begin{aligned}
			\kappa(hgh', hf) &= \hat{k}( \pi(hgh'), \Lambda^{\downarrow} [hf]) \rho( \mathrm{h}(hgh')) \\
					 &= \hat{k}( h \vartriangleright \pi(g), h[\Lambda^{\downarrow} f]) \rho(h \mathrm{h}(g)h') \\
					 &= \sigma(h)\hat{k}( \pi(g), [\Lambda^{\downarrow} f])\rho(h)^{-1} \rho(h) \rho(\mathrm{h}(g))\rho(h') \\
					 &= \sigma(h)\kappa(g, f)\rho(h').
		\end{aligned}
	\end{equation}
	That \cref{eq:steerability-g} and  \cref{eq:isomorphism-down} implies \cref{eq:steerability-x} is shown in a similar way.
	We have left to show the relation \cref{eq:lifted-operator}.
	Fix a kernel $\hat{k}: X \times \Gamma(G \times_\rho \mathbb{R}^c) \to \hom ( \mathbb{R}^c,\mathbb{R}^{c'})$.
	Then the action of the operator $\Lambda^{\uparrow} \circ \Psi_{\hat{k}} \circ \Lambda^{\downarrow}$ on $f \in \mathcal{I}_\rho$ is given by
	\begin{equation}
		\begin{aligned}
			\big[ [\Lambda^{\uparrow} \circ \Psi_k \circ \Lambda^{\downarrow}] f\big](g) &= \sigma( \mathrm{h}(g)^{-1})\big[ \Psi_k [\Lambda^{\downarrow} f]\big](\pi(g)) \\
												     &= \sigma( \mathrm{h}(g)^{-1})\int_X \hat{k}(x' - \pi(g), T_{(-\pi(g))}\Lambda^{\downarrow}f) [\Lambda^{\downarrow}f](x') \mathrm{d}x'.
		\end{aligned}
	\end{equation}
	Using the steerability constraint in \cref{eq:steerability-x} yields 
	\begin{equation}
	\sigma( \mathrm{h}(g)^{-1})\hat{k}(x' - \pi(g), T_{(-\pi(g))}\Lambda^{\downarrow}f) = \hat{k}(\mathrm{h}(g)^{-1} \vartriangleright( x' - \pi(g)), \mathrm{h}(g)^{-1}T_{(-\pi(g))}\Lambda^{\downarrow}f) \rho(\mathrm{h}(g)^{-1})
	\end{equation}
    where $g^{-1}$ denotes the action of $g^{-1}$ on $\mathcal{I}_\rho$ and $\mathrm{h}(g)^{-1}$ denotes the action of $\mathrm{h}(g)^{-1}$ on sections as defined in \cref{eq:rep-on-sections-of-vector-bundle}.
	It is straightforward to show the identities $\mathrm{h}(g)^{-1} \vartriangleright( x' - \pi(g)) = \pi( g^{-1} s(x'))$ and $\mathrm{h}(g)^{-1} \circ T_{(-\pi(g))} \circ \Lambda^{\downarrow} = \Lambda^{\downarrow} \circ g^{-1}$. 
    This yields
	\begin{equation}
		\begin{aligned}
			\big[ [\Lambda^{\uparrow} \circ \Psi_{\hat{k}} \circ \Lambda^{\downarrow}] f\big](g) &= \int_X \hat{k}(\pi( g^{-1} s(x')), [\Lambda^{\downarrow} \circ g^{-1}]f)\rho( \mathrm{h}(g)^{-1}) f(s(x')) \mathrm{d}x'.
		\end{aligned}
	\end{equation}
	At this stage we can identify the lifted kernel $\kappa$ as given in \cref{eq:isomorphism-up}
	\begin{equation}
		\begin{aligned}
			\big[ [\Lambda^{\uparrow} \circ \Psi_k \circ \Lambda^{\downarrow}] f\big](g) &= \int_X \kappa(g^{-1} s(x'), g^{-1}f) f(s(x')) \mathrm{d}x'
		\end{aligned}
	\end{equation}
	where we used the fact that $ \mathrm{h}(g^{-1}s(x')) = \mathrm{h}(g)^{-1}$.
	Finally we introduce the factor $\id = \int_H \mathrm{d}h' = \int_H \rho(h') \rho(h'^{-1}) \mathrm{d}h'$ and apply \cref{thm:quotient-integral-formula} yielding
	\begin{equation}
		\begin{aligned}
			\big[ [\Lambda^{\uparrow} \circ \Psi_k \circ \Lambda^{\downarrow}] f\big](g) &= \int_H \int_X \kappa(g^{-1} s(x'), g^{-1}f) \rho(h') \rho(h'^{-1}) f(s(x')) \mathrm{d}x' \mathrm{d}h' \\
			&= \int_H \int_X \kappa(g^{-1} s(x')h', g^{-1}f) f(s(x')h') \mathrm{d}x' \mathrm{d}h' \\
			&= \int_G \kappa(g^{-1} g', g^{-1}f) f(g') \mathrm{d}g'.
		\end{aligned}
	\end{equation}
	We identify this as $[\Phi_\kappa f](g)$ for the corresponding kernel $\kappa$, which concludes the proof.
\end{proof}

Using this dictionary between objects on $G$ and objects on $X$ we can now prove the result that the steerable kernels with node feature dependence, which are introduced as implicitly learned kernels in \citet{implicit-kernels}, fit into the framework we introduced in \cref{sec:nonlinear-maps}.

\begin{theorem}
\label{thm:instance-implicit-kernel}
	Assume $G = \mathbb{R}^d \rtimes H$ is an affine group with $H$ a subgroup of $O(d)$, and consider the space $\Gamma(G \times_\rho \mathbb{R}^c)$ of sections over the homogeneous space $X = \mathbb{R}^d$.
	Then the convolutional operator $\Psi_k: \Gamma(G \times_\rho \mathbb{R}^c) \to \Gamma(G \times_\sigma \mathbb{R}^{c'})$ given by a steerable kernel with node feature dependence $k: X \times \mathbb{R}^{d_z} \times \mathbb{R}^{d_z} \to \hom(\mathbb{R}^{c},\mathbb{R}^{c'})$ can be lifted to an operator as in \cref{eq:nonlinear-operator} where $\hat{\omega}$ factors as
	\begin{equation}
		\label{eq:instance-implicit-kernel}
		\hat{\omega}(g^{-1} f, g') = \kappa\big( g', [g^{-1}f](e), [g^{-1}f](g') \big) [g^{-1} f](g').
	\end{equation}
    Here the feature map $f$ includes both node and regular features.
\end{theorem}
\begin{proof}
	The convolutional operator $\Psi_k$ can be rewritten as
	\begin{equation}
		[\Psi_k f](x_i) = \int_X k\big( x' - x_i, [T_{(-x_i)}f](0), [T_{(-x_i)}f](x' - x_i)\big)f(x') \mathrm{d}x'
	\end{equation}
	where $0 \in X = \mathbb{R}^d$.
	Since the kernel $k$ depends only on arguments $T _{(-x_i)}$ and $x' - x_i$, we can apply \cref{lemma:lifted-kernel} to construct a lifted kernel $\kappa$ as in \cref{eq:isomorphism-up} such that $\Phi_\kappa = \Lambda^{\uparrow} \circ \Psi_k \circ \Lambda^{\downarrow}$.
	The arguments $(x'-x, f(x), f(x'))$ are lifted to arguments $(g^{-1}g', [\Lambda^{\uparrow}f](g), [\Lambda^{\uparrow}f](g'))$.
    To ease notation denote $f^\uparrow = \Lambda^{\uparrow}f$.
	We retrieved the expression for \cref{eq:instance-implicit-kernel} through a change of variable $g' \mapsto g^{-1}g'$ under the integral sign
	\begin{equation}
	\begin{aligned}
		[\Phi_{\hat{\omega}}f^\uparrow](g) &= \int_G \hat{\omega}(g^{-1} f^\uparrow, g') \\
					  &= \int_G \kappa\big( g', \big[g^{-1}f^\uparrow\big](e), \big[g^{-1}f^\uparrow\big](g') \big) \big[g^{-1} f^\uparrow\big](g') \mathrm{d}g' \\
					  &= \int_G \kappa\big( g', f^\uparrow(g), f^\uparrow(gg') \big) f^\uparrow(gg') \mathrm{d}g' \\ 
					  &= \int_G \kappa\big( g^{-1}g', f^\uparrow(g), f^\uparrow(g') \big) f^\uparrow(g') \mathrm{d}g' \\ 
	\end{aligned}
	\end{equation}
	where we used the invariance of the Haar measure $ \mathrm{d}(g^{-1}g') = \mathrm{d}g'$.
\end{proof}

\subsection{Conventional self-attention \& permutation equivariance}\label{sec:standard-attention-from-general-map}
Both \cref{sec:gcnn-from-general-map,sec:implicit-steerable-from-general-map} concerned convolution-based neural network architectures.
However the framework can also fit other types of architectures, notably attention-based ones.
We show a concrete example of this by presenting here how conventional dot-product self-attention \citep{vaswaniAttentionAllYou2017} fits into our framework.
We will handle two cases separately, namely the case without any position embeddings in the present section and the case with relative positional embeddings in the subsequent \cref{sec:relative-self-attention}.

The goal of this section is to show that our framework covers the standard self-attention operation.
$X$ is hereby chosen to be the finite set of $n$ tokens and the symmetry group $G=S_n$ is the group of all permutations of elements in $X$.
The self-attention operator $\Phi ^{ \text{attn}}$ is in this setting given by
\begin{equation}
\begin{aligned}
	\label{eq:discrete-self-attention}
	[\Phi ^{ \text{attn}}f](x) &= \sum_{x' \in X} \operatorname{Softmax} \left\{ \dfrac{f(x)^\top W_Q^\top W_K f(x') }{ \sqrt{d}} \right\} W_V f(x') \\
    &= \frac{1}{\mathcal{Z}\big(\rho_\Gamma (s(x)^{-1})f\big)} \sum_{x' \in X} \exp  \left\{ \dfrac{f(x)^\top W_Q^\top W_K f(x') }{ \sqrt{d}} \right\} W_V f(x')
    \end{aligned}
\end{equation}
where $s: X \to G$ is an arbitrary choice of section and $[\rho_\Gamma(g) f](x)=\rho(\mathrm{h}(g))f(g^{-1} \vartriangleright x)$ is the representation of $G$ on sections $f$ defined in \cref{eq:rep-on-sections-of-vector-bundle}.
Token are embedded by matrices $W_Q, W_K, W_V \in \mathbb{R}^{\tilde{c} \times c}$ and the normalisation factor of the softmax function is 
\begin{equation}
\begin{aligned}
\mathcal{Z}\big(\rho_\Gamma (s(x)^{-1})f\big) &= \sum_{x' \in X} \exp \left[ \big( [\rho_\Gamma (s(x)^{-1})f](x_0)^\top W_Q^\top W_K [\rho_\Gamma (s(x)^{-1})f](x')\big) / \sqrt{d} \right] \\
&= \sum_{x' \in X} \exp \left[ \big( f(x)^\top W_Q^\top W_K f(s(x) \vartriangleright x')\big) / \sqrt{d} \right].
\end{aligned}    
\end{equation}
For any section $s: X \to G$ the action $s(x) \vartriangleright: X \to X$ is a bijection on $X$ (with inverse $s(x)^{-1} \vartriangleright: X \to X$).
It is therefore possible to relabel $x' \mapsto s(x) \vartriangleright x'$ and get an expression which makes it clear that $1/\mathcal{Z}$ is a normalisation factor independent of the particular choice of $s$
\begin{equation}
    \mathcal{Z}\big(\rho_\Gamma (s(x)^{-1})f\big) = \sum_{x' \in X} \exp \left[ \big( f(x)^\top W_Q^\top W_K f(x')\big) / \sqrt{d} \right].
\end{equation}
We write the argument in this form to make a direct connection to $g^{-1}f$ below.
In fact the current setting where the representation $\rho$ of $H$ on $V_\rho$ is trivial then the action on $f$ reduces to $[\rho_\Gamma(g)f](x) = f(g^{-1} \vartriangleright x)$.

The permutation equivariance of standard self-attention is a special case of the $G$-equivariance assumed in \cref{sec:nonlin-equivariant-maps} with $G$ discrete, and the set $X$ is interpreted as an index set which enumerates vector features.
In this case the subgroup $H$ is the group of permutations which fix some arbitrarily chosen origin $x_0 \in X$, and the representations $\rho$ and $\sigma$ of $H$ are assumed to be trivial representations.
As was already noted in \cref{sec:lietransformer-from-general-map}, the form of scaled-dot attention does not strightforwardly generalise to non-trivially transforming features.

As all finite sets of size $n$ are isomorphic, we can assume without loss of generality that $X = [n] = \{0, 1, 2, \dots, ,n -1\}$.
The elements in the group $ S_n$ represent the $n!$ unique ways of permuting the set $X = [n]$.
In particular, for any element $i \in [n]$ there exist a subgroup $ H= \operatorname{Stab}_{S_n}(i) \leq S_n$ of all permutations $\tau \in S_n$ which fix the element $i$, $\tau \vartriangleright i = i$.
As the stabiliser at $i$ will consist of all permutation of the other $n-1$ elements in $X$, we have that $\operatorname{Stab}_{S_n}(i) \cong S_{n-1}$ for all $i \in [n]$.
Again without loss of generality, we can choose $x_0 = 0$ and invoke \cref{def:homogeneous-space-isomorphism} to construct the quotient $ S_n / \operatorname{Stab}_{S_n}(0) \cong X$.
Note that the projection $\pi: S_n \to X$ maps a permutation $\tau \in S_n$ to the element $\tau \rhd x_0 = \tau \rhd 0 \in X$.

\begin{theorem}
\label{thm:instance-self-attention}
	Assume $X$ is a finite set of size $n$ and $G = S_n$ is the symmetric group acting on $X$ as permutations.
    Further let $H = S_{n-1}$ be the subgroup of $G$ which keeps the first element in $X$ fixed, and let the representations $\rho, \sigma$ be trivial representations of $H$.
	Then the standard discrete self-attention in \cref{eq:discrete-self-attention} can be lifted to an operator as in \cref{eq:nonlinear-operator} where $\hat{\omega}$ is given by
	\begin{equation}
		\hat{\omega}( g ^{-1} \hat{f}, g') = \frac{1}{\hat{\mathcal{Z}}\big(g^{-1}\hat{f}\big)}\exp\left\{ \dfrac{[g ^{-1} \hat{f}](e)^\top W_Q^\top W_K [ g ^{-1} \hat{f}](g') }{ \sqrt{d}}\right\} W_V [g ^{-1} \hat{f}](g').
	\end{equation}
	Here $\hat{f} = \Lambda^{\uparrow}f$ denotes the feature map $f: X \to V_\rho$ in \cref{eq:discrete-self-attention} lifted to the group $G$, and 
    \begin{equation}
    \hat{\mathcal{Z}}\big(g^{-1}\hat{f}\big) = \int_G \exp \left[ \big( \hat{f}(g)^\top W_Q^\top W_K \hat{f}(g')\big) / \sqrt{d} \right] dg'.\nonumber
    \end{equation}
\end{theorem}
\begin{proof}
The change of variable $g' \mapsto g^{-1}g'$ under the integral sign yields
\begin{equation}
	\begin{aligned}
		[\Phi _{ \hat{\omega}} \hat{f}](g) &= \int_G \hat{\omega}( g ^{-1} \hat{f}, g') \mathrm{d}g' \\
					    &= \frac{1}{\hat{\mathcal{Z}}\big(g^{-1}\hat{f}\big)} \int_G \exp \left\{ \dfrac{[g ^{-1} \hat{f}](e)^\top W_Q^\top W_K [ g ^{-1} \hat{f}](g') }{ \sqrt{d}} \right\} W_V [g ^{-1} \hat{f}](g') \mathrm{d}g' \\
						       &= \frac{1}{\hat{\mathcal{Z}}\big(g^{-1}\hat{f}\big)} \int_G \exp \left\{ \dfrac{ \hat{f}(g)^\top W_Q^\top W_K \hat{f}(g') }{ \sqrt{d}} \right\} W_V \hat{f}(g') \mathrm{d}g' 
	\end{aligned}
\end{equation}
where we also use the left-invariance of the Haar measure.
Let $H \leq G$ be the subgroup of $ S_n$ that fixes the first element in $X$, as discussed above.
Given an arbitrary choice of global section $s: X \to G$ one can apply \cref{thm:quotient-integral-formula} to decompose the integral as\footnote{Since we are in the discrete setting we do not impose continuity on the section, and as such there will always exist global sections.}
\begin{equation}
	[\Phi _{ \hat{\omega}} \hat{f}](g) = \frac{1}{\hat{\mathcal{Z}}\big(g^{-1}\hat{f}\big)} \int_X \int_H \exp \left\{ \dfrac{\hat{f}(g)^\top W_Q^\top W_K \hat{f}(s(x')h') }{ \sqrt{d}} \right\} W_V \hat{f}(s(x')h') \,\mathrm{d}h' \,\mathrm{d}x'.
\end{equation}
The feature map $\hat{f}$ is induced from the trivial representation, hence it is invariant $\hat{f}(gh) = \hat{f}(g)$ for all $g \in G$, $h \in H$.
Evaluating at $g = s(x)$ reduces the expression to
\begin{equation}
	\begin{aligned}
		[\Phi _{ \hat{\omega}}\hat{f}](s(x)) &= \frac{1}{\hat{\mathcal{Z}}\big(s(x)^{-1}\hat{f}\big)} \int_X \int_H \exp \left\{ \dfrac{\hat{f}(s(x))^\top W_Q^\top W_K \hat{f}(s(x')) }{ \sqrt{d}} \right\} W_V \hat{f}(s(x')) \mathrm{d}h' \mathrm{d}x' \\
		&= \frac{1}{\hat{\mathcal{Z}}\big(s(x)^{-1}\hat{f}\big)} \int_X \exp \left\{ \dfrac{\hat{f}(s(x))^\top W_Q^\top W_K \hat{f}(s(x')) }{ \sqrt{d}} \right\} W_V \hat{f}(s(x')) \mathrm{d}x' \int_H \mathrm{d}h' \\
		&= \frac{1}{\hat{\mathcal{Z}}\big(s(x)^{-1}\hat{f}\big)} \int_X \exp \left\{ \dfrac{\hat{f}(s(x))^\top W_Q^\top W_K \hat{f}(s(x')) }{ \sqrt{d}} \right\} W_V \hat{f}(s(x')) \mathrm{d}x'
	\end{aligned}
\end{equation}
since the Haar measure on $H$ is assumed to be normalised.
Here we can recognise the isomorphisms from \cref{eq:function_lift} and \cref{eq:function_sink}
\begin{equation}
	\hat{f}(s(x)) = [\Lambda ^{\downarrow}\hat{f}](x) = f(x), \quad 
	[\Phi _{ \hat{\omega}} \hat{f}](s(x)) = [[\Lambda ^{\downarrow} \circ \Phi _{ \hat{\omega}} ] \hat{f}](x) = [[\Lambda ^{\downarrow} \circ \Phi _{ \hat{\omega}} \circ \Lambda ^{ \uparrow}] f](x).
\end{equation}
By a final observation $s(x)^{-1}\hat{f} = [s(x)^{-1} \circ \Lambda^{\uparrow}]f = [\Lambda^{\uparrow} \circ \rho_\Gamma(s(x)^{-1})]\hat{f}$ we can also put $\mathcal{Z}$ in a form which depends on $f$, and conclude
\begin{equation}
    [[\Lambda ^{\downarrow} \circ \Phi _{ \hat{\omega}} \circ \Lambda ^{\uparrow}]f](x) = \frac{1}{\mathcal{Z}\big(\rho_\Gamma(s(x)^{-1})f\big)} \int_X \exp \left\{ \dfrac{f(x)^\top W_Q^\top W_K f(x') }{ \sqrt{d}} \right\} W_V f(x') \mathrm{d}x',
\end{equation}
which reduces to a lifted version of self-attention $\Phi ^{ \text{attn}}$ in the sense of \cref{lemma:lifted-kernel} if $X$ is a finite set.
This concludes the proof.
\end{proof}

The fact that we can fit $\Phi ^{ \text{attn}}$ into the framework of \cref{sec:nonlinear-maps} implies that it inherits all properties of $\Phi _{ \hat{\omega}}$ that were derived in the previous section.
In particular, the regular equivariance $k \vartriangleright [\Phi _{ \hat{\omega}}f](g) = [\Phi _{ \hat{\omega}}(k \vartriangleright f)](g)$ specialises to permutation equivariance of the self-attention operator $\Phi ^{ \text{attn}}$ in the sense that if we permute the feature vectors  $f(x) \mapsto f(\tau \vartriangleright x)$ for any permutation $\tau \in  S_n$, the output of self-attention transforms accordingly as $[\Phi ^{ \text{attn}}f](x) \mapsto [\Phi ^{ \text{attn}}f](\tau \vartriangleright x)$.
This is the well-known permutation equivariance of vanilla self-attention.

\subsection{Relative position encoded self-attention \& translation equivariance}
\label{sec:relative-self-attention}

The full permutation equivariance is quite restrictive and is only appropriate for set-structured data, while these networks are often used to process sequential data.
This requires a smaller equivariance group which motivates the introduction of a positional bias.
It is widely known in the machine learning community that the introduction of positional bias breaks the permutation equivariance and leaves only a subgroup of the symmetric group as the symmetry group.
The first wave of self-attention architectures employed \emph{absolute} positional embeddings which leaves no non-trivial equivariance \citep{vaswaniAttentionAllYou2017}.
Later iterations use various versions of translation equivariant \emph{relative} position encodings, leading to attention scores that are symmetric under shifts of the index $x \mapsto x + k$ for $k \in \mathbb{Z}$.
In this section we show that the relative position biases \citep{shaw-etal-2018-self} and rotary embeddings \citep{su2023roformerenhancedtransformerrotary}, both examples of translation equivariant embeddings, fit into our framework.
The relevant symmetry group here are integer translations $G = (\mathbb{Z},+)$ acting by addition.
Also the index space is formalised as integers $X = \mathbb{Z}$ which one can imagine as an infinite one-dimensional grid of sequential token positions.
This is a mathematically convenient setting: since the group $G = (\mathbb{Z},+)$ and the homogeneous space $X = \mathbb{Z}$ are equal as spaces, the stabiliser subgroup and quotient constructions are trivial $ \mathbb{Z} = X = G/H = \mathbb{Z}/\{e\}$ where $H = \{e\}$ is the trivial group.
This removes the need to lift maps and operators between $G/H$ and $G$ since these lifts are all trivial.
The details follow.

\subsubsection*{Relative position bias}
Self-attention on an infinite grid $ \mathbb{Z}$ with relative positional bias $\psi: \mathbb{Z} \to  \mathbb{R}$ is given by
\begin{equation}
	\label{eq:relative-self-attention}
	[\Phi ^{ \text{attn}}_\psi f](x) = \sum_{x' \in \mathbb{Z}} \operatorname{Softmax} \left\{ \dfrac{f(x)^\top W_Q^\top W_K f(x') }{ \sqrt{d}} \odot \psi(x'-x) \right\} W_V f(x')
\end{equation}
where the operator $\odot$ is usually taken to be multiplication or addition.
Let $(\mathbb{Z},+)$ act on $ \mathbb{Z}$ by translation, $a \vartriangleright b = a + b$.
One can see that the operator $\Phi ^{ \text{attn}}_\psi$ is  $ \mathbb{Z}$-equivariant by 1) observing that it is a subgroup of permutations, making the first term equivariant, while 2) the position bias is invariant under translations.
As noted in the previous paragraph, the domain $ \mathbb{Z}$ equals the symmetry group and the operator is therefore of the form studied in \cref{sec:nonlinear-maps} without any further modifications.

\begin{theorem}
\label{thm:attention-positional}
Assume $G = (\mathbb{Z},+)$ acts on $X = \mathbb{Z}$ by addition and let $H= \{e\}$ be the trivial group.
Then the standard discrete self-attention with relative positional bias in \cref{eq:relative-self-attention} is of the form in \cref{eq:nonlinear-operator} where $ \hat{\omega}$ is given by
\begin{equation}
		\hat{\omega}( g ^{-1}f, g') = \operatorname{Softmax} \left\{ \dfrac{[g ^{-1}f](e)^\top W_Q^\top W_K [ g ^{-1}f](g') }{ \sqrt{d}} \odot \psi(g') \right\} W_V [g ^{-1}f](g').
\end{equation}
\end{theorem}
\begin{proof}
	Writing out \cref{eq:nonlinear-operator} and applying a change of variable $g' \mapsto g^{-1}g'$ yields
	\begin{equation}
		\begin{aligned}
			[\Phi _{ \hat{\omega}}f ](g) &= \int_G \hat{\omega}( g ^{-1}f, g')\mathrm{d}g' \\
			&= \int_G \operatorname{Softmax} \left\{ \dfrac{[g ^{-1}f](e)^\top W_Q^\top W_K [ g ^{-1}f](g') }{ \sqrt{d}} \odot \psi(g') \right\} W_V [g ^{-1}f](g') \mathrm{d}g' \\
							      &= \int_G \operatorname{Softmax} \left\{ \dfrac{f(g)^\top W_Q^\top W_K f](gg') }{ \sqrt{d}} \odot \psi(g') \right\} W_V f](gg') \mathrm{d}g' \\
							      &= \int_G \operatorname{Softmax} \left\{ \dfrac{f(g)^\top W_Q^\top W_K f](g') }{ \sqrt{d}} \odot \psi(g^{-1}g') \right\} W_V f](g') \mathrm{d}g'.
		\end{aligned}
	\end{equation}
	Note that for a discrete group like $(\mathbb{Z},+)$ the integral over the Haar measure $\int_G [ \cdots ]dg'$ reduces to a sum (via the counting measure) as $\sum _{j \in \mathbb{Z}} [ \cdots ]$.
	Also change to the appropriate notation $g = x$ and $g ^{-1}g' = g' - g = x' - x$ for elements in the abelian group $(\mathbb{Z},+)$, which yields
	\begin{equation}
		\begin{aligned}
			[\Phi _{ \hat{\omega}}f ](x) &= \sum _{x' \in \mathbb{Z}} \operatorname{Softmax} \left\{ \dfrac{f(x)^\top W_Q^\top W_K f](x') }{ \sqrt{d}} \odot \psi(x' - x) \right\} W_V f](x') dg' \\
		\end{aligned}
	\end{equation}
	which is the expression of self-attention with relative positional bias in \cref{eq:relative-self-attention}.
\end{proof}

It should be pointed out that the proof above only depended on the fact that the relative positional embedding $\psi$ depended on the difference $x' -x$.
In fact, any positional embeddings which depend only on such a difference (and crucially not on the feature map $f$) fit into the framework in a similar way.

\subsubsection*{Rotary position embeddings}
Another important example of embeddings which induce translation equivariance are the \emph{rotary position embeddings} (RoPE) introduced in \citet{su2023roformerenhancedtransformerrotary},
which embed complex valued queries $q(x) = W_Q f(x)$ and keys $k(x') = W_K f(x')$ by rotating their phases based on their sequence positions $x$ and $x'$ in $\mathbb{Z}$, respectively.
In implementations, the $i$-th complex feature channel is modelled by a pair of real valued channels, acted on by a rotation matrix $R(\theta_i x) \in \mathrm{SO}(2)$.\footnote{The $R(\theta_i x) \in \mathrm{SO}(2)$ are real-valued unitary irreducible representations of the translation group $(\mathbb{Z},+)$.}
The frequencies $\Theta = \{\theta_1, \theta_2, \dots, \theta_{\tilde{c}/2}\}$
with which the pairs' phases rotate differ hereby per channel,
such that the overall action on feature vectors is modelled by the orthogonal matrix
$R_\Theta(x) := R(\theta_1 x) \oplus \dots \oplus R(\theta_{\tilde{c}/2} x)$.
Self-attention with such embedded queries and keys becomes
\begin{equation}
\begin{aligned}
\label{eq:attention-rotary}
    [\Phi ^{ \text{attn}}_\psi f](x) &= \sum_{x' \in \mathbb{Z}} \operatorname{Softmax} \left\{ \dfrac{f(x)^\top W_Q^\top R_{\Theta}(x)^\top R_{\Theta}(x') W_K f(x') }{ \sqrt{d}} \right\} W_V f(x') \\
    &= \sum_{x' \in \mathbb{Z}} \operatorname{Softmax} \left\{ \dfrac{f(x)^\top W_Q^\top R_{\Theta}(x'-x) W_K f(x') }{ \sqrt{d}} \right\} W_V f(x') \,,
\end{aligned}
\end{equation}
where the second step, resulting in a dependence on relative sequence positions $x'-x$ only, made use the matrices' orthogonality.
For details we refer the reader to the original article \citep{su2023roformerenhancedtransformerrotary}.

We state the observation here without any proof since the derivation follows straightforwardly from the techniques used to prove \cref{thm:attention-positional}.
\begin{theorem}
    Assume $G = (\mathbb{Z},+)$ acts on $X = \mathbb{Z}$ by addition and let $H= \{e\}$ be the trivial group.
Then the self-attention layer with rotary positional embeddings in \cref{eq:attention-rotary} is of the form in \cref{eq:nonlinear-operator} where $ \hat{\omega}$ is given by
\begin{equation}
		\hat{\omega}( g ^{-1}f, g') = \operatorname{Softmax} \left\{ \dfrac{[g ^{-1}f](e)^\top W_Q^\top R_{\Theta}(g^{-1}g') W_K [ g ^{-1}f](g') }{ \sqrt{d}} \right\} W_V [g ^{-1}f](g').
\end{equation}
\end{theorem}
By formalising rotary embeddings in this framework one could potentially find equivalent embedding techniques for more general geometries than the affine setting considered here.

\subsection{Equivariant self-attention: The LieTransformer}
\label{sec:lietransformer-from-general-map}
We finish by considering another equivariant self-attention architecture in the LieTransformer introduced in \citet{hutchinsonLieTransformerEquivariantSelfattention2021}.
The authors introduce a family of continuous attention-inspired operators for feature maps on a homogeneous space $X \cong G/H$.
The operators act on feature maps $f \in \mathcal{I}_\rho$ lifted to the group $G$ in the same way as introduced in \cref{sec:linear-maps}.
The proposed operators are of the form
\begin{equation}
\begin{aligned}
	\label{eq:lietransformer-def}
	\big[\Phi^{ \text{attn}}_\alpha f\big](g) &= \frac{1}{\mathcal{Z}_\alpha(g^{-1}f)} \int_G \alpha\big(f(g), f(g'), g^{-1}g'\big) W_V f(g')\, \mathrm{d}g', \\
        \mathcal{Z}_\alpha(g^{-1}f) &= \int_G \alpha\big([g^{-1}f](e), [g^{-1}f](g'), g'\big) \mathrm{d}g'
        \end{aligned}
\end{equation}
for some map $\alpha: V_\rho \times V_\rho \times G \to \mathbb{R}$, matrix $W_V \in \hom(V_\rho, V_\sigma)$ and normalisation $\mathcal{Z}_\alpha: \mathcal{I}_\rho \to \mathbb{R}$ which depends on the choice of $\alpha$.
One can find that operators of this form are equivariant via a change of variable $g' \mapsto kg'$
\begin{equation}
	\begin{aligned}
		\big[ \Phi^{ \text{attn}}_\alpha(k \vartriangleright f)\big](g)
        &= \frac{1}{\mathcal{Z}_\alpha(g^{-1}kf)} \int_G \alpha\big(f(k^{-1}g), f(k^{-1}g'), g^{-1}g'\big) W_V f(k^{-1}g')\, \mathrm{d}g' \\
		&= \frac{1}{\mathcal{Z}_\alpha\big((k^{-1}g)^{-1}f\big)} \int_G \alpha\big(f(k^{-1}g), f(g'), g^{-1}kg'\big) W_V f(g')\, \mathrm{d}g' \\
		&= [\Phi^{ \text{attn}}_\alpha f](k^{-1}g) \\
        &= k \vartriangleright [\Phi^{ \text{attn}}_\alpha f](g).
	\end{aligned}
\end{equation}
However, without any additional restrictions on $\alpha$ there is no guarantee that $\Phi^{ \text{attn}}_\alpha f \in \mathcal{I}_\sigma$ for some representation $\sigma$ for all  $\alpha$ and $f \in \mathcal{I}_\rho$.
The authors therefore restrict to the setting where both representations $\rho$ and $\sigma$ act trivially and subsequently $f \in \mathcal{I}_\rho$ is an invariant function $f(gh) = f(g)$.
With this assumption the output is guaranteed to lie in $ \mathcal{I}_\sigma$, that is $[\Phi^{ \text{attn}}_\alpha f](gh) = [\Phi^{ \text{attn}}_\alpha f](g)$ for all $\alpha$ and $f \in \mathcal{I}_\sigma$
It is in principle possible to construct an equivalent of dot-product self-attention for non-trivially transforming features if one introduces the constraint that the feature embedding matrices $W_Q,W_K,W_V$ intertwines the representations $\rho$ and $\sigma$ in specific ways.

Calling this equivariant self-attention is motivated by the observation that certain choices of $\alpha$ and $ \operatorname{norm}$ yield operators that resemble equivariant versions of standard self-attention.
In particular
\begin{equation}
\label{eq:lietransformer-attention}
\alpha(f(g), f(g'), g^{-1}g') = \exp\left[ \dfrac{1}{ \sqrt{d}} f(g)^\top W_Q^\top W_K f(g') + \psi(g ^{-1}g') \right]
\end{equation}
for scaled dot-product attention with relative positional bias $\psi$.
This formulation directly generalises the setting where $G = (\mathbb{Z}, +)$ given in \cref{sec:relative-self-attention}.
Several other useful choices of $\alpha$, as well as for the normalisation $\mathcal{Z}$, can be found in the appendix of the original article \citep{hutchinsonLieTransformerEquivariantSelfattention2021}.

The following result shows that the general form of the LieTransformer given in \cref{eq:lietransformer-def}, as well as all specialisations such as \cref{eq:lietransformer-attention} are in fact special cases of the framework introduced in \cref{sec:nonlinear-maps}.
\begin{theorem}\label{thm:instance-lietransformer}
	Assume $\rho$ and $\sigma$ are the trivial representations.
	Then the LieTransformer defined in \cref{eq:lietransformer-def} is given in our framework as in \cref{eq:nonlinear-operator} where $ \hat{\omega}$ factors as
\begin{equation}\label{eq:lietransformer-from-general-map}
	\hat\omega(g^{-1}f,g')
    = \frac{1}{\mathcal{Z}_\alpha(g^{-1}f)} \alpha\big( [g^{-1}f](e), [g^{-1}f](g'), g' \big)\big\} W_V[g^{-1}f](g')
\end{equation}
where $e \in G$ is the unit element.
\end{theorem}
\begin{proof}
	The standard form of the LieTransformer in \cref{eq:lietransformer-def} is retrieved straightforwardly after a change of variable $g' \mapsto g^{-1}g'$ under the integral sign and using the left-invariance of the Haar measure $\mathrm{d}g'$.
	The equivariance of $\Phi^{ \text{attn}}_\alpha$ then follows directly from the equivariance result in \cref{cor:steerable_functional_argument_reduction}, and the Mackey condition derived in the same theorem ensures that $ \Phi^{ \text{attn}}_\alpha f \in \mathcal{I}_\sigma$ whenever $f \in \mathcal{I}_\rho$.
\end{proof}

One should note that in our framework it is clear that dependence on higher-order terms can be introduced without breaking equivariance.
In particular, one is allowed to introduce dependences such as $[g ^{-1}f](g'^2)$ into $\alpha$, which yields a dependence on $f(g' g ^{-1}g')$ after the change of variable $g' \mapsto g^{-1}g'$.
Note however that the Mackey condition of \cref{cor:steerable_functional_argument_reduction} still needs to be fulfilled, which is non-trivial to impose.
Nonetheless, if introduced properly this kind of dependence would yield a strictly more expressive self-attention layer, although it is not immediately clear how training dynamics and performance would be affected.

%% file: sections/conclusions.tex
In this article we have introduced a mathematical framework of non-linear equivariant integral operators and demonstrated how it incorporates various existing equivariant neural network architectures in a unified manner.
In particular, we explicitly show how results proven in the generality of the framework specialises to the setting of both convolutional and self-attentional equivariant architectures including $G$-CNNs, the continuous equivariant LieTransformer, the implicit steerable kernel and the standard discrete self-attention with relative positional encodings.
We also prove that any equivariant operator can be incorporated into the framework if one considers the appropriate function class.

While the current work is exclusively theoretical in nature, the connection to existing architectures with efficient implementations such as equivariant convolutions hints at possible avenues to implement the framework in code and to conduct experiments.
In particular, we believe that the implicitly learned kernels are good candidates for modelling the family of operators described mathematically in this article.

The integral ansatz we consider for our family of operators also opens up for possible extension using dependence on other types of feature vectors than those modelled by $f$.
In particular, a potential extension to the present work is to include dependence not only on feature maps $f$ with domain in $G$, but also on higher-order features with domains such as $G \times G$. 
If $G$ is a finite group, then feature maps with domain $G$ can be interpreted as node features, feature maps with domain $G \times G$ as edge features, and more generally feature with domain $G^n$ can be interpreted as features on $n$-simplices.
Studying operators with dependences on such features can be useful in various fields, in particular in the settings of graph networks and topological data analysis.

We have restricted our attention to equivariant neural networks on homogeneous spaces $G/H$. A natural extension of our work is to consider equivarant non-linear maps on arbitrary smooth manifolds $\mathcal{M}$. This would involve a non-linear generalization of gauge-equivariant CNNs \citet{cohen2019gauge,cheng2019covariancephysicsconvolutionalneural,haan2021gauge,gerkenGeometricDeepLearning2023,weiler2023EquivariantAndCoordinateIndependentCNNs}.
With such a generalisation, one might hope to include also the gauge-equivariant transformers (\citet{heGaugeEquivariantTransformer2021}) in a similar fashion to how group equivariant transformers were incorporated into the present framework. We hope to return to this in future work.

%% file: sections/acknowledgments.tex
We thank Jan Gerken, Dmitry Gourevitch and Max Carnesten for helpful discussions.
The work by E.N., O.C.,\ and D.P.\ was supported by the Wallenberg AI, Autonomous Systems and Software Program (WASP) funded by the Knut and Alice Wallenberg Foundation.

%% file: sections/appendix-equivariance-relations.tex
\section{Functional formulation}\label{app:functional-appraoch-to-non-lin-maps}
In \cref{sec:nonlinear-maps} we introduced a general framework based on the integral of a map $ \omega:\, \mathcal{I}_\rho \times G \times G^{\prime} \,\to\, V_\sigma $, however some statements cannot be made when assuming that $ \omega $ is an ordinary function.
To reach these higher levels of generality we must move up to a distributional framework.
In this section we now present a version of our framework in which the general map $ \Phi:\, \mathcal{I}_\rho \,\to\, \mathcal{I}_\sigma $ is viewed in the setting of distributions.
We refer the interested reader to \cite{Rudin1991-nu,Diestel1977-kb} and \cite{Folland1999-tm} for a thorough background.

In this section we assume that $G$ and $ G^{\prime} $ to be locally compact Lie groups and all integrals are done with respect to the left Haar-measure.
We
Additionally, we assume that the vector space $ V_\sigma $, on which the subgroup $ H \leq G $ acts on through the representation $ \sigma $, is finite dimensional which implies that this is a separable hilbert space.

To this end we first introduce the space of test functions.

Normally one defines test functions to be the space of smooth compactly supported functions, which in this case we would denote $ C _{c}^{\infty}(G,V_\sigma) $, but in our case want to restrict our space further to smooth compactly supported functions $ \phi $ satisfying the Mackey condition
    \begin{equation}
        \phi(gh)=\sigma(h^{-1})\phi(g), \qquad \forall h\in H,\ g\in G.
    \end{equation}
    We denote the space of such functions by $ C _{c,\sigma}^{\infty}(G,V_\sigma) $.

One can obtain these functions by a projection, which is a symmetrisation operation, given by    
    \begin{equation}\label{eq:functional-projection}
        P: C _{c}^{\infty}(G,V_\rho)\to C _{c,\rho}^{\infty}(G,V_\rho),\qquad \phi(g)=[P f](g)= \int_H \Delta_H(h)\rho(h^{-1})f(gh^{-1})\mathrm{d}h,
    \end{equation}
    where $ \Delta_H $ is the modularity function of $ H $, which is identically one for a unimodular group.
    Additionally, this projection is $ G $-equivariant in that $ [P[kf]](g)=[k[Pf]](g) $.

    Now we can define our test functions.
\begin{definition}[Space of test functions]
    Let $ C _{c}^{\infty}(G, V_\sigma) $ be the space of smooth functions from the group $ G $ to the $ H $-representation $ (\sigma, V_\sigma) $.
    The space of test functions will be the subset $ C _{c,\sigma}^{\infty}(G,V_\sigma) \subset C _{c}^{\infty} $ which satisfy the Mackey constraint
    \begin{equation}
        \phi(gh)=\sigma(h^{-1})\phi(g), \qquad \forall h\in H,\ g\in G,
    \end{equation}
    where $ H \leq G $ is a subgroup.
    Further, $ C _{c}^{\infty}(G,V_\sigma) $ is an LF-space, and the space of test functions $ C _{c,\sigma}^{\infty}(G,V_\sigma) $ inherits this structure as it is a closed subspace of $ C _{c}^{\infty}(G,V_\sigma) $.
\end{definition}
\begin{remark}
    An equivalent way to view $ C _{c,\sigma}^{\infty}(G,V_\sigma) $ is all elements of the induced representation $ \mathcal{I}_\sigma $ that are smooth and have compact support.
\end{remark}
\begin{definition}[Distributions]
   We define distributions $ \varphi \in \mathcal{D}^{\prime}(C _{c,\sigma}^{\infty}(G,V_\sigma)) $, i.e.\ continuous linear functionals, on this space of test functions through their action on test functions 
   \begin{equation}
       \varphi(\phi)=\int_G \left\langle\, \varphi(g) \,|\, \phi(g) \,\right\rangle _\sigma \Delta_G(g)\mathrm{d}g.
   \end{equation}
\end{definition}

Before we move on, we need more structure on our representation $ V_\sigma $, namely $ V_\sigma $ needs to be an inner product space where we denote the inner product as 
\begin{equation}
   \left\langle\, \bullet \,|\, \bullet \,\right\rangle_\sigma :\, V_\sigma \times V_\sigma \,\to\, \mathbb{C}. 
\end{equation}
With this we can now define the distributions that correspond to the map $ \omega:\, \mathcal{I}_\rho \times G \times G \,\to\, V_\sigma $ that was discussed in \cref{sec:nonlinear-maps}.

\begin{definition}
    Let $ \Omega \subset \operatorname{Map}(\mathcal{I}_\rho \times G^{\prime}, \mathcal{D}^{\prime}(C _{c,\sigma}^{\infty})) $ be a set of maps $ \varphi $ which maps $ (f,g^{\prime}) \in \mathcal{I}_\rho \times G^{\prime} $ into the distributions $ \mathcal{D}^{\prime}(C _{c,\sigma}^{\infty}(G,V_\sigma)) $.
    Specifically, for each $ (f,g^{\prime}) \in \mathcal{I}_\rho \times G^{\prime} $ the distribution $ \varphi[f,g^{\prime}] $ acts as
    \begin{equation}
        \varphi[f,g^{\prime}](\phi)=\int_{G} \left\langle\, \varphi[f,g^{\prime}](g) \,|\, \phi(g) \,\right\rangle_\sigma \Delta_G(g)\mathrm{d}g.
    \end{equation}
\end{definition}
    The proper intuition here is that $ \varphi[f,g^{\prime}](g) $ is the analogue to $ \omega(f,g,g^{\prime}) $ where $  \omega:\, \mathcal{I}_\rho \times G \times G^{\prime} \,\to\, V_\sigma  $ was introduced in \cref{def:general-Omega-space}.
    Additionally, $ \omega $ is subject to the constraint
    \begin{equation}
        \omega(f,gh,g^{\prime})=\sigma(h^{-1})\omega(f,g,g^{\prime}), \qquad  \forall g\in G,\ g^{\prime}\in G^{\prime},\ h\in H.
    \end{equation}
    For the case that $ \sigma $ is unitary on $ V_\sigma $ and that $ G $ is a unimodular group then we can obtain a corresponding transformation of $ \varphi $.
    To get this, note that
\begin{equation}
    \begin{aligned}
        \varphi[f,g^{\prime}](\phi)&=\int_G \left\langle\, \varphi[f,g^{\prime}](g) \,|\, \phi(g) \,\right\rangle_\sigma \mathrm{d}g\\
                                   &=\int_G \left\langle\, \varphi[f,g^{\prime}](gh) \,|\, \phi(gh) \,\right\rangle_\sigma \mathrm{d}g\\
                                   &= \int_G \left\langle\, \sigma(h)\varphi[f,g^{\prime}](gh) \,|\, \phi(g) \,\right\rangle_\sigma \mathrm{d}g.
    \end{aligned}
\end{equation}
Where we in the final step used that $ \phi\in C _{c,\sigma}^{\infty} \subset \mathcal{I}_\sigma $, and that for all $ u,v \in V_\sigma $ we have 
\begin{equation}
    \left\langle\, \sigma(h)u \,|\, v \,\right\rangle_\sigma = \left\langle\, u \,|\, \sigma(h^{-1})v \,\right\rangle _\sigma
\end{equation}
as $ \sigma $ was assumed unitary.
Since this now has to hold for all test functions $ \phi $ we obtain that 
\begin{equation}
    \varphi[f,g^{\prime}](g)=\sigma(h)\varphi[f,g^{\prime}](gh) \quad \Leftrightarrow \quad \varphi[f,g^{\prime}](gh)=\sigma(h^{-1})\varphi[f,g^{\prime}](g).
\end{equation}
Hence we get the corresponding constraint in the distributional setting.
\begin{definition}[Regular distribution]
    In this setting a distribution $ \varphi $ is called regular if there exists an ordinary function $ \psi\in \mathcal{I}_\sigma $ such that
    \begin{equation}
        \varphi(\phi)=\int_G \left\langle\, \varphi(g) \,|\, \phi(g) \,\right\rangle _\sigma \Delta_G(g)\mathrm{d}g = \int_G \left\langle\, \psi(g) \,|\, \phi(g) \,\right\rangle _\sigma \Delta_G(g)\mathrm{d}g,\qquad \forall \phi \in C _{c,\sigma}^{\infty}.
    \end{equation}
    If $ \varphi $ is a regular distribution, we will denote its corresponding function $ \widetilde{\varphi} $.
\end{definition}
Now we can state the distributional correspondence to the non-linear map $ \Phi:\, \mathcal{I}_\rho \,\to\, \mathcal{I}_\sigma $ from \cref{sec:nonlinear-maps}.
This map is defined through its action of a corresponding distribution on test functions.
\begin{definition}
    If $ \phi $ is a test function as defined above and $ \varphi \in \Omega  $, then we define the operator $  \Phi_\varphi :\, \mathcal{I}_\rho \,\to\, \mathcal{D}^{\prime}(\mathcal{I}_\sigma) $ through its action on test functions:
    \begin{equation}
        [\Phi_\varphi f](\phi)=\int_{G^{\prime}}\varphi[f,g^{\prime}](\phi)\mathrm{d}g^{\prime}.
    \end{equation}
\end{definition}
\begin{remark}
    Note that for this operator to be well-defined we need that map taking $ g^{\prime}\in G^{\prime} $ to the distribution $ \varphi[f,g^{\prime}] $ is weakly measurable.
\end{remark}

Any function $ \hat{f}:G\to V_\sigma $ generates a regular distribution through 
\begin{equation}
    \hat{f}(\phi)=\int_{G} \left\langle\, \hat{f}(g) \,|\, \phi(g) \,\right\rangle _\sigma \Delta_G(g)\mathrm{d}g.
\end{equation}
If, then, $ \lambda:\, \mathcal{I}_\rho \,\to\, \mathcal{I}_\sigma $ is any map between induced representations, then $ \lambda[f] $ generates a regular distribution through its action defined above.
From this a natural question arises as to whether this $ \Phi f_\varphi $ is universal in the sense that any regular distribution $ \lambda [f] $ can be written as
\begin{equation}
    \lambda[f](\phi) =[\Phi_\varphi f](\phi).
\end{equation}
In an intuitive sense we could then say that $ \Phi_\varphi= \lambda$ as the equality above should hold for all test functions $ \phi $.

Before we get to the statement we need to define the common $ \delta $-distribution.
\begin{definition}
    The distribution $ v\delta $ for $ v \in V_\sigma $ acts on test functions $ \phi $ as
    \begin{equation}
        [v \delta](\phi)=\int_{G} \left\langle\, v \delta(g) \,|\, \phi(g) \,\right\rangle \Delta_G(g)\mathrm{d}g= \left\langle\, v \,|\, \phi(e) \,\right\rangle_\sigma.
    \end{equation}
\end{definition}

    We can also consider left translates of $ v \delta $.
    Specifically, the distribution $ v [L_k \delta]= v\delta_k $ for $ k \in G $ and $ v \in V_\sigma $ acts on test functions $ \phi $ as 
    \begin{equation}
        [v\delta_k](\phi)= \int_G\left\langle\, v \delta_k(g) \,|\, \phi(g) \,\right\rangle_\sigma \Delta_G(g)\mathrm{d}g= \int_G\left\langle\, v \delta(k^{-1}\hat{g}) \,|\, \phi(\hat{g}) \,\right\rangle_\sigma \mathrm{d}\hat{g}= \left\langle\, v \,|\, \phi(k) \,\right\rangle_\sigma.
    \end{equation}
\begin{remark}
    If $ \sigma $ is unitary we get an additional condition.
    As our test functions $ \phi \in C _{c,\sigma}^{\infty}(G,V_\sigma) $ and hence satisfy $ \phi(gh)=\sigma(h^{-1})\phi(g) $ we need that  
    \begin{equation}
        [v \delta_{kh}](\phi)= \left\langle\, v \,|\, \phi(kh) \,\right\rangle_\sigma = \left\langle\, v \,|\, \sigma(h^{-1})\phi(k) \,\right\rangle_\sigma = \left\langle\, \sigma(h)v \,|\, \phi(k) \,\right\rangle _\sigma =[(\sigma(h)v)\delta_k](\phi).
    \end{equation}
    That is, the $ v\delta_k $ must satisfy 
    \begin{equation}
        v \delta_{kh}=(\sigma(h)v) \delta_k.
    \end{equation}
\end{remark}
\begin{remark}
    With this action, and the projection defined in \cref{eq:functional-projection}, a short calculation yields 
    \begin{equation}
        [P (v\delta_k)](\phi)=\Delta_G(k) \left\langle v | \phi(k) \right\rangle_{\sigma}= [v \delta_k](\phi).
    \end{equation}
    That is, the projected $ P\delta $ acts identically to the unprojected $ \delta $.
\end{remark}

\begin{theorem}\label{thm:universality-on-regular-distributions}
    For any map $ \lambda:\, \mathcal{I}_\rho \,\to\, \mathcal{I}_\sigma $ there exists a $ \varphi\in \Omega $ such that $ \lambda $ can be modelled by $ \Phi_\varphi $ in the sense
    \begin{equation}
        \lambda[f](\phi) = [\Phi_\varphi f](\phi),
    \end{equation}
    for all test functions $ \phi $ and functions $ f \in \mathcal{I}_\rho $.
\end{theorem}
\begin{proof}
    We start by expanding $ [\Phi_\varphi f](\phi) $ as
    \begin{equation}
        \begin{aligned}
            [\Phi_\varphi f](\phi)&=\int_{G^{\prime}}\varphi[f,g^{\prime}](\phi)\mathrm{d}g^{\prime}\\
                                  &=\int_{G^{\prime}}\int_G \left\langle\, \varphi[f,g^{\prime}](g) \,|\, \phi(g) \,\right\rangle _\sigma \Delta_G(g)\mathrm{d}g\, \mathrm{d}g^{\prime}.
        \end{aligned}
    \end{equation}
    Now, let $ \varphi[f,g^{\prime}](g)=\delta(g^{\prime})\lambda[f](g) $.
    This yields, with a swap of our integrals,
    \begin{equation}
        \begin{aligned}
            [\Phi_\varphi f](\phi)&=\int_{G^{\prime}}\int_G \left\langle\, \varphi[f,g^{\prime}](g) \,|\, \phi(g) \,\right\rangle _\sigma \Delta_G(g)\mathrm{d}g\, \mathrm{d}g^{\prime}\\
                                  &=\int_{G^{\prime}}\int_G \left\langle\, \delta(g^{\prime})\lambda[f](g) \,|\, \phi(g) \,\right\rangle _\sigma \Delta_G(g)\mathrm{d}g\, \mathrm{d}g^{\prime}\\
                                  &=\int_G \int_{G^{\prime}} \left\langle\, \delta(g^{\prime})\lambda[f](g) \,|\, \phi(g) \,\right\rangle _\sigma \mathrm{d}g^{\prime}\, \Delta_G(g)\mathrm{d}g \\
                                  &=\int_G \left\langle\, \lambda[f](g) \,|\, \phi(g) \,\right\rangle _\sigma \Delta_G(g)\mathrm{d}g\\
                                  &= \lambda[f](\phi).
        \end{aligned}
    \end{equation}
    Hence, through the choice $ \varphi[f,g^{\prime}](g)=\delta(g^{\prime})\lambda[f](g) $ we have 
    \begin{equation}
         \lambda[f](\phi) = [\Phi_\varphi f](\phi), \qquad  \forall \phi C _{c,\sigma}^{\infty}(G,V_\sigma).
    \end{equation}
    as we wanted.
\end{proof}

Now we turn to imposing equivariance on these maps.
\begin{theorem}
    If $ \varphi $ is equivariant, in the sense of $ \varphi[kf,g^{\prime}](\phi)=\varphi[f,g^{\prime}](k^{-1} \phi) $ for all $ k \in G $, then $ \Phi_\varphi $ is equivariant.
    That is
\begin{equation}
    [\Phi [kf]](\phi)=[\Phi f](k^{-1}\phi)=[k[\Phi f]](\phi), \quad \forall \phi\in C _{c,\sigma}^{\infty}.
\end{equation}
    In the case that $ \varphi[f,g^{\prime}] $ is a regular distribution then $  \varphi[kf,g^{\prime}](\phi)=\varphi[f,g^{\prime}](k^{-1} \phi)  $ implies that
    \begin{equation}
        \widetilde\varphi[kf,g^{\prime}](g)= \widetilde\varphi[f,g^{\prime}](k^{-1}g),
    \end{equation}
    which is the condition presented in \cref{sec:nonlin-equivariant-maps}.
\end{theorem}

\begin{proof}
The proof follows from a straight forward calculation:
\begin{equation}
    \begin{aligned}
        [\Phi_\varphi [kf]](\phi)&= \int_{G^{\prime}} \varphi[kf,g^{\prime}](\phi)\mathrm{d}g^{\prime}\\
                                 &= \int_{G^{\prime}} \varphi[f,g^{\prime}](k^{-1}\phi)\mathrm{d}g^{\prime}\\
                                 &= [\Phi_\varphi f](k^{-1}\phi)\\
                                 &= [k[\Phi_\varphi f]](\phi).
    \end{aligned}
\end{equation}
For the case when $ \varphi[f,g^{\prime}] $ is a regular distribution then we get
\begin{equation}
    \begin{aligned}
        \varphi[kf,g^{\prime}](\phi)&=\varphi[f,g^{\prime}](k^{-1}\phi)&& \Leftrightarrow \\
        \int_G \left\langle\, \widetilde\varphi[kf,g^{\prime}](g) \,|\, \phi(g) \,\right\rangle _\sigma \Delta_G(g)\mathrm{d}g &= \int_G \left\langle\, \widetilde{\varphi}[f,g^{\prime}](g)  \,|\, \phi(kg) \,\right\rangle _\sigma \Delta_G(g)\mathrm{d}g && \Leftrightarrow \\
        \int_G \left\langle\, \widetilde{\varphi} [kf,g^{\prime}](g) \,|\, \phi(g) \,\right\rangle_\sigma \Delta_G(g)\mathrm{d}g &= \int_G \left\langle\, \widetilde{\varphi} [f,g^{\prime}](k^{-1}g) \,|\, \phi(g) \,\right\rangle \Delta_G(g)\mathrm{d}g.
    \end{aligned}
\end{equation}
Since this needs to hold for all $ \phi $ we extract the pointwise equality that 
\begin{equation}
    \widetilde{\varphi}[kf,g^{\prime}](g)= \widetilde{\varphi}[f,g^{\prime}](k^{-1}g), 
\end{equation}
which is what we wanted and completes the proof.
\end{proof}

A natural question to ask is: if $ \varphi[f,g^{\prime}] $ is a regular distribution for each $ g^{\prime} $, i.e.\ can be viewed as a function, does the same hold for $ \Phi_\varphi[f] $?
As the next theorem shows, this is indeed the case.

\begin{theorem}
    If $ \varphi[f,g^{\prime}] $ is a regular distribution for each $ g^{\prime} $ such that the map $ g^{\prime} \mapsto \widetilde{\varphi}[f,g^{\prime}](g) $ is Bochner-integrable for each $ g $, then $ \Phi_\varphi $ is a regular distribution and its corresponding function is given by
    \begin{equation}
        [ \widetilde \Phi_{ \widetilde{\varphi} } f](g)=\int_{G^{\prime}} \widetilde\varphi[f,g^{\prime}](g)\mathrm{d}g^{\prime}.
    \end{equation}
\end{theorem}
\begin{remark}
    Assuming that the map $g^{\prime} \,\mapsto\, \varphi[f,g^{\prime}](g) $ is Bochner integrable for each $ g $ implies that we assume that the map is strongly measurable and that 
    \begin{equation}
        \int_{G^{\prime}} \left\| \varphi[f,g^{\prime}](g) \right\|_\sigma \mathrm{d}g^{\prime} < \infty. 
    \end{equation}
    This ensures that $ \int_{G^{\prime}} \varphi[f,g^{\prime}](g)\mathrm{d}g^{\prime} \in V_\sigma $ for each $ g \in G $.
\end{remark}
\begin{proof}
    Assuming that $ \varphi[f,g^{\prime}] $ is a regular distribution we know that
    \begin{equation}
        \varphi[f,g^{\prime}](\phi)=\int_G \left\langle\, \widetilde{\varphi}[f,g^{\prime}](g)  \,|\, \phi(g) \,\right\rangle_\sigma \Delta_G(g)\mathrm{d}g.
    \end{equation}
    This means that we get 
    \begin{equation}
        \begin{aligned}
            [\Phi _\varphi f](\phi)&=\int_{G^{\prime}} \varphi[f,g^{\prime}](\phi)\mathrm{d}g^{\prime}&&& (\text{use that }\varphi\text{ is regular})\\
                                   &=\int_{G^{\prime}}\int_G \left\langle\, \widetilde{\varphi} [f,g^{\prime}](g) \,|\, \phi(g) \,\right\rangle _\sigma \Delta_G(g)\mathrm{d}g \mathrm{d}g^{\prime} &&& (\text{change integration order})\\
                                   &=\int_G \int_{G^{\prime}} \left\langle\, \widetilde{\varphi}[f,g^{\prime}](g)  \,|\, \phi(g) \,\right\rangle _\sigma \mathrm{d}g^{\prime} \Delta_G(g)\mathrm{d}g &&& (\phi(g)\text{ is constant w.r.t. } g^{\prime})\\
                                   &=\int_G \left\langle\, \left.\int_{G^{\prime}} \widetilde{\varphi}[f,g^{\prime}](g) \mathrm{d}g^{\prime} \,\right|\, \phi(g) \,\right\rangle _\sigma \Delta_G(g)\mathrm{d}g.
        \end{aligned}
    \end{equation}
    Since $ g^{\prime} \mapsto \widetilde{\varphi} [f,g^{\prime}](g) $ is Bochner-integrable for each $ g $ we can define 
    \begin{equation}
        [\widetilde{\Phi}_{\widetilde{\varphi}}f](g):=\int_{G^{\prime}} \widetilde\varphi[f,g^{\prime}](g)\mathrm{d}g^{\prime},
    \end{equation}
    which yields
    \begin{equation}
        [\Phi_\varphi f](\phi)=\int_{G} \left\langle\, \left[ \widetilde{\Phi}_{\widetilde{\varphi}}f \right] (g)  \,|\, \phi(g) \,\right\rangle_\sigma \Delta_G(g)\mathrm{d}g.
    \end{equation}
    This shows that $ \Phi_\varphi f $ is a regular distribution with the corresponding function 
    \begin{equation}
        [\widetilde{\Phi}_{\widetilde{\varphi}}f](g):=\int_{G^{\prime}} \widetilde\varphi[f,g^{\prime}](g)\mathrm{d}g^{\prime},
    \end{equation}
    as wanted.
\end{proof}
This shows that if $ \varphi[f,g^{\prime}]:\, G \,\to\, V_\sigma $ is a function then so is $ \Phi_\varphi[f]:\, G \,\to\, V_\sigma $ and hence we can view $ \Phi_\phi $ as a map $ \Phi_\phi:\, \mathcal{I}_\rho \,\to\, \mathcal{I}_\sigma $.
In that case we can freely define $ \omega(f,g,g^{\prime}):=\varphi[f,g^{\prime}](g) $ which yields the setting in \cref{sec:nonlinear-maps}.
In that sense the distributional approach covers the theory presented in \cref{sec:nonlinear-maps}.

As mentioned in \cref{sec:conclusions} a natural extension to the framework is to consider feature maps which depend not just on $G$ but also on higher-order domains, such as $G \times G$.
A feature map $f(g_1,g_2)$ on $G\times G$ can be interpreted as encoding edge features on the edge connecting $g_1$ and $g_2$.
Additionally, this is a natural extension to the distributional formulation as, to fully utilise the kernel structure often present in the theory, one needs to have test functions defined on $G \times G$.
In this appendix, we can view our test functions $\phi$ as functions on $G\times G$, just that they are constant with respect to the second argument.
Hence allowing for non-trivial dependence is a strict generalisation also motivated from the machine learning side.